%% file: main.tex
\documentclass{article}

\usepackage[final]{neurips_2025}

\usepackage[utf8]{inputenc} 
\usepackage[T1]{fontenc}    
\usepackage{hyperref}       
\usepackage{url}            
\usepackage{booktabs}       
\usepackage{amsfonts}       
\usepackage{nicefrac}       
\usepackage{microtype}      
\usepackage{xcolor}         
\usepackage{tcolorbox} 

\usepackage{algorithm}
\usepackage{algpseudocode}
\usepackage{amsmath}
\usepackage{graphicx}
\usepackage{subcaption}
\usepackage{caption}
\usepackage{amsthm}
\usepackage{thmtools, thm-restate}

\title{Valid Inference with Imperfect Synthetic Data}

\author{%
  Yewon Byun$^{1}$, Shantanu Gupta$^{1}$, Zachary C. Lipton$^{1}$, Rachel Leah Childers$^{2}$, Bryan Wilder$^{1}$ \\
  Machine Learning Department, Carnegie Mellon University$^{1}$ $\hspace{1em}$ University of Zurich$^{2}$ \\
  \texttt{\{yewonb, shantang, zlipton, bwilder\}@cs.cmu.edu, rachel.childers@df.uzh.ch}
}

\begin{document}

\maketitle

\begin{abstract}
Predictions and generations from large language models are increasingly being explored as an aid in limited data regimes, such as in computational social science and human subjects research. 
While prior technical work has mainly explored the potential to use model-predicted labels for unlabeled data in a principled manner, there is increasing interest in using large language models to generate entirely new synthetic samples (e.g., synthetic simulations), such as in responses to surveys. However, it remains unclear 
by what means practitioners can combine such data with real data
and yet produce statistically valid conclusions upon them.
In this paper, we introduce a new estimator based on generalized method of moments, providing a hyperparameter-free solution with strong theoretical guarantees to address this challenge. Intriguingly, we find that interactions between the moment residuals of synthetic data and those of real data (i.e., when they are predictive of each other) can greatly improve estimates of the target parameter.
We validate the finite-sample performance of our estimator across different tasks in computational social science applications, demonstrating large empirical gains.
\end{abstract}

\section{Introduction}
 
Practitioners increasingly leverage large language models (LLMs) as cheap but noisy labelers for automating tasks traditionally reliant on manual human annotations \cite{ziems2024can}.
Beyond annotation, recently, practitioners have started to explore the possibility of leveraging LLMs for more diverse and open-ended forms of model-generated data, such as outputting entirely new synthetic samples, 
e.g., simulating human responses to surveys
or human participants in early pilot studies \citep{argyle2023out, brand2023using, dominguez2024questioning, anthis2025llm, hwang2025human}.
Determining the extent to which researchers should integrate LLM simulations---whether by simulating all samples, combining simulated and real samples, or relying entirely on human participants---remains an open and task-dependent question.
While such pipelines leveraging fully synthetic simulations have yet to be fully realized, reliable mechanisms for aggregating these data sources are indeed what will inform both the feasibility of such design choices and how such pipelines should be implemented in practice.

A persistent challenge, however, is that naively aggregating synthetic samples with real data for downstream inference often leads to greatly biased estimates, compromising the statistical validity of downstream conclusions. Ideally, we would like to realize the benefits of incorporating information from these additional data sources while retaining favorable statistical properties---consistency and proper asymptotic coverage. 
We consider the setting where practitioners have access to a corpus of unlabeled text and a small set of human-annotated samples with labeled covariates and outcomes. 
Here, practitioners can leverage LLMs to
(1) predict covariates and outcomes for the unlabeled text samples;
and (2) generate new text samples conditioned on available samples 
and label the covariates and outcomes for them similarly to (1).

First of all, it is not immediately obvious how to even produce synthetic samples such that they can be used in a principled manner. 
Naively drawing samples from a generative model and treating them as additional samples alongside real data makes it impossible to provide statistical guarantees for the resulting estimate if the generative model does not perfectly match the real distribution, which is expected in practice. 
We propose a specific sampling strategy in which each synthetic sample is generated conditional on an individual real text in Section \ref{sec:prelims}. What makes this formulation statistically powerful is that it introduces a correlation structure between each real text and synthetic sample. This correlation structure will prove critical for principled methods for integrating synthetic data, as it enables us to more effectively share information across them. 

In Section \ref{sec:gmm}, we introduce a new estimation framework based on Generalized Method of Moments (GMM) that allows us to incorporate this synthetic data. The GMM framework allows us to incorporate multiple sources of information by adding moments. The optimal weighting in GMM produces a combination of these moments that minimizes the variance of all estimators based on these moments \citep{chamberlain1987asymptotic}. This optimal weighting measures the cross-correlations between the synthetic and real data, producing a weighting matrix that reduces the variance of the real data moments if there is information from the synthetic data moments. 
Prospectively, it is not intuitive that the incorporation 
of additional moments based exclusively on synthetic data (defined in terms of a separate parameter from the target parameter)
should yield any benefits (or \textit{even affect}) the estimation of the target parameter of the real data. 
Intriguingly, we find that the incorporation of synthetic data leads to more precise estimation and tighter confidence intervals when the synthetic data moment residuals are predictive of the real data moment residuals. When they are independent from each other, the variance reduces to the optimal variance based only on the fully observed data. That is, in the worst case where synthetic data is \textit{completely uninformative}, including it does not hurt (at least asymptotically). Finally, in Section \ref{sec:results}, we analyze the finite-sample performance of our estimator using real-world datasets that encompass varying computational social science tasks, demonstrating large empirical gains.

At a fundamental level, this work takes a step towards understanding how synthetic data from foundation models can systematically be leveraged to support valid inference. As the usage and future promise of foundation models continue to grow, so too will the complexity of pipelines that incorporate their outputs. Our framework provides a foundation for easily extensible estimation methods that can safely incorporate the growing variety and quality of synthetic data sources from such models. 
More broadly, this GMM-based estimation framework for incorporating auxiliary data may be of broader interest as an alternative to the predominant debiasing-based methods in the surrogacy literature \citep{angelopoulos2023prediction}, as it can more flexibly accommodate multiple proxy covariates and proxy outcomes compared to existing approaches.

\section{Related Work}

\paragraph{Statistical Inference and Debiasing Methods.} 
Our work is broadly related to performing statistical inference with missing data, where past works have explored approaches to yielding valid and efficient parameter estimates \citep{robins1994estimation}. Other work has notably explored the usage of ML models to estimate nuisance parameters \citep{chernozhukov2018double}. The most related line of research are debiasing methods that focus on combining ground truth data with surrogate predictions (often produced by a machine learning model) to perform statistical inference \citep{egami2023using, gligoric2024can}. These frameworks are often referred to as prediction-powered inference \citep{angelopoulos2023prediction, angelopoulos2023ppi++} in the machine learning literature. 
Such methods have been well-studied in the context of predicted outcomes and, more recently, predicted covariates \citep{ji2025predictions}. 
A key difference between these works and our setting is that the primary focus of our work is how to incorporate fully synthetic samples, which remains unaddressed by previous work. For clarity, we refer to samples as fully synthetic when (1) the underlying text is synthetically generated and (2) both its covariates and outcomes are model predictions.

\vspace{-3mm}
\paragraph{LLMs for Data Annotation and Synthetic Simulation Tasks.} 
Our work is motivated by the increasingly growing use 
and future promise of LLMs for annotations 
and simulations \citep{ziems2024can, 10.1145/3706599.3716299, anthis2025llm}. 
There has been growing interest in using LLMs in fully synthetic simulation studies, 
with primary applications in exploratory research or early pilot studies. For instance, recent work has studied simulating social interactions and behaviors \citep{chenpersona, park2023generative}.
Other works have explored LLMs for simulating survey responses \citep{dillion2023can, rothschild2024opportunities, dominguez2024questioning}, analyzing how well simulations approximate human responses while cautioning about drawbacks 
such as limited diversity and lack of context-awareness.
Whether such simulations should be used on their own or in combination with "real" data remains an open and task-dependent question.
In summary, this line of work shows the potential of incorporating synthetic data 
powered through strong generative models in such downstream pipelines
but also exhibits clear failure modes and imperfect conclusions from such studies.
While most of these works focus on qualitative takeaways 
and early signals for future experiments, 
we focus on the forward-looking setting
of making statistically valid inference 
given such synthetic samples.

\section{Preliminaries}
\label{sec:prelims}

\paragraph{Notation and Setup.} We consider a parameter estimation task where the goal is to estimate a target parameter $\theta^\star \in \mathbb{R}^d$. Let $(T, X, Y) \sim \mathcal{D}$ denote a random triple drawn from an unknown data-generating distribution $\mathcal{D}$ over text inputs $T \in \mathcal{T}$, covariates about the text (e.g., structured metadata) $X \in \mathcal{X} \subseteq \mathbb{R}^d$, and 
labels $Y \in \mathcal{Y}$.
For example, $T$ can be texts from online requests, where $X$ are linguistic markers of hedging (i.e., notions of uncertainty) and $Y$ is perceived politeness. Due to labeling budget constraints, we assume we only observe a small fraction of human-annotated data (i.e., ground-truth covariates and labels about the text). Specifically, we have access to labeled dataset $\mathcal{D}_{\text{labeled}} = \{(T_i, X_i, Y_i)\}_{i=1}^{n}$ that is sampled i.i.d. from $\mathcal{D}$ and an unlabeled corpus of text $\mathcal{D}_{\text{unlabeled}} = \{(T_j)\}_{j=n+1}^{n+m}$ sampled i.i.d. from $\mathcal{D}_T$ (i.e., the marginal distribution over $T$), where $m \gg n$.
To supplement this limited supervision, we leverage machine learning models in the following two ways.

\paragraph{Proxy Covariates and Labels.} We use a machine learning model $f$ to produce predictions 
$\{f_{X}({T}_j), f_{Y}{(T}_j)\}$ for the available set of input texts $T \in \mathcal{T}$. Here, $f_{X}$ and $f_{Y}$ denote the same machine learning model, using separate prompts for the target outcome (either a covariate $X$ or outcome $Y$) (see Appendix \ref{appx:experiments} for details). This yields the following
$\mathcal{D}_{\text{proxy}} = \{(T_i, f_X(T_i), f_Y(T_i)\}_{i=1}^{n} \cup \{(T_j, f_X(T_j), f_Y(T_j)\}_{j=n+1}^{n+m}$. For simplicity, we will refer to these as \textbf{proxy samples} and denote them as $(T, \hat{X}, \hat{Y})$. We will refer to the distribution over proxy samples as $\hat{D}$. This is the main setting generally considered in the surrogacy literature (restricted to predicted outcomes).

\paragraph{Synthetic Covariates and Labels.} We propose a new data augmentation process which generates new samples using the same machine learning model $f$ (employing it as a generative model, instead of a classifier). Specifically, our method conditions the generation process on each individual text $T_j$ as an example and asks the model to generate a new synthetic sample given that context. Formally, for each $i$, we sample a new text $\tilde{T}_{i}$, conditioned on $(T_i, X_i)$ if the sample is labeled and $(T_j, \hat{X}_j)$ if the sample is unlabeled (since $X_{j}$ is not available, by definition).  More concretely, $\tilde{T}_k \sim \mathbb{P} (\cdot \mid T_i, X_i) \text{ if labeled}$ and $\tilde{T}_k \sim \mathbb{P} (\cdot \mid T_j, \hat{X}_j) \text{ if unlabeled}$. See Appendix \ref{appx:experiments} for prompts used for synthetic data generation.
Based on the generated sample, which we denote as $\tilde{T}_{i}$, we then extract its corresponding covariates and outcomes using $f$ similarly as in proxy samples. More concretely, $\tilde{X}_k \sim \mathbb{P} (\cdot \mid \tilde{T}_k)$ and $\tilde{Y}_k \sim \mathbb{P} (\cdot \mid \tilde{T}_k)$.
This results in $\mathcal{D}_{\text{synthetic}} = \{(\tilde{T}_k, \tilde{X}_k, \tilde{Y}_k)\}_{k=1}^{n+m}.$ We refer to the distribution over \textbf{synthetic samples} $(\tilde{T}, \tilde{X}, \tilde{Y})$ as $\tilde{D}$. 

This specific sampling process has two motivations. First, from a machine learning perspective it can be seen as a form of in-context prompting, where the model is given an example from the dataset in order to align it more closely with the task. Iteratively prompting with different samples $T_i$ is also likely to produce more diverse samples than asking for many samples with the same prompt. Second, from a statistical perspective, it introduces a correlation structure between each real text $T_i$ and synthetic sample $\tilde{T}_i$. This correlation structure will prove critical for principled methods for integrating synthetic data because it allows us to more effectively share information across them. 

Finally, we introduce some notation that combines all of these data sources into draws from a single joint distribution. Specifically, we introduce a new random variable $s \in \{0,1\}$ which is an indicator for whether $T$ is labeled (1) or unlabeled (0). Then, we view the complete generative process as draws $(T, s,  s\cdot X, s\cdot Y, \tilde{X}^1, \tilde{Y}^1...\tilde{X}^M, \tilde{Y}^M)$ for $M$ different kinds of auxiliary data. So far, we have discussed two kinds, proxy and synthetic, that we employ empirically ($M = 2$), but our methods are fully extensible to additional kinds of auxiliary data. For example, we could include samples from multiple different generative models. The real $(X, Y)$ are observed only for labeled points with $s = 1$ while the auxiliary data is available for all samples. The joint distribution over this full tuple is induced by the composition of the generative processes for the components described above.

\section{Combining Synthetic Information via Generalized Method of Moments}
\label{sec:gmm}

To estimate the target parameter $\theta^\star$, we adopt a generalized method of moments (GMM) approach \citep{hansen1982large} that combines information from the different types of data in the following manner.

\subsection{Moment Conditions}

Our framework is applicable whenever the target parameter can be identified by a set of moment conditions, functions whose expectation should be zero at the true value of the parameter. Moment-based estimation is a broad and flexible framework that includes almost all commonly used statistical frameworks (e.g., maximum likelihood, generalized linear models, instrumental variables, etc). We begin by defining the moment conditions that identify $\theta^{*}$ under the distribution of interest (i.e., the real-data distribution $\mathcal{D}$). In the following section, we introduce how this can be adapted to incorporate surrogate data (i.e., proxy and synthetic data).

 Formally, we consider the problem of estimating a parameter $\theta \in \mathbb{R}^d$. The true value $\theta^*$ is identified as the solution to a set of $p \geq d$ moment conditions 
 \begin{align*}
     \mathbb{E}[\psi^{(\ell)}(\theta^{*})] = 0, \quad \ell=1...p
 \end{align*}
 where the $\psi^{(\ell)}$ are continuously differentiable functions $\mathbb{R}^d \to \mathbb{R}$. For example, in a maximum likelihood model, we would have one $\psi$ for the derivative of the log-likelihood with respect to each parameter, and the moment conditions enforce that $\theta^*$ satisfies the first-order conditions for maximizing the likelihood. Let $\psi(\theta) = [\psi^{(1)}(\theta)...\psi^{(p)}(\theta)]^\top$ denote a column vector stacking the $p$ moments.

\subsection{Constructing Our GMM Estimator}

To leverage the auxiliary data (i.e., proxy data and synthetic data) in making our GMM estimator more efficient, we can construct a set of auxiliary moments for each additional source of data. We estimate an additional set of auxiliary parameters $\eta_1, ..., \eta_M \in \mathbb{R}^p$, one parameter vector for each set of new auxiliary data. In the specific instantiation of the model that we use here, we always have $M = 2$ (proxy and synthetic data), but in principle our method is extensible to many sources of auxiliary data, for example synthetic samples generated from several different models. Roughly, each new parameter vector $\eta_i$ can be understood as the parameter that we would estimate using each auxiliary data source, and our augmented model will automatically determine how to use these auxiliary estimates to inform the estimate of the parameter of interest $\theta$. 

For each new parameter vector $\eta_i$, we introduce a corresponding set of new moments to estimate this parameter and allow its estimate to inform the estimate of $\theta$. Specifically, we introduce for each $\eta_i$ two new blocks of moments that are copies of the original moments for $\theta$. Intuitively, one block of moments will be evaluated only on the real (labeled) data, while the other will be taken on the pooled set of labeled data and auxiliary dataset $i$. The pooled-data moment will allow us to improve the estimation of $\eta_i$ using the larger sample. The version evaluated only on the real data will allow GMM to evaluate how well the moments for the auxiliary parameter correlate with those of the true parameter on the same data, and share information across them if the auxiliary moments are informative (as we would expect if the generated data is high quality).  

Formally, let $S_t \in \mathbb{R}^p$ stack $p$ copies of the indicator variable $s_t$ for whether a data point $t$ is labeled. In block matrix notation, the combined model takes the form of the augmented moments

\begin{equation}
g_t(\theta, \eta) = \left[
\begin{array}{c}
S_t \\
S_t \\
\vdots \\
S_t \\
1 \\
\vdots \\
1
\end{array}
\right] \odot
\left[
\begin{array}{c}
\psi(\theta) \\
\psi(\eta_1) \\
\vdots \\
\psi(\eta_M) \\
\psi(\eta_1) \\
\vdots \\
\psi(\eta_M)
\end{array}
\right] \in \mathbb{R}^{p + 2Mp}
\label{eq:proxy}
\end{equation}

We will then jointly estimate $(\theta, \eta)$ as the solution to the moment condition $\mathbb{E}[g_t(\theta,\eta)] = 0$. For clarity, we refer to our estimator that uses real and proxy data ($M = 1$) as 
\textbf{GMM-Proxy} and our estimator that uses real, proxy, and synthetic data ($M = 2$) as \textbf{GMM-Synth} throughout the paper. We remark that since the parameter of interest $\theta$ appears only in its original set of moments, which are evaluated only on the labeled data, this new moment condition still identifies the target parameter $\theta^*$. However, as we discuss below, when we apply standard methods for efficiently estimating the augmented GMM, the new moment conditions will be leveraged to reduce the variance of the estimate without compromising consistency or asymptotic normality.

Before turning to estimation, we provide a concrete example of our moment construction for the case of generalized linear models (GLMs) in two-dimensions. 

\subsection{Example 1. Generalized Linear Models}\label{sec:glm}

Recall that the standard GLM formulation optimizes the objective function,
\begin{align*}
    \ell_\theta(x, y) = -yx^T \theta + f(x^T \theta),
\end{align*}

where $f$ is a function that is convex and infinitely differentiable. We remark that this recovers the setting of logistic regression when $f(z) = \log(1 + \exp({z}))$. Let us assume a two-dimensional setting for illustration. This translates to the population moment conditions of 
\begin{align*}
    \mathbb{E}\left[X_1\left(Y - \frac{\partial f}{\partial \theta_1}(X^T \theta^*)\right)\right] & = 0, \quad 
    \mathbb{E}\left[X_2\left(Y - \frac{\partial f}{\partial \theta_2}(X^T \theta^*)\right)\right]  = 0 
\end{align*}

We have similar moments for proxy and synthetic data, where we use parameters $\eta = (\eta^{(1)}, \eta^{(2)})$, which are also two-dimensional. Within our GMM framework, we construct the following set of moment conditions across the observed, proxy, and synthetic data. 
\begin{align*}
g_t(\theta, \eta) =
\left[
\begin{array}{c}
\vphantom{X_t ( Y_t - \frac{\partial f}{\partial \theta_1}(X_t^T \theta) )} s_t \\
\vphantom{X_t ( Y_t - \frac{\partial f}{\partial \theta_1}(X_t^T \theta) )} s_t \\
\vphantom{\hat{X}_t ( \hat{Y}_t - \frac{\partial f}{\partial \eta^{(1)}_1}(\hat{X}_t^T \eta^{(1)}) )} s_t \\
\vphantom{\hat{X}_t ( \hat{Y}_t - \frac{\partial f}{\partial \eta^{(1)}_1}(\hat{X}_t^T \eta^{(1)}) )} s_t \\
\vphantom{\tilde{X}_t ( \tilde{Y}_t - \frac{\partial f}{\partial \eta^{(2)}_1}(\hat{X}_t^T \eta^{(2)}) )} s_t \\
\vphantom{\tilde{X}_t ( \tilde{Y}_t - \frac{\partial f}{\partial \eta^{(2)}_1}(\hat{X}_t^T \eta^{(2)}) )} s_t \\
\vphantom{\hat{X}_t ( \hat{Y}_t - \frac{\partial f}{\partial \eta^{(1)}_1}(\hat{X}_t^T \eta^{(1)}) )} 1 \\
\vphantom{\hat{X}_t ( \hat{Y}_t - \frac{\partial f}{\partial \eta^{(1)}_1}(\hat{X}_t^T \eta^{(1)}) )} 1 \\
\vphantom{\tilde{X}_t ( \tilde{Y}_t - \frac{\partial f}{\partial \eta^{(2)}_1}(\hat{X}_t^T \eta^{(2)}) )} 1 \\
\vphantom{\tilde{X}_t ( \tilde{Y}_t - \frac{\partial f}{\partial \eta^{(2)}_1}(\hat{X}_t^T \eta^{(2)}) )} 1 \\
\end{array}
\right]
\odot
\left[
\begin{array}{c}
X_{t,1} ( Y_t - \frac{\partial f}{\partial \theta_1}(X_t^T \theta) ) \\
X_{t,2} ( Y_t - \frac{\partial f}{\partial \theta_2}(X_t^T \theta) ) \\
\hat{X}_{t,1} ( \hat{Y}_t - \frac{\partial f}{\partial \eta^{(1)}_1}(\hat{X}_t^T \eta^{(1)}) ) \\
\hat{X}_{t,2} ( \hat{Y}_t - \frac{\partial f}{\partial \eta^{(1)}_2}(\hat{X}_t^T \eta^{(1)}) ) \\
\tilde{X}_{t,1} ( \tilde{Y}_t - \frac{\partial f}{\partial \eta^{(2)}_1}(\tilde{X}_t^T \eta^{(2)}) ) \\
\tilde{X}_{t,2} ( \tilde{Y}_t - \frac{\partial f}{\partial \eta^{(2)}_2}(\tilde{X}_t^T \eta^{(2)}) ) \\
\hat{X}_{t,1} ( \hat{Y}_t - \frac{\partial f}{\partial \eta^{(1)}_1}(\hat{X}_t^T \eta^{(1)}) ) \\
\hat{X}_{t,2} ( \hat{Y}_t - \frac{\partial f}{\partial \eta^{(1)}_2}(\hat{X}_t^T \eta^{(1)}) ) \\
\tilde{X}_{t,1} ( \tilde{Y}_t - \frac{\partial f}{\partial \eta^{(2)}_1}(\tilde{X}_t^T \eta^{(2)}) ) \\
\tilde{X}_{t,2} ( \tilde{Y}_t - \frac{\partial f}{\partial \eta^{(2)}_2}(\tilde{X}_t^T \eta^{(2)}) ) \\
\end{array}
\right]
\end{align*}

\subsection{GMM Estimation}

Given our augmented moment conditions $g$, we estimate the parameters $(\theta,\eta)$ by minimizing the GMM objective:
\begin{equation}
\hat{\theta}_T,\hat{\eta}_T = \arg\min_{\theta \in \Theta, \eta \in \mathbb{R}^{2Mp}} \widehat{Q}_T(\theta, \eta),
\end{equation}
where
\begin{equation}\label{eq:gmm_estimator}
\widehat{Q}_T(\theta, \eta) = 
\left[
\frac{1}{T} \sum_{t=1}^T g_t(\theta, \eta)
\right]^\top
\widehat{\mathbf{W}}_T
\left[
\frac{1}{T} \sum_{t=1}^T g_t(\theta, \eta)
\right].
\end{equation}

\paragraph{Two-step GMM estimator.} We adopt the two-step GMM procedure as described in \citet{newey1994large}. First, we compute the one-step estimator $\hat{\theta}^{(\text{os})}_T,\hat{\eta}^{(\text{os})}_T$ using an identity weight matrix $\widehat{\mathbf{W}}_T = \mathbf{I}$. Then, we estimate the optimal weight matrix as:
\begin{equation}
\widehat{\Omega}_T(\hat{\theta}^{(\text{os})}_T,\hat{\eta}^{(\text{os})}_T) = 
\left[
\frac{1}{T} \sum_{t=1}^T g_t(\hat{\theta}^{(\text{os})}_T,\hat{\eta}^{(\text{os})}_T) g_t(\hat{\theta}^{(\text{os})}_T,\hat{\eta}^{(\text{os})}_T)^\top
\right],
\end{equation}
and set
\begin{equation}
\widehat{\mathbf{W}}_T = \left[\widehat{\Omega}_T(\hat{\theta}^{(\text{os})}_T,\hat{\eta}^{(\text{os})}_T)\right]^{-1}.
\label{eq:optimal}
\end{equation}

This optimal weighting has the interpretation as the inverse empirical covariance of the moment conditions on the one-step estimate. We then compute the final two-step estimator by minimizing $\widehat{Q}_T(\theta)$ with this updated weighting matrix. This choice of $\widehat{\mathbf{W}}_T$ yields an asymptotically efficient estimator under standard GMM regularity conditions. Following \citet{chamberlain1987asymptotic}, this choice of weighting minimizes the semiparametric efficiency bound with respect to the semi-parametric model defined by these moments (see Appendix \ref{appx:optimality} for more details). 

The adoption of two-step GMM is a critical component of our proposed estimation framework. Indeed, in the first-step estimates, the synthetic and proxy data will have no impact on the estimate of $\theta$ because they never appear in the moment conditions concerning $\theta$. In the second stage though, the weight matrix $\widehat{\mathbf{W}}_T$ accounts for the covariance between moment conditions, where off-diagonal terms in the matrix allow moments for the auxiliary data sources to influence the estimation of $\theta$.

\subsection{Consistency and Asymptotic Inference}

We now present results on the consistency and asymptotic behavior of our GMM estimators. 
\begin{restatable}{proposition}{gmm_asymptotics}\label{prop:asymptotics}
Our estimate $\hat{\theta}_T$ (as defined in Equation \ref{eq:gmm_estimator}) is consistent and asymptotically normal. It converges in distribution as 
\begin{align*}
    \sqrt{T}((\hat{\theta}_T^{\prime},\hat{\eta}_T^{\prime})^{\prime} - (\theta^{\prime},\eta^{\prime})^{\prime}) \xrightarrow[]{d} \mathcal{N}(0, V)
\end{align*}
where the covariance $V$ is given by 
\begin{align*}
    V = \left( G(\theta,\eta)^T \widehat{\mathbf{W}}G(\theta,\eta) \right)^{-1} G(\theta,\eta)^T \widehat{\mathbf{W}} F \widehat{\mathbf{W}} G(\theta,\eta) \left( G(\theta,\eta)^T \widehat{\mathbf{W}}G(\theta,\eta) \right)^{-1},
\end{align*}
and where $G(\theta,\eta)$ is the Jacobian of the population moments at the ground truth parameter values $\theta,\eta$. 
\end{restatable}

For optimal weight matrix in Equation \ref{eq:optimal}, this simplifies to $V=(G(\theta,\eta)^T F^{-1}G(\theta,\eta))^{-1}$. These are standard results on GMM estimators, which follow by straightforwardly applying the results in \citet{hansen1982large}. We remark that these asymptotic results require a set of conditions on the sample moments, which are slightly nuanced in this setting with multiple sources of information. 
We discuss these conditions and prove that they are satisfied in Appendix \ref{appx:conditions} for the setting of proxy and synthetic samples. 
Given this asymptotic behavior, we can derive valid confidence intervals for our parameter estimates.
\subsection{Why does synthetic data improve performance?}
\label{sec:understanding}
To understand where the benefits arise from incorporating the proxy and synthetic data into our GMM estimator, we analyze the interaction between our moment conditions. Note that the functions $\psi$ are often referred to as ``residuals" in the GMM literature; since $\psi(\theta)$ should be zero in-expectation, deviations from zero are interpretable as a kind of residual. The key intuition is that synthetic data will improve performance when the synthetic-data residuals are predictive of the real-data residuals. 

First, we note that if the synthetic data were perfectly simulated, $X$ and $Y$ would be perfectly recovered from the unlabeled text $T$. With ground truth $X, Y$, we can perfectly recover the residual terms.
In settings where we have good but imperfect simulations, $\hat{X}$,$\hat{Y}$ and $\tilde{X}, \tilde{Y}$ are highly correlated with the errors in the true data, and we can approximately estimate the real-data residuals with the synthetic data. Within our GMM-based approach, this is all handled implicitly in our two-step estimation procedure. During the first estimation step, each set of parameters (e.g., defined on the observed, proxy, and synthetic data) is independently identified since the initial weighting is an identity matrix. The key insight is that, during the second estimation step, the weighting matrix $\widehat{\mathbf{W}}$, which is the inverse of the moment covariance matrix, captures the interactions between the observed residual terms and the residuals from the synthetic data in our GMM objective. This is captured in the off-diagonal terms of the moment covariance matrix. 
Partitioning the moments into observed data residuals $m_t(\theta)$ and synthetic data residuals $h_t(\eta)$, we derive an explicit formula for the asymptotic variance of $\sqrt{T}(\hat{\theta}_T-\theta)$ in Theorem \ref{prop:target_asymptotic_variance} with the full proof in Appendix \ref{appx:partitioned}. 

\begin{restatable}{theorem}{targetvariance}
\label{prop:target_asymptotic_variance}
The asymptotic variance of $\sqrt{T}(\hat{\theta}_T-\theta)$ is given by 
{\small
$$(\frac{d \mathbb{E}[m(\theta)]}{d\theta^\prime}A\frac{d \mathbb{E}[m(\theta)]}{d\theta}-\frac{d \mathbb{E}[h(\eta)]}{d\eta^\prime}B^{\top}\frac{d \mathbb{E}[m(\theta)]}{d\theta}(\frac{d \mathbb{E}[h(\eta)]}{d\eta^\prime}D\frac{d \mathbb{E}[h(\eta)]}{d\eta})^{-1}\frac{d \mathbb{E}[m(\theta)]}{d\theta^\prime}B\frac{d \mathbb{E}[h(\eta)]}{d\eta})^{-1}.$$
}
with $A,B,D, m(\theta), h(\eta)$ defined in Appendix \ref{appx:partitioned}.
\end{restatable}

We find two important conclusions. First, when these residuals are independent of the observed data, the formula reduces to the optimal variance based only on the fully observed data. That is, in the worst case where synthetic data is completely uninformative, including it does not hurt (at least asymptotically). Second, when the real and synthetic residuals are correlated (as we would hope), we derive a lower bound on the variance which is proportional to the residual variance in a regression of the observed data residuals on the span of the synthetic data residuals. This bound is minimized by choosing moments that span the conditional expectation of the observed data residuals given $T_i$, a sufficient condition for which is that the conditional distribution of $\hat{X},\hat{Y}$ or $\tilde{X},\tilde{Y}$ given $T$ equals the conditional distribution of $X,Y$.

\section{Experimental Results}
\label{sec:results}

\subsection{Baselines}
\label{sec:baselines}

Existing methods in the literature are well-studied in the context of predicted outcomes and more recently, predicted covariates. However, it remains unclear how to aggregate information from fully synthetic data. We consider how to adapt existing debiasing methods; PPI++ \citep{angelopoulos2023ppi++} and recent variants \citep{ji2025predictions}.

\paragraph{RePPI} The most general approach is perhaps given by RePPI \citep{ji2025predictions}, which predicts the optimal loss through fitting an arbitrary model that maps the proxy and synthetic loss to the real loss. This results in the objective defined in Proposition \ref{prop:reppi}. The resulting estimate retains asymptotic normality conditions (see Appendix \ref{appx:reppi} for the proof and algorithm details).

\paragraph{PPI++Proxy, PPI++Synth} To adapt PPI++ \citep{angelopoulos2023ppi++} to this setting, we take an instantiation of RePPI, where the model is a convex combination to limit the number of parameters. This results in the following objective defined in Proposition \ref{prop:ppi-multiple}. We refer to the estimator with $\alpha = 1$ as PPI++Proxy, as the synthetic terms vanish, yielding an estimator that combines real and proxy data. We refer to the estimator with tunable $\alpha \in [0,1]$ as PPI++Synth, which combines real, proxy, and synthetic data. 
Note that the addition of this hyperparameter $\alpha$ adds increased complexity, and techniques such as cross-fitting must be used to select it in a statistically valid fashion. The resulting estimate retains asymptotic normality conditions (see Appendix \ref{appx:ppi-multi} for the proof and algorithm details).

\paragraph{PPI++Synth (Oracle)} As an upper bound, we conduct a grid search over different possible $\alpha$ values \textit{without} cross-fitting. Note, this is not a valid solution in the setup, as it requires peeking in hyperparameter selection, but it provides an oracle version of the baseline for reference.

\begin{figure} [t]
    \centering
    \begin{subfigure}[b]{0.32\textwidth}
        \includegraphics[width=\textwidth]{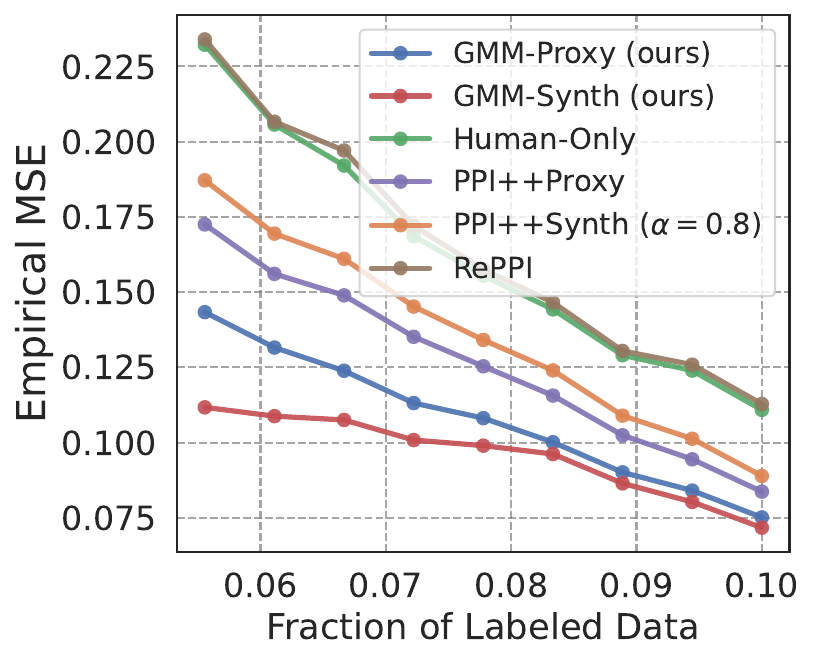}
    \end{subfigure}
    \hfill
    \begin{subfigure}[b]{0.32\textwidth}
        \includegraphics[width=\textwidth]{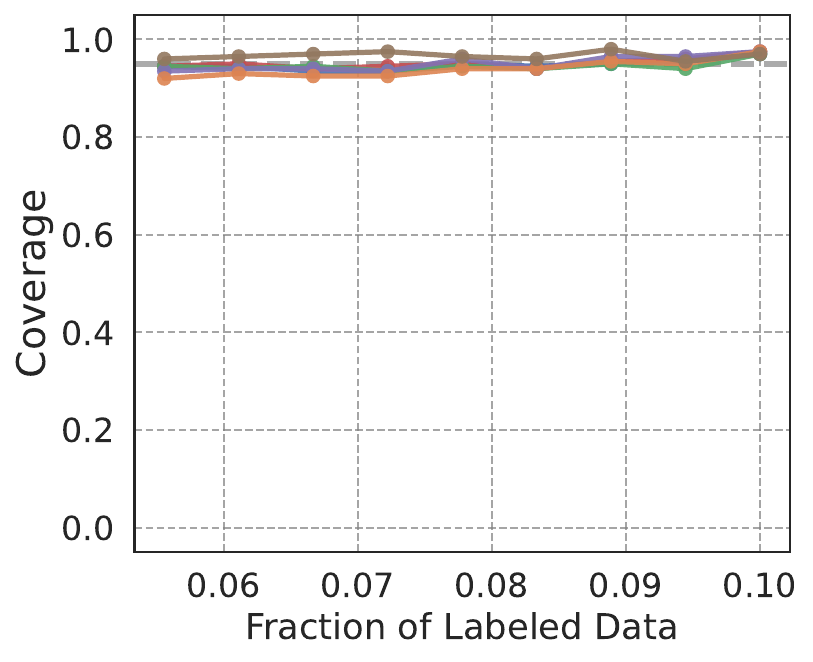}
    \end{subfigure}
    \hfill
    \begin{subfigure}[b]{0.32\textwidth}
        \includegraphics[width=\textwidth]{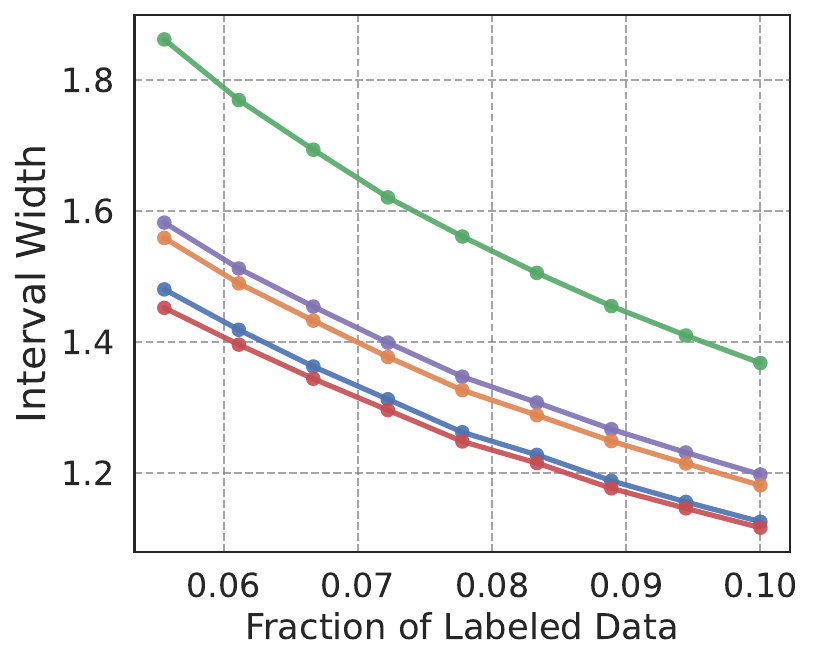}
    \end{subfigure}

    \vspace{0.25cm}
    \begin{subfigure}[b]{0.32\textwidth}
        \includegraphics[width=\textwidth]{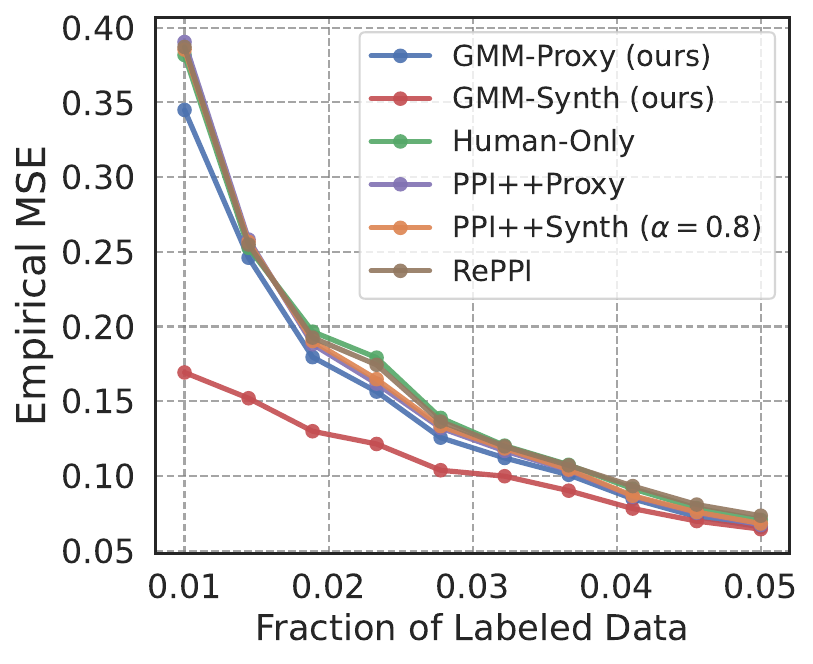}
    \end{subfigure}
    \hfill
    \begin{subfigure}[b]{0.32\textwidth}
        \includegraphics[width=\textwidth]{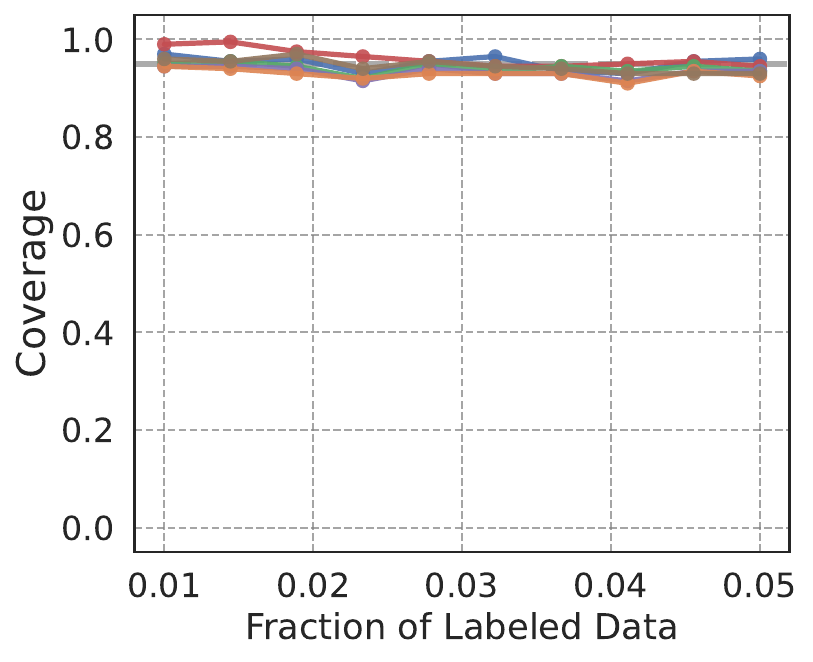}
    \end{subfigure}
    \hfill
    \begin{subfigure}[b]{0.32\textwidth}
        \includegraphics[width=\textwidth]{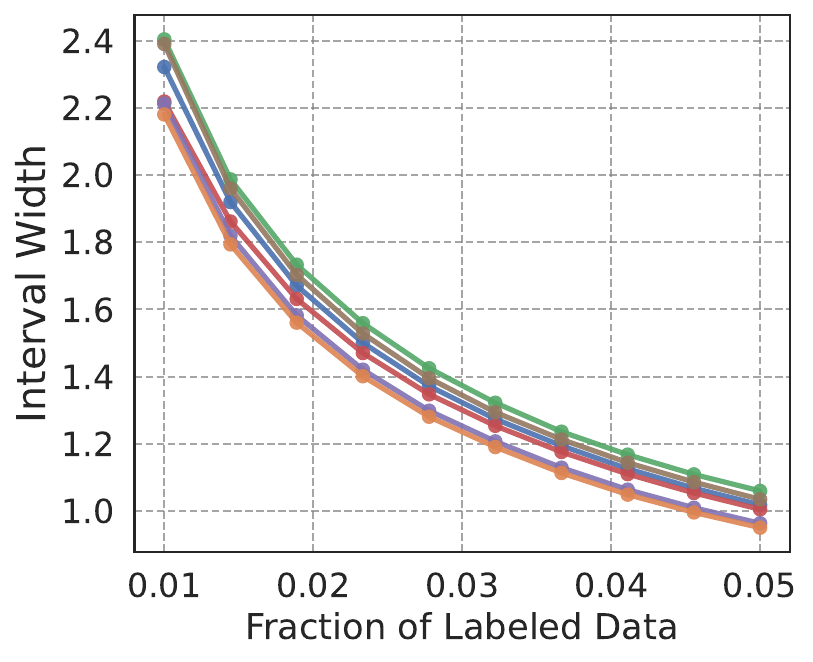}
    \end{subfigure}

    \vspace{0.25cm}

     \begin{subfigure}[b]{0.32\textwidth}
        \includegraphics[width=\textwidth]{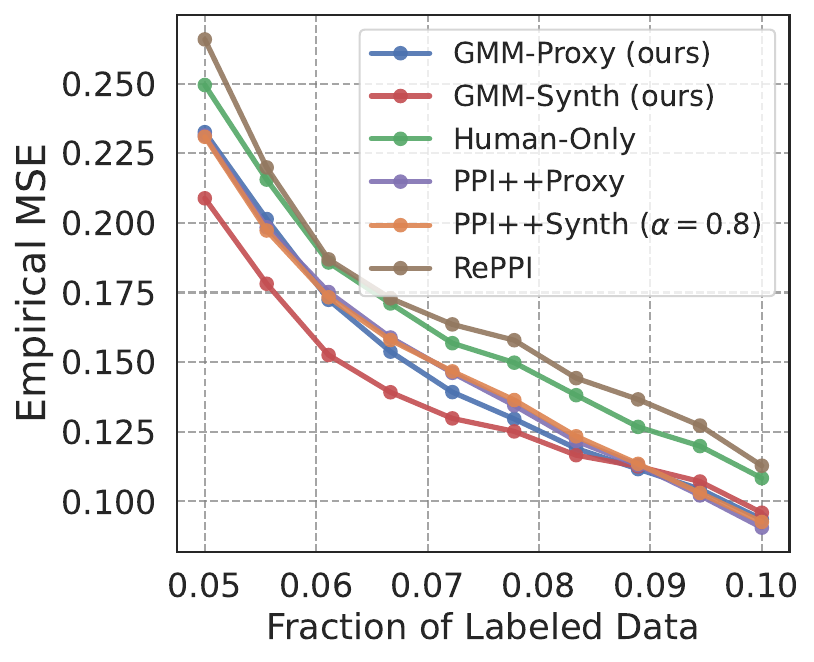}
    \end{subfigure}
    \hfill
    \begin{subfigure}[b]{0.32\textwidth}
        \includegraphics[width=\textwidth]{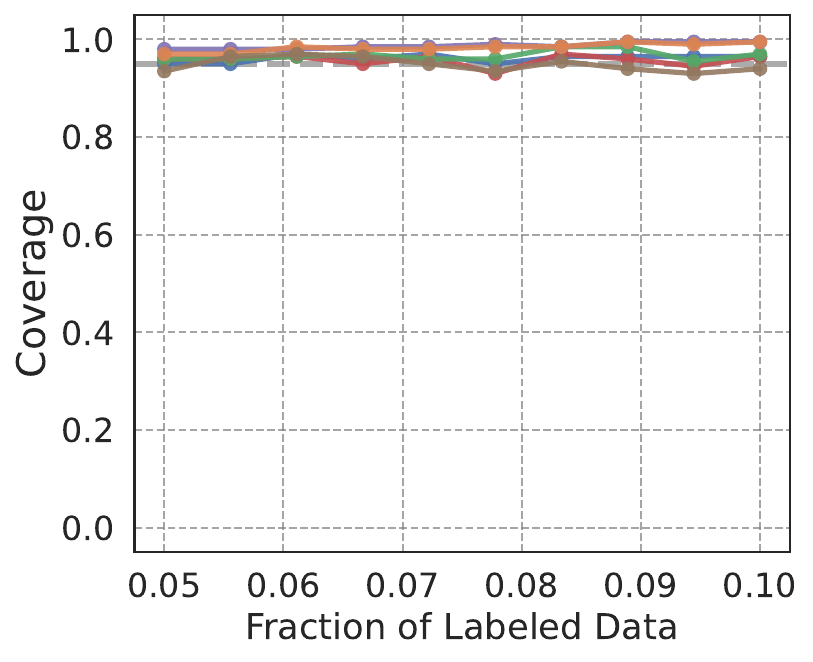}
    \end{subfigure}
    \hfill
    \begin{subfigure}[b]{0.32\textwidth}
        \includegraphics[width=\textwidth]{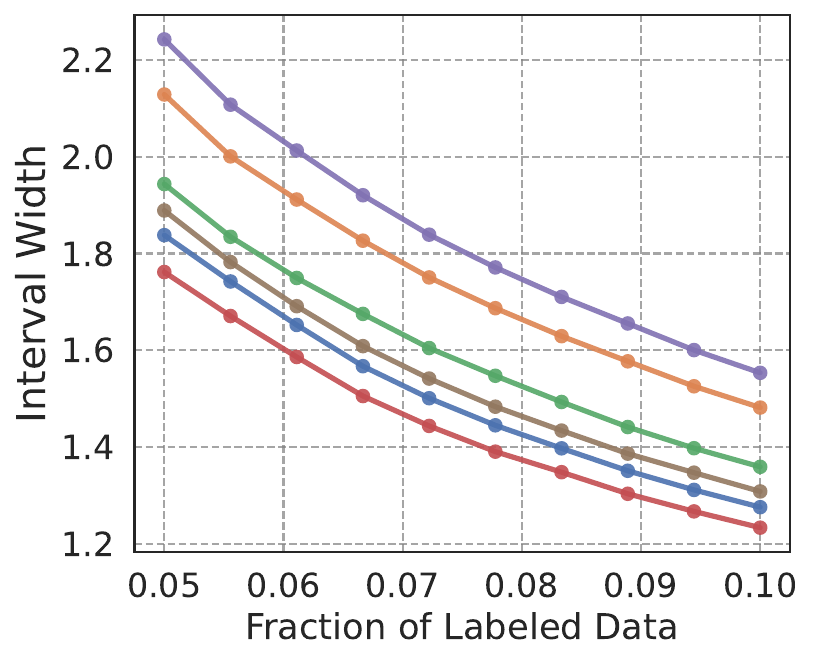}
    \end{subfigure}

    \vspace{0.25cm}

    \begin{subfigure}[b]{0.32\textwidth}
        \includegraphics[width=\textwidth]{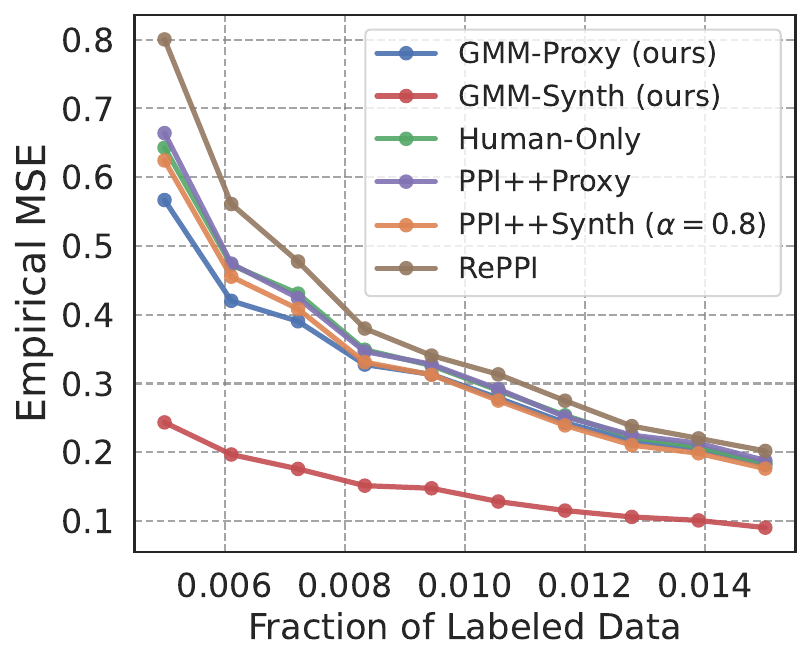}
    \end{subfigure}
    \hfill
    \begin{subfigure}[b]{0.32\textwidth}
        \includegraphics[width=\textwidth]{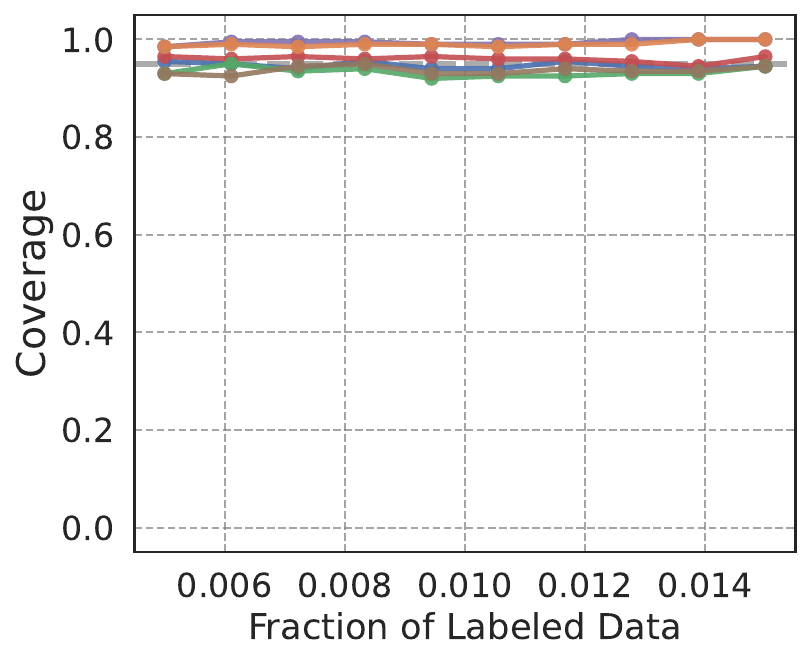}
    \end{subfigure}
    \hfill
    \begin{subfigure}[b]{0.32\textwidth}
        \includegraphics[width=\textwidth]{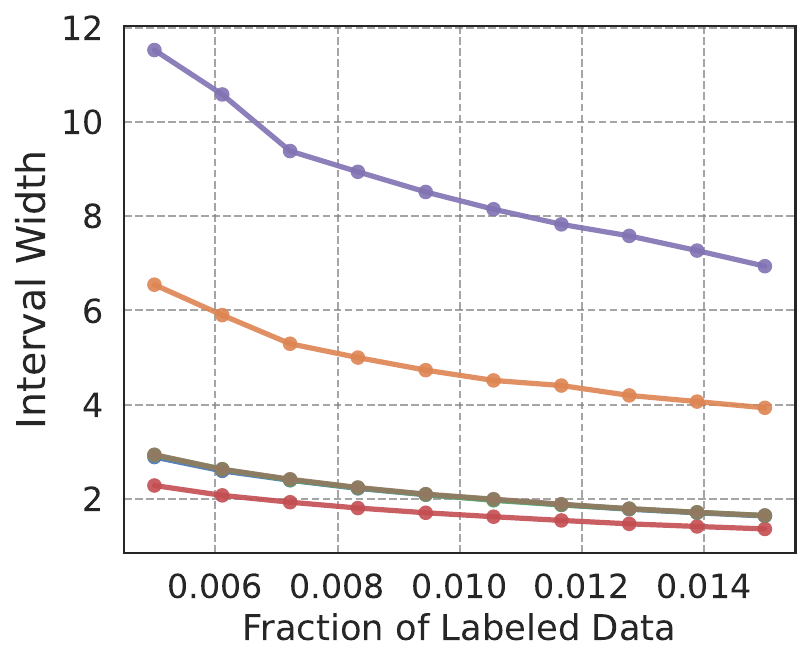}
    \end{subfigure}

    \caption{\textbf{Main Results (Logistic regression)}. We observe large reductions in MSE, especially in very low-label regimes. Each row corresponds to a task (i.e., 1pp, Hedging, Stance, Congressional Bills Data (from top to bottom)); each column corresponds to a metric (i.e., MSE, coverage, confidence interval width (from left to right)). Note that we report the PPI++Synth oracle number for PPI++Synth (see Figure \ref{fig:cross_fitting_lr} for PPI++Synth with cross-fitting results). When the best performing PPI++Synth is equivalent to PPI++Proxy (i.e., $\alpha = 1$), we report the second-best performing PPI++Synth method. See Figure \ref{fig:grid-search-lr} in Appendix \ref{appx:results} for full grid-search results over different $\alpha$ values. Results are averaged over 200 trials.}
    \vspace{-3mm}
    \label{fig:key_results}
\end{figure}

\subsection{Experimental Setup}

\begin{figure} [t]
    \centering
    \begin{subfigure}[b]{0.32\textwidth}
        \includegraphics[width=\textwidth]{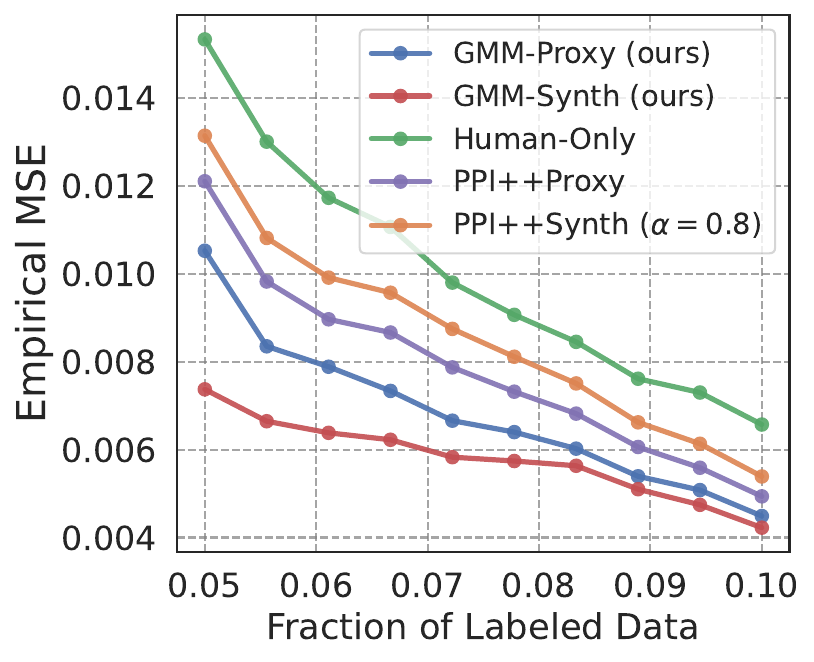}
    \end{subfigure}
    \hfill
    \begin{subfigure}[b]{0.32\textwidth}
        \includegraphics[width=\textwidth]{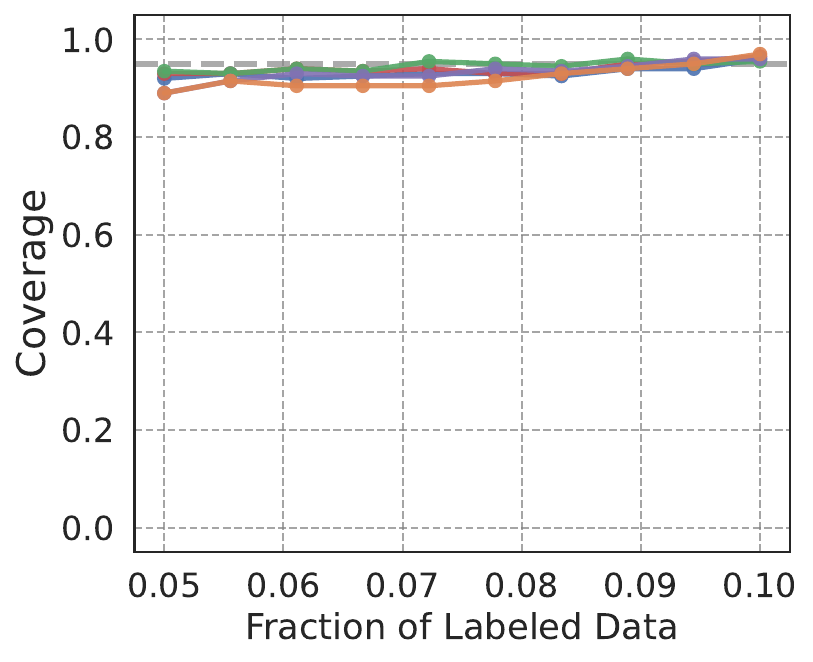}
    \end{subfigure}
    \hfill
    \begin{subfigure}[b]{0.32\textwidth}
        \includegraphics[width=\textwidth]{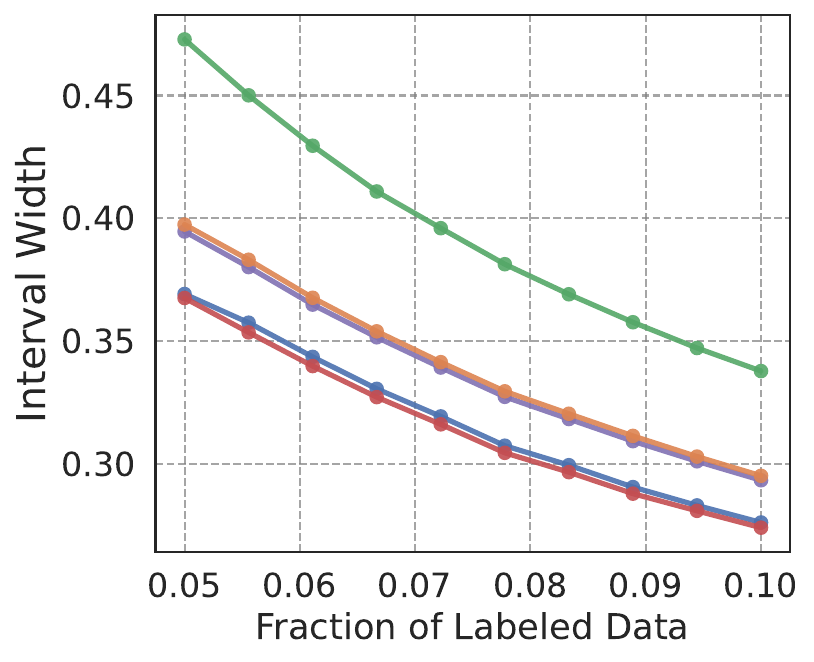}
    \end{subfigure}

    \vspace{0.25cm}
    \begin{subfigure}[b]{0.32\textwidth}
        \includegraphics[width=\textwidth]{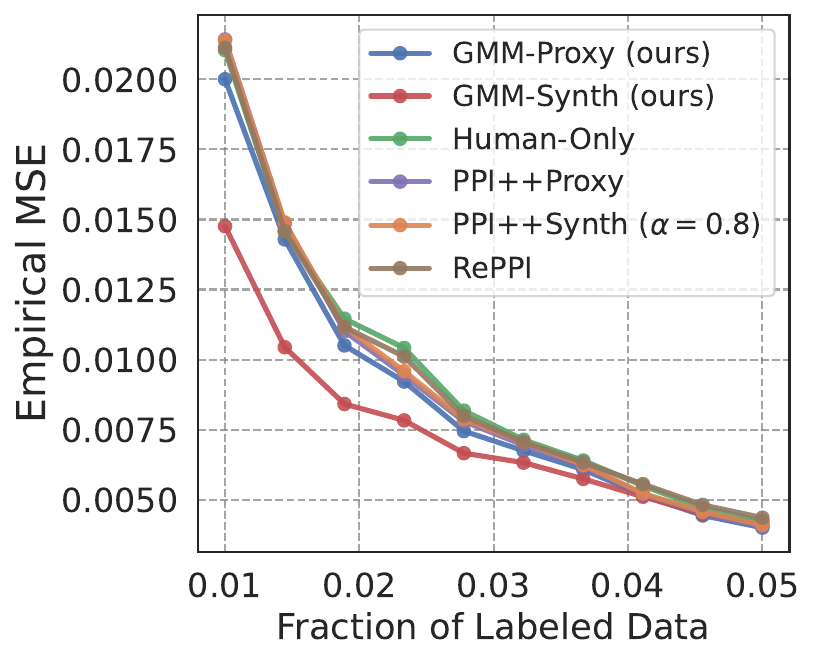}
    \end{subfigure}
    \hfill
    \begin{subfigure}[b]{0.32\textwidth}
        \includegraphics[width=\textwidth]{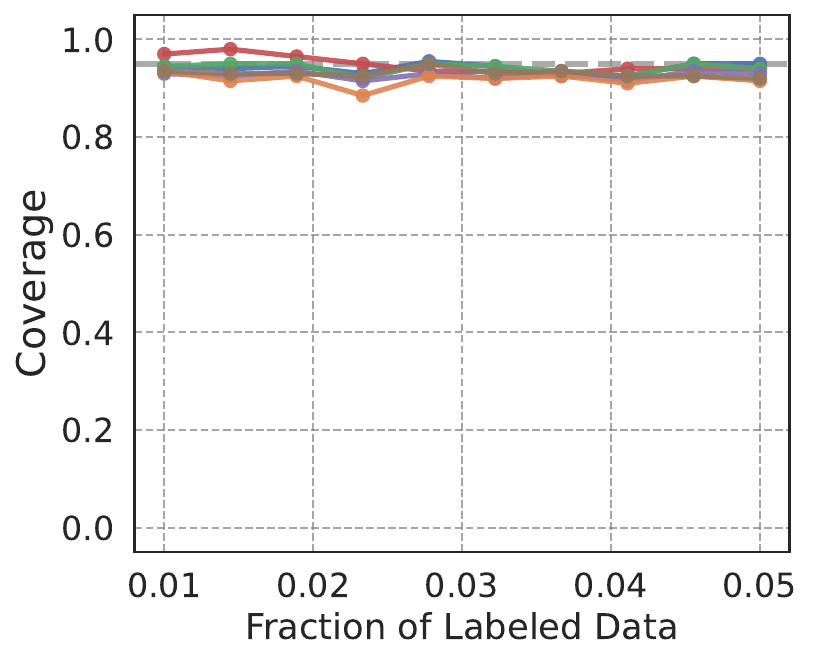}
    \end{subfigure}
    \hfill
    \begin{subfigure}[b]{0.32\textwidth}
        \includegraphics[width=\textwidth]{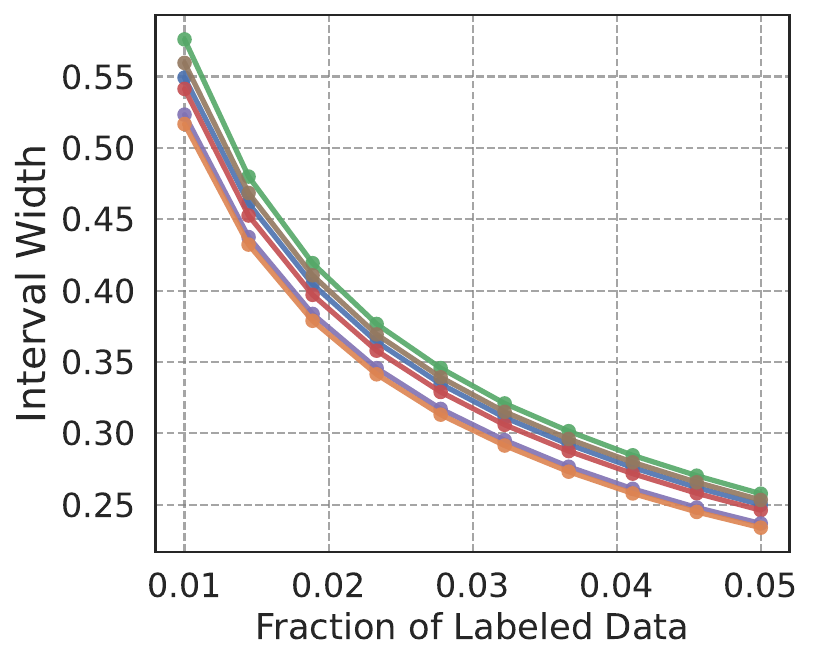}
    \end{subfigure}

    \vspace{0.25cm}

     \begin{subfigure}[b]{0.32\textwidth}
        \includegraphics[width=\textwidth]{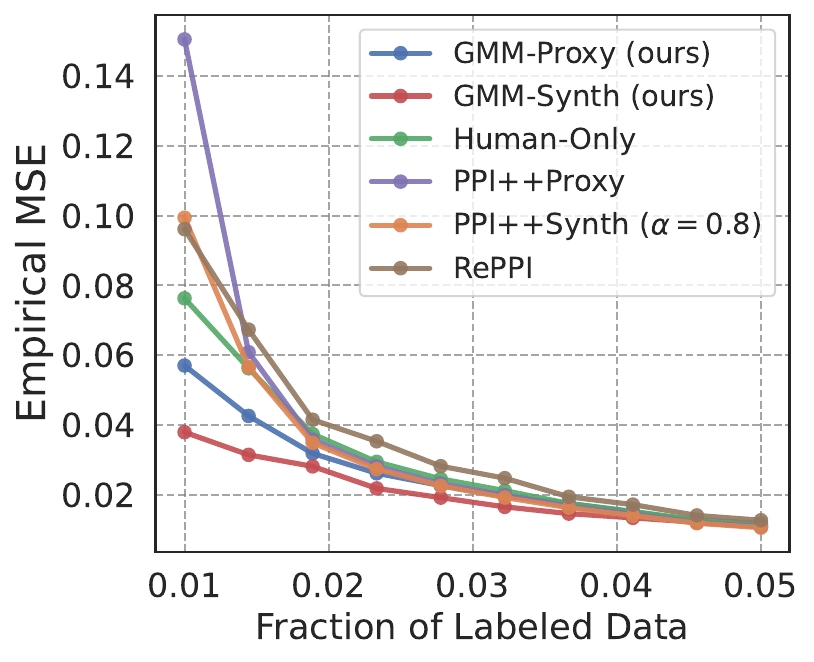}
    \end{subfigure}
    \hfill
    \begin{subfigure}[b]{0.32\textwidth}
        \includegraphics[width=\textwidth]{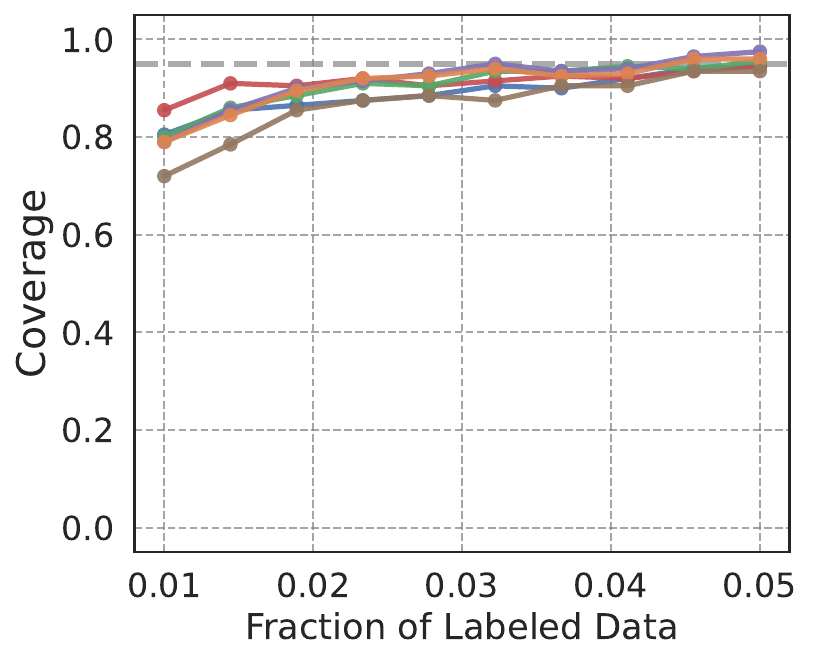}
    \end{subfigure}
    \hfill
    \begin{subfigure}[b]{0.32\textwidth}
        \includegraphics[width=\textwidth]{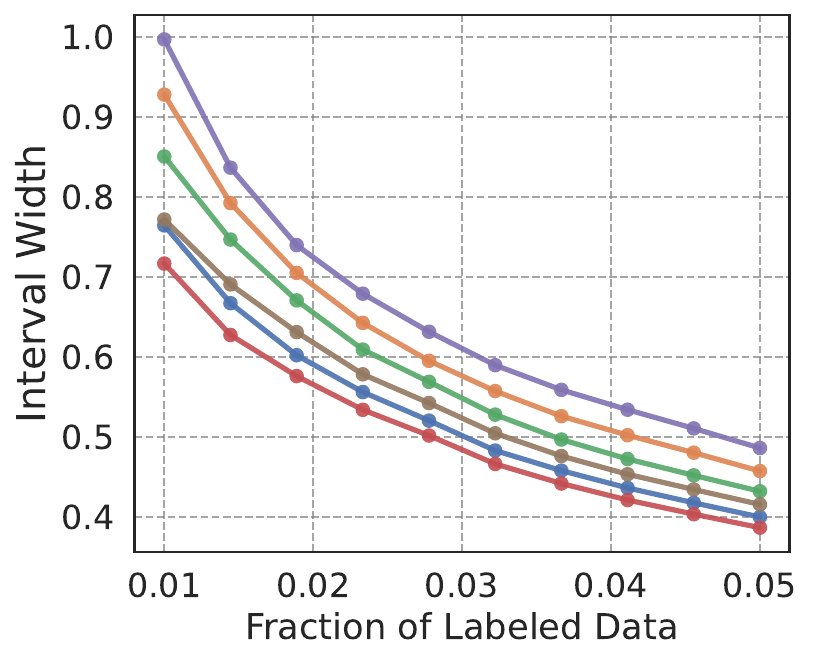}
    \end{subfigure}

    \vspace{0.25cm}

     \begin{subfigure}[b]{0.32\textwidth}
        \includegraphics[width=\textwidth]{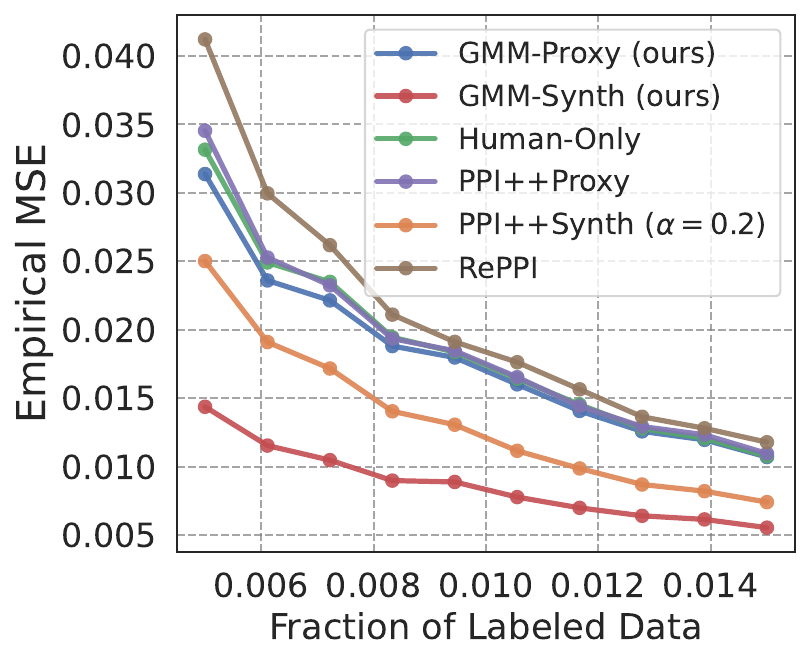}
    \end{subfigure}
    \hfill
    \begin{subfigure}[b]{0.32\textwidth}
        \includegraphics[width=\textwidth]{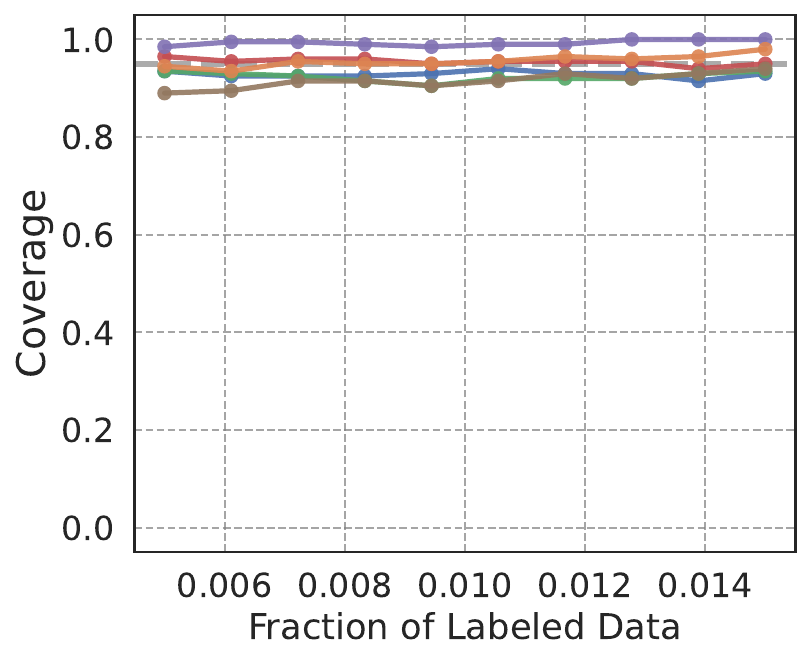}
    \end{subfigure}
    \hfill
    \begin{subfigure}[b]{0.32\textwidth}
        \includegraphics[width=\textwidth]{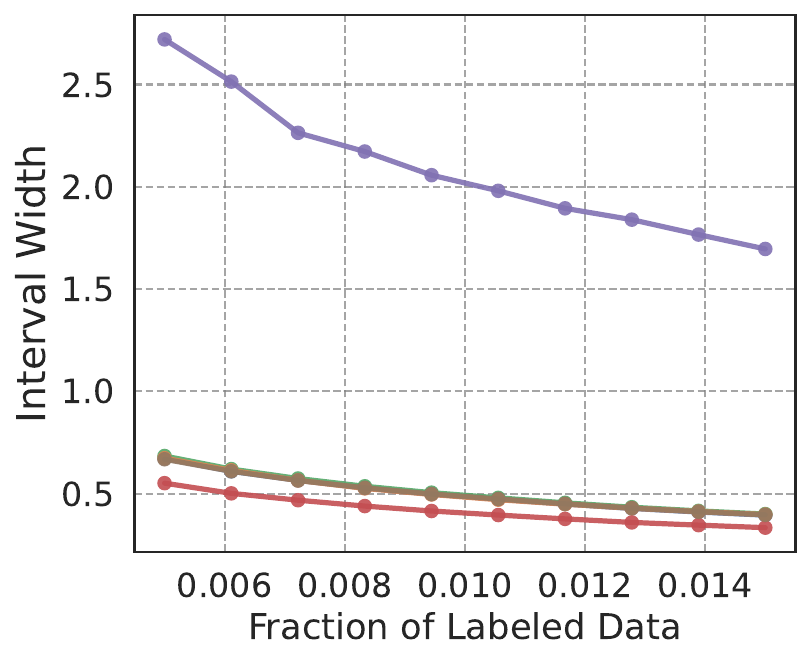}
    \end{subfigure}

    \caption{\textbf{Main Results (OLS)}. We again observe large reductions in MSE, especially in very low-label regimes. Each row corresponds to a task (i.e., 1pp, Hedging, Stance, Congressional Bills Data (from top to bottom)); each column corresponds to a metric (i.e., MSE, coverage, confidence interval width (from left to right)). Note that we report the PPI++Synth oracle number for PPI++Synth (see Figure \ref{fig:cross_fitting_ols} for PPI++Synth with cross-fitting results). When the best performing PPI++Synth is equivalent to PPI++Proxy (i.e., $\alpha = 1$), we report the second-best performing PPI++Synth method. See Figure \ref{fig:grid-search-ols} in Appendix \ref{appx:results} for full grid-search results over different $\alpha$ values. Results are averaged over 200 trials.}
    \vspace{-4mm}
    \label{fig:key_results_ols}
\end{figure}

\paragraph{Datasets.} 
We validate the finite-sample performance of our estimator for logistic regression and ordinary least squares (OLS) regression on the following 4 computational social science tasks: First, we use online requests posted on Stack Exchange and Wikipedia \citep{danescu2013computational} to estimate how the presence of hedging markers (i.e., expressions of uncertainty) affect perceived politeness. Second, we use the same dataset to estimate how the usage of first-person plural pronouns affect perceived politeness. Third, we use a corpus of climate-related news headlines \citep{hmielowski2014attack} to estimate the effect of affirming linguistic devices on media stance toward global warming (i.e., whether the news headline supports or rejects climate change). Lastly,
we use congressional bills texts \citep{adler2011congressional} to estimate the effect of a legislator's DW-Nominate measure \citep{lewis2024congressional} of ideology on the type of bill (whether the bill pertains to macroeconomy). In all tasks, the target parameter is the regression coefficient corresponding to the explanatory variable of interest.

\paragraph{Models and Metrics.}
We use GPT-4o \citep{hurst2024gpt} without any task-specific fine-tuning to generate both proxy and synthetic data. We also include additional results, using open-source, worse quality models (i.e., Llama-3-8b and Qwen-3-8b) in Appendix \ref{appx:results} (Figures \ref{fig:llama_lr}, \ref{fig:llama_ols}, \ref{fig:qwen_lr}, \ref{fig:qwen_ols}). We evaluate our method's performance against the adapted baselines discussed in Section \ref{sec:baselines} using four key metrics: empirical mean-squared error (MSE), coverage, confidence interval width, and effective sample size. The effective sample size represents the number of human-labeled samples that a 
classical estimator would require to achieve the same MSE as our method's estimate. In other words, this metric quantifies how many human annotations the method effectively saves while maintaining equivalent mean squared error. 

\subsection{Results}

The results for our method's performance are shown in Figures $\ref{fig:key_results}$ and \ref{fig:key_results_ols}. We will highlight some key observations here. Observation 1: GMM-Synth achieves the lowest MSE, outperforming all baselines on 8 out of 8 downstream tasks. Notably, performance gains (of more than 50\% reductions in MSE) are most pronounced when the fraction of labeled data is small, precisely the setting where the need for synthetic data is best motivated. Crucially, this does not come at a loss of validity of the parameter estimates; GMM-Synth retains valid coverage across all tasks and results in tighter confidence intervals in 7 out of 8 downstream tasks. In Figures \ref{fig:ess-lr} and \ref{fig:ess-ols} (in Appendix \ref{appx:results}), we further observe that our method substantially improves effective sample size in data-limited settings. In other words, our method effectively saves large amounts of human annotations (up to more than 50\%) while maintaining equivalent mean squared error. Observation 2: 
GMM-Synth consistently exhibits gains over GMM-Proxy across all considered tasks. This demonstrates that synthetic data provides additional benefits beyond those of proxy-labeled data, and that our method effectively integrates these multiple sources while retaining their respective benefits.
In other words, it shows how much additional benefit there is for the practitioner to not only use the model to label unlabeled samples but also use it to generate entirely new unlabeled samples to aid in statistical inference. 
Observation 3: Interestingly, unlike the additional gains demonstrated in GMM-Synth compared to GMM-Proxy, we observe that incorporating synthetic data via PPI++Synth (compared to PPI++Proxy) lead to benefits that are much less pronounced in 3 tasks, and, in fact, result in no gains in MSE in the remaining 5 tasks. This empirically demonstrates that our method is able to incorporate synthetic data much more effectively when compared to adaptations of existing debiasing methods in the literature. Importantly, we note that across all settings, using the proxy data and synthetic data alone yields greatly biased estimates (see Figure \ref{fig:synth-only} in Appendix \ref{appx:results}). Further, we note that the same conclusions hold with even worse quality synthetic data generations from weaker language models; see Appendix \ref{appx:results} for results on open-source models such as LLaMA \citep{grattafiori2024llama} or Qwen \citep{yang2025qwen3}.

\vspace{-2pt}
\section{Discussion}
How synthetic data pipelines should be designed and implemented in practice hinges on reliable mechanisms for integrating information from them.
In this work, we introduce a principled framework for reliably incorporating fully synthetic samples into downstream statistical analyses. We provide practical guidance for constructing synthetic samples from text-based foundation models in ways that support valid inference, and propose a new estimator based on generalized method of moments (GMM) estimation, where the key intuition is that synthetic data will improve performance when the synthetic-data residuals are predictive of the real-data residuals.
Across the studied inferential tasks, we indeed observe a large degree of improvement in estimation, especially in very low-label regimes.
More broadly, this work takes a first step toward understanding how imperfect synthetic data from foundation models can systematically be leveraged to support valid inference and to make reliable downstream conclusions. 
With the increased adoption and growing capabilities of foundation models, pipelines that incorporate their outputs will only become more complex. 
Our method provides an easily extensible estimation framework that can safely integrate the increasing variety and quality of synthetic data sources.

\paragraph{Limitations and future directions.}  A potential limitation of our framework is its reliance on the quality of the generative model (e.g., an LLM), as expected. As with other debiasing approaches, very poor-quality synthetic data would yield little-to-no benefits in statistical efficiency. 
Moreover, our theoretical guarantees, like those of debiasing methods, hold asymptotically and thus may fail to hold in extremely low-data regimes, potentially leading to undercoverage of the target parameter.

\begin{ack}
We thank Michael Oberst and Gati Aher for helpful discussions in developing this work. We also thank Santiago Cortes-Gomez for thoughtful comments on the manuscript. YB was supported in part by the AI2050 program at Schmidt Sciences (Grant G-22-64474).
\end{ack}

\bibliography{main}
\bibliographystyle{plainnat} 

\include{appendix}

\end{document}

%% file: appendix.tex
\appendix

\section{Conditions for Consistency and Asymptotic Normality}\label{appx:conditions}

We provide a discussion about the necessary conditions for a GMM estimator to be consistent and asymptotically normal, showing that these conditions are indeed met for our augmented GMM. 

As mentioned in the construction of our estimator, we define one moment condition for each parameter on the observed data $D$. 
We also define two moments for each parameter on the proxy and synthetic data. 
This leads to an overidentified system, with more moments than parameters, ensuring that the target parameter is identifiable.  

Next, we establish a few conditions for valid asymptotic properties of our GMM estimator, specifically about the convergence and distributions of the stacked vector of the sample moments $g_t(\theta^*, \eta^*)$ at its optimum, which we will refer to as $g_t$ for brevity.
First, we require that this vector of moments converges to its expectation, or that
$$
\frac{1}{N}\sum_{t=1}^{N} g_t \;\to\; \mathbb{E}[g_t],
$$
where $N = n + m$ is our total amount of data. 
Next, all moments must jointly respect the central limit theorem, or that
$$\sqrt{N} \left(\frac{1}{N}\sum_{t=1}^{N} g_t \right) \xrightarrow[]{d}  \mathcal{N}(0, F),$$
where $F$ is some finite covariance matrix of all the moments $g_t$.

Under these standard regularity conditions on the moment vector $g_t$ \citep{newey1994large}, these conditions are immediately satisfied for the moments defined on observed data, as each observation of the moments is independent. The same holds for the moments defined on proxy data, since $\hat{X}, \hat{Y}$ are functions of independent inputs $T$, and are therefore also independent across observations. 
The case of synthetic data is slightly more nuanced, but we show that the required conditions still hold, through the following lemma.

\begin{restatable}{lemma}{synthetic_moments}
\label{lemma:synthetic_moments}
Let \(\{\phi_j \}_{j=1}^{N}\) represent observations of the subset of moments corresponding to parameters of the synthetic data, and assume \(\mathbb{E}\|\phi_j\|^2<\infty\). Then, they are i.i.d., and consequently
\begin{equation*}
  \frac{1}{N}\sum_{j=1}^{N} \phi_j
    \; \xrightarrow{}\;\mathbb{E}[\phi_j]
  \quad\text{and}\quad
  \sqrt{N}\left(\frac{1}{N}\sum_{j=1}^{N} \phi_j\right)
    \;\xrightarrow{d}\;
  \mathcal{N}(0, \Sigma(\phi_j)),    
\end{equation*}
where $\Sigma(\phi_j)$ is the covariance matrix of $\phi_j$. 
\end{restatable}

\begin{proof}
We begin by noting that texts \(\{T_j\}_{j=1}^{N}\) are drawn i.i.d. from the marginal distribution \(\mathcal{D}_T\). For each \(T_j\), a synthetic text \(\tilde T_j\) is produced by a generative model (i.e., by an LLM), which uses independent randomness for each call. The model is conditioned only on an individual sample \((T_j, X_j)\) if \(j\) is labeled or \((T_j, \hat X_j)\) otherwise. Since the generative process for each \(T_j\) is independent and the mapping \(\tilde T_j\mapsto (\tilde X_j,\tilde Y_j)\) is applied identically to each sample, the resulting pairs \((\tilde X_j, \tilde Y_j)\) are also i.i.d. As these pairs are drawn i.i.d., then the conditions are met via the central limit theorem.
\end{proof}

This result shows that the required conditions on the sample moments hold in our setting of proxy and synthetic samples; under the regularity conditions of \citet{newey1994large} Theorem 3.2, one immediately obtains Proposition \ref{prop:asymptotics} on the asymptotic behavior of our GMM estimator.

\section{Asymptotic Efficiency} \label{appx:optimality}
From \citet{chamberlain1987asymptotic}, we know that the lower bound on asymptotic variance among all regular estimators based on the moment restrictions is precisely the asymptotic variance that can be achieved in the general case by using the GMM estimator. Specifically, it achieves the semiparametric efficiency bound--the smallest possible asymptotic variance attainable by any regular estimator using these moment conditions. This corresponds to the local asymptotic minimax risk over all statistical models such that these moment equalities hold, in the sense that for any $\nu\in\mathbb{R}^d$

\begin{equation*}
\underset{N\to\infty}{\lim}\underset{\tilde{\theta}\text{\ measurable}}{\inf}\underset{(\theta,\tilde{\mathcal{D}})\in\Gamma(\theta^*,\mathcal{D})}{\sup}\mathbb{E}_{\tilde{\mathcal{D}}}[[\nu^{T}\sqrt{N}(\tilde{\theta}-\theta)]^2]\geq \nu^{T}G(\theta^*,\eta^*)^{T}F^{-1}G(\theta^*,\eta^*)\nu
\end{equation*}

where $\Gamma(\theta^*,\mathcal{D})$ is any local neighborhood of true parameter $\theta^*$ and data-generating distribution $\mathcal{D}$ satisfying the moment conditions and regularity conditions $C_1$ in \citet{chamberlain1987asymptotic} for parameter $\theta$. 

\begin{proof}
    Follows directly from \citet{chamberlain1987asymptotic} Theorem 2.
\end{proof}

In words, this indicates that no estimator can achieve lower asymptotic variance than GMM uniformly over all local distributions satisfying the same moment conditions.

\section{Moment Conditions}\label{appx:moments}

We provide a concrete example of our moment construction for the case of generalized linear models (GLMs) in two-dimensions. 

\subsection{Example 1. Generalized Linear Models}

Recall that the standard GLM formulation optimizes the objective function,
\begin{align*}
    \ell_\theta(x, y) = -yx^T \theta + f(x^T \theta),
\end{align*}

where $f$ is a function that is convex and infinitely differentiable. We remark that this recovers the setting of logistic regression when $f(z) = \log(1 + \exp({z}))$. Let us assume a two-dimensional setting for illustration. This translates to the population moment conditions of 
\begin{align*}
    \mathbb{E}\left[X_1\left(Y - \frac{\partial f}{\partial \theta_1}(X^T \theta^*)\right)\right] & = 0, \quad 
    \mathbb{E}\left[X_2\left(Y - \frac{\partial f}{\partial \theta_2}(X^T \theta^*)\right)\right]  = 0 
\end{align*}

We have similar moments for proxy and synthetic data, where we use parameters $\eta = (\eta^{(1)}, \eta^{(2)})$, which are also two-dimensional. Within our GMM framework, we construct the following set of moment conditions across the observed, proxy, and synthetic data. 

\begin{align*}
g_t(\theta, \eta) =
\left[
\begin{array}{c}
\vphantom{X_t ( Y_t - \frac{\partial f}{\partial \theta_1}(X_t^T \theta) )} s_t \\
\vphantom{X_t ( Y_t - \frac{\partial f}{\partial \theta_1}(X_t^T \theta) )} s_t \\
\vphantom{\hat{X}_t ( \hat{Y}_t - \frac{\partial f}{\partial \eta^{(1)}_1}(\hat{X}_t^T \eta^{(1)}) )} s_t \\
\vphantom{\hat{X}_t ( \hat{Y}_t - \frac{\partial f}{\partial \eta^{(1)}_1}(\hat{X}_t^T \eta^{(1)}) )} s_t \\
\vphantom{\tilde{X}_t ( \tilde{Y}_t - \frac{\partial f}{\partial \eta^{(2)}_1}(\hat{X}_t^T \eta^{(2)}) )} s_t \\
\vphantom{\tilde{X}_t ( \tilde{Y}_t - \frac{\partial f}{\partial \eta^{(2)}_1}(\hat{X}_t^T \eta^{(2)}) )} s_t \\
\vphantom{\hat{X}_t ( \hat{Y}_t - \frac{\partial f}{\partial \eta^{(1)}_1}(\hat{X}_t^T \eta^{(1)}) )} 1 \\
\vphantom{\hat{X}_t ( \hat{Y}_t - \frac{\partial f}{\partial \eta^{(1)}_1}(\hat{X}_t^T \eta^{(1)}) )} 1 \\
\vphantom{\tilde{X}_t ( \tilde{Y}_t - \frac{\partial f}{\partial \eta^{(2)}_1}(\hat{X}_t^T \eta^{(2)}) )} 1 \\
\vphantom{\tilde{X}_t ( \tilde{Y}_t - \frac{\partial f}{\partial \eta^{(2)}_1}(\hat{X}_t^T \eta^{(2)}) )} 1 \\
\end{array}
\right]
\odot
\left[
\begin{array}{c}
X_{t,1} ( Y_t - \frac{\partial f}{\partial \theta_1}(X_t^T \theta) ) \\
X_{t,2} ( Y_t - \frac{\partial f}{\partial \theta_2}(X_t^T \theta) ) \\
\hat{X}_{t,1} ( \hat{Y}_t - \frac{\partial f}{\partial \eta^{(1)}_1}(\hat{X}_t^T \eta^{(1)}) ) \\
\hat{X}_{t,2} ( \hat{Y}_t - \frac{\partial f}{\partial \eta^{(1)}_2}(\hat{X}_t^T \eta^{(1)}) ) \\
\tilde{X}_{t,1} ( \tilde{Y}_t - \frac{\partial f}{\partial \eta^{(2)}_1}(\tilde{X}_t^T \eta^{(2)}) ) \\
\tilde{X}_{t,2} ( \tilde{Y}_t - \frac{\partial f}{\partial \eta^{(2)}_2}(\tilde{X}_t^T \eta^{(2)}) ) \\
\hat{X}_{t,1} ( \hat{Y}_t - \frac{\partial f}{\partial \eta^{(1)}_1}(\hat{X}_t^T \eta^{(1)}) ) \\
\hat{X}_{t,2} ( \hat{Y}_t - \frac{\partial f}{\partial \eta^{(1)}_2}(\hat{X}_t^T \eta^{(1)}) ) \\
\tilde{X}_{t,1} ( \tilde{Y}_t - \frac{\partial f}{\partial \eta^{(2)}_1}(\tilde{X}_t^T \eta^{(2)}) ) \\
\tilde{X}_{t,2} ( \tilde{Y}_t - \frac{\partial f}{\partial \eta^{(2)}_2}(\tilde{X}_t^T \eta^{(2)}) ) \\
\end{array}
\right]
\end{align*}

\section{Partitioned GMM Asymptotic Variance}\label{appx:partitioned}

We now derive the asymptotic variance of our GMM estimator for specifically the target parameter $\hat{\theta}_T$.

\targetvariance*

\begin{proof}
    With the optimal choice of weight matrix for the full GMM estimation problem, the asymptotic variance of the vector $(\hat{\theta},\hat{\eta})$ converges to 
$(G^{T}F^{-1}G)^{-1}$.
To obtain the variance for $\hat{\theta}$ specifically, partition the moments into $g_t(\theta,\eta)=(m_t(\theta)^{\prime},h_t(\eta)^{\prime})^\prime$, where  $m_t(\theta)=S_t\odot \psi(\theta)$, and

\begin{equation*}
h_t(\eta) = \left[
\begin{array}{c}
S_t \\
S_t \\
\vdots \\
S_t \\
1 \\
\vdots \\
1
\end{array}
\right] \odot
\left[
\begin{array}{c}
\psi(\eta^{(1)}) \\
\vdots \\
\psi(\eta^{(M)}) \\
\psi(\eta^{(1)}) \\
\vdots \\
\psi(\eta^{(M)})
\end{array}
\right]
\end{equation*}

Given this partitioning, we can express

$$G(\theta,\eta)=\left[
\begin{array}{cc}
\frac{d \mathbb{E}[m(\theta)]}{d\theta} & 0 \\
0 & \frac{d \mathbb{E}[h(\eta)]}{d\eta}
\end{array}
\right]$$

$$F=\left[
\begin{array}{cc}
\mathbb{E}[m_t(\theta)m_t(\theta)^{\prime}] & \mathbb{E}[m_t(\theta)h_t(\eta)^{\prime}] \\
\mathbb{E}[h_t(\eta)m_t(\theta)^{\prime}] & \mathbb{E}[h_t(\theta)h_t(\theta)^{\prime}]
\end{array}
\right]$$

By the partitioned inverse formula, we can express $F^{-1}$ as 

$$\left[
\begin{array}{cc}
A & B \\
B^{\top} & D
\end{array}
\right]$$

where the upper left block $A$ is 

$$(\mathbb{E}[m_t(\theta)m_t(\theta)^{\prime}]-\mathbb{E}[m_t(\theta)h_t(\eta)^{\prime}]\mathbb{E}[h_t(\theta)h_t(\theta)^{\prime}]^{-1}\mathbb{E}[h_t(\eta)m_t(\theta)^{\prime}])^{-1}$$

This term can be interpreted as the inverse of the asymptotic residual variance of a regression of $m_t(\theta)$ on the span of the vector $h_t(\eta)$. 

The lower right block $D$ is, symmetrically, the asymptotic residual variance of a regression of $h_t(\theta)$ on the span of the vector $m_t(\eta)$: 

$$(\mathbb{E}[h_t(\theta)h_t(\theta)^{\prime}]-\mathbb{E}[h_t(\theta)m_t(\eta)^{\prime}]\mathbb{E}[m_t(\theta)m_t(\theta)^{\prime}]^{-1}\mathbb{E}[m_t(\eta)h_t(\theta)^{\prime}])^{-1}$$

Finally, the off-diagonal term multiplies $A$ by the coefficient in a regression of $m$ on $h$: 

$$B=-A\mathbb{E}[m_t(\theta)h_t(\eta)^{\prime}]\mathbb{E}[h_t(\theta)h_t(\theta)^{\prime}]^{-1}$$

For the full variance,

$$G^{\top}F^{-1}G=\left[
\begin{array}{cc}
\frac{d \mathbb{E}[m(\theta)]}{d\theta^\prime}A\frac{d \mathbb{E}[m(\theta)]}{d\theta} & \frac{d \mathbb{E}[m(\theta)]}{d\theta^\prime}B\frac{d \mathbb{E}[h(\eta)]}{d\eta} \\
\frac{d \mathbb{E}[h(\eta)]}{d\eta^\prime}B^{\top}\frac{d \mathbb{E}[m(\theta)]}{d\theta} & \frac{d \mathbb{E}[h(\eta)]}{d\eta^\prime}D\frac{d \mathbb{E}[h(\eta)]}{d\eta}
\end{array}
\right]$$

Applying the partitioned inverse formula again, the upper left block of $(G^{\top}F^{-1}G)^{-1}$, which gives exactly the asymptotic variance of $\sqrt{T}(\hat{\theta}_T-\theta)$, is equal to 
{\small
$$(\frac{d \mathbb{E}[m(\theta)]}{d\theta^\prime}A\frac{d \mathbb{E}[m(\theta)]}{d\theta}-\frac{d \mathbb{E}[h(\eta)]}{d\eta^\prime}B^{\top}\frac{d \mathbb{E}[m(\theta)]}{d\theta}(\frac{d \mathbb{E}[h(\eta)]}{d\eta^\prime}D\frac{d \mathbb{E}[h(\eta)]}{d\eta})^{-1}\frac{d \mathbb{E}[m(\theta)]}{d\theta^\prime}B\frac{d \mathbb{E}[h(\eta)]}{d\eta})^{-1}$$
}
This can be interpreted similarly as the inverse of the asymptotic variance of the residual prediction error from a regression of $A^{-1/2}\frac{d m(\theta)}{d\theta}$ onto the span of a weighted linear combination of terms in $\frac{d h(\eta)}{d\eta}$. 
\end{proof}

We remark that a lower bound on the total variance is given by $(\frac{d \mathbb{E}[m(\theta)]}{d\theta^\prime}A\frac{d \mathbb{E}[m(\theta)]}{d\theta})^{-1}$, which is minimized when $A$ is maximized. Among choices of moment functions $h_t(\eta)$ that depend solely on $T_t$, $A$ is maximized in the positive semi-definite order when the span of $h_t(\eta)$ contains $\mathbb{E}[m(\theta)|T_t]$. A sufficient but not necessary condition for this is that for some $j\in 1\ldots M$, the conditional moments of the simulation are identical to those of the real data:
$$E[\psi(\eta_j)|T_i]=E[\psi(\theta)|T_i]$$ 

This calibration condition is satisfied when the conditional distribution of the simulated data given $T$ equals that of the real data, which is a natural simulation target, though not required for valid inference.

\section{Experimental Details}\label{appx:experiments}

\subsection{Baseline Details}
\subsubsection{RePPI Implementation}
\label{appx:reppi}

In adapting RePPI \citep{ji2025predictions} to our setting, we can model the imputed loss function in PPI with a ML-based approach. While in their paper, they choose a particular form of $$g_\theta(X, \hat{Y}) = \frac{1}{1 + r}\theta^T s^*(X, \hat{Y}),$$ 
where $s^*$ is the conditional score function. In our setting, we do not have access to $X$ on unlabeled instances, meaning that our model of the conditional score must take in inputs of $s^*(\hat{X}, \hat{Y}, \tilde{X}, \tilde{Y})$. In the case for GLMs and if we have access to unlabeled instances $X$, we know that the score is given by $\nabla \ell_\theta (X, Y) = X (f'(X^T \theta) - \mathbb{E}[Y | X, \hat{Y}])$, where $f$ is as defined in Section \ref{sec:glm}. In that setting, we would only need to model $\mathbb{E}[Y | X, \hat{Y}]$. However, in our setting where we only have access to proxy and synthetic data, we need to directly model $\nabla \ell_\theta(X, Y)$ or try to learn $\mathbb{E}[\nabla \ell_\theta(X, Y) | \hat{X}, \hat{Y}, \tilde{X}, \tilde{Y}]$, as we do not observe $X$ to use in our predictions on unlabeled data.

We note that in our experiments, the sample splitting approach proposed in RePPI performs poorly due to cross-fitting; there simply is not enough data to accurately estimate (i) the ground truth parameter on one fold, (ii) learn the ML model on the second fold, and (iii) have an accurate target parameter estimate on the final fold, given each split has a size of only $\frac{1}{3}$ of the number of labeled data. We observe similar poor behavior when doing PPI++ Synth Cross-Fitting with limited data, but RePPI is even more intensive, given that we need to do 3 splits of data rather than 2.
As such, to learn the ML model for the imputed loss in RePPI, we choose to adopt a linear regression model (defined over a small number of covariates), which satisfies the required Donsker conditions to enable us to avoid any requirements on sample splitting as in standard DML approaches \citep{van1996weak, chernozhukov2018double}. At a high level, a linear regression model is sufficiently simple that it cannot overfit to noise in the original data, which is the main goal sample splitting aims to address.

Therefore, our implementation of RePPI is as follows: we (1) fit $\hat{\theta}$ by optimizing the human-only loss, (2) we optimize an linear regression model that learns to map $h: (\hat{X}, \hat{Y}, \tilde{X}, \tilde{Y}) \to \nabla \ell_{\hat{\theta}}(X, Y)$, and (3) we perform power tuning and produce our parameter estimate by minimizing the imputed loss that incorporates $\hat{\theta}$ and $h$, all on the full available data. We use the same linear regression in estimating the imputed loss; the exception to this is on Congressional Bills, where we use XGBoost as linear regression performs very poorly in estimating the score function.

\begin{restatable}{proposition}{reppi}
\label{prop:reppi}
The RePPI objective with multiple predicted covariates and outcomes is given by
\begin{align}
    L^{\text{RePPI}}(\theta) & := \frac{1}{n}\sum_{i=1}^{n} \ell_\theta (X_{i}, Y_{i}) - \left(\frac{1}{n}\sum_{i=1}^n g_\theta(\hat{X}_i, \hat{Y}_i, \tilde{X}_i, \tilde{Y}_i) - \frac{1}{N}\sum_{i=1}^N g_\theta(\hat{X}_i, \hat{Y}_i, \tilde{X}_i, \tilde{Y}_i)\right).
\end{align}
where 
\begin{align*}
    g_\theta(\hat{X}_i, \hat{Y}_i, \tilde{X}_i, \tilde{Y}_i) = \frac{1}{1 + r} \theta^T s^*(\hat{X}_i, \hat{Y}_i, \tilde{X}_i, \tilde{Y}_i), \; s^*(\hat{X}_i, \hat{Y}_i, \tilde{X}_i, \tilde{Y}_i) = \mathbb{E}[\nabla \ell_{\theta^*}(X, Y)| \hat{X}_i, \hat{Y}_i, \tilde{X}_i, \tilde{Y}_i],
\end{align*}
and $\theta^*$ is the target parameter estimate. The resulting estimate retains asymptotic normality conditions.
\end{restatable}

\subsubsection{PPI++Proxy and PPI++Synth Implementation} \label{appx:ppi-multi}

We now present a discussion on our adapted debiasing-based approach from Proposition \ref{prop:ppi-multiple}. 

\begin{restatable}{proposition}{ppimulticov}
\label{prop:ppi-multiple}
The adapted PPI++ objective with multiple predicted covariates and outcomes is given by
\begin{align}
    L^{PP}(\theta) & := \frac{1}{N}\sum_{i=1}^N [(1-\alpha) \cdot \ell_{\theta}(\tilde{X}_{i}, \tilde{Y}_i) + \alpha \cdot \ell_{\theta}(\hat{X}_{i}, \hat{Y}_i)] \\
    & + \frac{1}{n}\sum_{i=1}^{n} (\ell_\theta (X_{i}, Y_{i}) -  [(1 - \alpha) \cdot \ell_{\theta}(\tilde{X}_{i}, \tilde{Y}_i) + \alpha \cdot \ell_{\theta}(\hat{X}_{i}, \hat{Y}_i)]).
\end{align}
where the estimate retains asymptotic normality conditions (see Appendix \ref{appx:ppi-multi} for the proof and algorithm details).
\end{restatable}

\paragraph{Asymptotic Normality}

First, it is relatively straightforward to show that this is an unbiased estimate of the true objective.
\begin{align*}
    \mathbb{E}[L^{PP}(\theta)] & = (1-\alpha) \cdot \mathbb{E}[\ell_{\theta}(\tilde{X}, \tilde{Y})] + \alpha \cdot \mathbb{E}[\ell_{\theta}(\hat{X}, \hat{Y})] \\
    & \quad \; + \; \mathbb{E}[\ell_\theta (X, Y)] -  \mathbb{E}[(1 - \alpha) \cdot \ell_{\theta}(\tilde{X}, \tilde{Y})] - \alpha \cdot \mathbb{E}[\ell_{\theta}(\hat{X}, \hat{Y})])] \\
    & = \mathbb{E}[\ell_\theta(X, Y)].
\end{align*}
Note that this holds for any choice of the hyperparameter $\alpha$. 

Under the same assumptions as in the PPI++ paper \citep{angelopoulos2023ppi++} (e.g., that $\frac{n}{n+m} \to c$ for some constant $c$ and, in the case of generalized linear models, the Hessian is non-singular, we perform their same approach to power tuning), we recover the asymptotic normality guarantees of the parameter estimate (as in Corollary 1 from \citet{angelopoulos2023ppi++}).

\paragraph{Hyperparameter Selection via Cross-fitting}

The added complexity from these modified debiasing-based approaches arises from the hyperparameter $\alpha$. We now discuss an approach for selecting $\alpha$ by performing cross-fitting. As previously mentioned, we can treat $\alpha$ as a simple version of RePPI \citep{ji2025predictions} where we fit a convex combination of proxy and synthetic losses.

Namely, we partition our available data into two splits. We select $\alpha$ on one fold by minimizing:
\begin{align*}
    \arg\min_{\alpha \in [0, 1]} L^{PP}(\theta_1),
\end{align*}
where $\theta_1$ is defined as the solution to the naive minimzation of $\mathbb{E}[\ell_\theta(X, Y)]$ on the same split. This essentially captures picking the $\alpha$ that best combines the proxy and synthetic losses to best mimic the behavior of the standard loss function.

We then take this optimal $\alpha$ and use it to produce a parameter estimate on the held-out fold. We aggregate these estimates as is standard in cross-fitting approaches. We outline this process in Algorithm \ref{alg:ppi-crossfit}.

\begin{algorithm}[t]
\caption{Cross‐Fitting for PPI\textsuperscript{++}Synth}
\label{alg:ppi-crossfit}
\begin{algorithmic}[1]
  \Require \\
    Labeled data 
      \(\mathcal{D}=\{(T_i,X_i,Y_i)\}_{i=1}^n\),\\
    Proxy data 
      \(\widehat{\mathcal{D}}=\{(T_j,\widehat X_j,\widehat Y_j)\}_{j=1}^{n+m}\),\\
    Synthetic data 
      \(\widetilde{\mathcal{D}}=\{(\widetilde T_j,\widetilde X_j,\widetilde Y_j)\}_{j=1}^{n+m}\),\\
    K folds
  \Ensure Debiased estimate \(\hat\theta_{\mathrm{CF}}\)
  \State Split \(\mathcal{D}\) into folds \(\{\mathcal{I}_1,\dots,\mathcal{I}_K\}\) \\
  \For{\(k=1,\dots,K\)}
    \State define train‐fold \(\mathcal{I}_{\mathrm{train}}=\bigcup_{r\neq k}\mathcal{I}_r\)
    \State $
      \hat\theta_1^{-k}
      \;\gets\;
      \arg\min_{\theta}
      L_{\mathrm{PP}}^{-k}(\theta;0)
    $\Comment{(1) initial fit on train‐fold} \\
    \State $
      \hat\alpha^{-k}
      \;\gets\;
      \arg\min_{\alpha\in[0,1]}
      L_{\mathrm{PP}}^{-k}\bigl(\hat\theta_1^{-k};\alpha\bigr)
    $ \Comment{(2) select mixture weight \(\alpha\) on train‐fold)} \\
    \State $
      \hat\theta^{k}
      \;\gets\;
      \arg\min_{\theta}
      L_{\mathrm{PP}}^{k}\bigl(\theta;\hat\alpha^{-k}\bigr)
    $ \Comment{(3) final fit on held‐out fold with chosen \(\alpha\))} \\
  \EndFor
  \State \Return 
    \(\displaystyle
      \hat\theta_{\mathrm{CF}}
      = \frac{1}{K}\sum_{k=1}^K \hat\theta^{k}
    \)
\end{algorithmic}
\end{algorithm}

\subsection{Prompt Texts}

We present the full text prompts that were used to generate proxy covariates and labels (for the proxy data) and synthetic data. Note that the prompts used to extract covariates and labels from the synthetic text are identical to those used for the proxy data.

\begin{tcolorbox}[title=Proxy Data Generation Prompts]
\textbf{Politeness (First Plural Pronouns) - Covariates:}\\
Does the following text contain first person plural pronouns (e.g., we, us, our, ourselves)? Output either yes or no.\\
Text:  
\texttt{"""\\\{content\}\\ """}  \\
\textbf{Answer:} 
\vspace{1em}

\textbf{Politeness (First Plural Pronouns) - Labels:}\\
Is the following text polite? Output either A or B. Output a letter only. \\
A) Polite\\
B) Impolite\\
Text:  
\texttt{"""\\\{content\}\\ """}  \\
\textbf{Answer:} 

\vspace{2em}

\textbf{Politeness (Hedging) - Covariates:}\\
Does the following text contain hedging devices---expressions that indicate uncertainty, caution, or a lack of full commitment to a claim (e.g., may, might, could, would, possibly, probably, perhaps, apparently, suggest, indicate, seem, appear, it is likely that, it seems that)?
Respond with yes or no only.\\
Text:  
\texttt{"""\\\{content\}\\ """}  \\
\textbf{Answer:} 
\vspace{1em}

\end{tcolorbox}

\begin{tcolorbox}[title=Proxy Data Generation Prompts (continued)]

\textbf{Politeness (Hedging) - Labels:}\\
Is the following text polite? Output either A or B. Output a letter only. \\
A) Polite\\
B) Impolite\\
Text:  
\texttt{"""\\\{content\}\\ """}  \\
\textbf{Answer:} 

\vspace{2em}

\textbf{Stance Dataset - Covariates:}\\
Does the following text contain any affirmative device words?
Output either yes or no.\\
Text:  
\texttt{"""\\\{content\}\\ """}  \\
\textbf{Answer:} 

\textbf{Stance Dataset - Labels:}\\
A statement can agree, be neutral, or disagree with the statement: “Climate change/global warming is a serious concern”. Classify the following statement into one of the three categories. Output either A, B, or C. Output a letter only. \\
A) Agree\\
B) Neutral\\
C) Disagree \\
Statement:  
\texttt{"""\\\{content\}\\ """}  \\
\textbf{Answer:} 

\vspace{2em}

\textbf{Congressional Bills Dataset - Covariates:}\\
You are a political scientist familiar with the U.S. Congress and the DW-NOMINATE scoring system, which places legislators and legislation on a left-right ideological spectrum ranging approximately from -1 (most liberal) to +1 (most conservative).
Below is the text of a proposed bill. Based on the policy content, language, and framing of the bill, estimate the DW-NOMINATE score that best represents its ideological position. Output a single nonzero float between -1 and +1 representing the estimated DW-NOMINATE score of the bill.

Bill: 
\texttt{"""\\\{content\}\\ """}  \\
\textbf{Answer:} 

\vspace{1em}

\textbf{Congressional Bills Dataset - Labels:}\\
Does the following text relate to the economy? Output either true or false.

Text: \texttt{"""\\\{content\}\\ """}  \\
\textbf{Label:}

\end{tcolorbox}

\begin{tcolorbox}[title=Synthetic Data Generation Prompts, colback=gray!5, colframe=black, fonttitle=\bfseries]
\textbf{Politeness (First Plural Pronouns)}\\
Consider texts taken from user requests on Stack Exchange or Wikipedia.  
Each text is labeled as either polite or impolite, and either contains or does not contain first-person plural pronouns. Below is an example that \texttt{\{x\}}:  \\
\textbf{Example:}  
\texttt{"""\\\{example\}\\ """}  

Now, generate a new example of a request that also \texttt{\{x\}}.  

\vspace{1em}
\textbf{Politeness (Hedging)}\\
Consider texts taken from user requests on Stack Exchange or Wikipedia. Each text can be labeled as either polite or impolite, and as either containing a hedging device or not containing one. 
Hedging devices are expressions that indicate uncertainty, caution, or a lack of full commitment to a claim (e.g., may, might, could, would, possibly, probably, perhaps, apparently, suggest, indicate, etc.). 
Below is an example that \texttt{\{x\}}:  \\
\textbf{Example:}  
\texttt{"""\\\{example\}\\ """}  

Now, generate a new example of a request that also \texttt{\{x\}}.

\vspace{1em}

\end{tcolorbox}

\begin{tcolorbox}[title=Synthetic Data Generation Prompts (continued), colback=gray!5, colframe=black, fonttitle=\bfseries]
\textbf{Stance}  \\

Consider news headlines that take a stance — agree, disagree, or neutral — on the statement:
“Climate change/global warming is a serious concern."

Each headline also either contains or does not contain an affirmative device.

Below is an example of a headline.\\
\textbf{Example:}  
\texttt{"""\\\{example\}\\ """}  \\
Affirmative device: \texttt{\{x\}}

Now, generate a new news headline about global warming that also \texttt{\{x\}}.

\vspace{1.5em}
\textbf{Congressional Bills Data}\\
You are a political language model trained to generate realistic examples of U.S. congressional bills.  
Each bill is labeled as either ``related to the economy" or ``not related to the economy", and is associated with a DW-NOMINATE score representing ideological position (ranging from $-1$ liberal to $+1$ conservative).

\textbf{Example:}  \\
Bill Text:  
\texttt{"""\\\{example\}\\ """}  \\
DW-NOMINATE Score: \texttt{\{dw\_nominate\_score\}}

Now, generate a new example of a bill that also has a DW-NOMINATE score of \texttt{\{dw\_nominate\_score\}}.  
Output only the new bill text:  
\texttt{"""}
\end{tcolorbox}

\section{Additional Results} \label{appx:results}

We present additional experimental results consisting of: 
\begin{itemize}
    \item Prediction accuracy of GPT-4o, Llama-3-8b, Qwen-3-8b (Think) for the covariates and outcomes of interest (Table \ref{tab:accuracy})
    \item Performance of a naive estimator that \textit{only} uses synthetic data (Figure \ref{fig:synth-only})
    \item Effective sample size results (Figures \ref{fig:ess-lr} and \ref{fig:ess-ols})
    \item Grid search results for PPI++Synth (Oracle) across different $\alpha$ values (Figures \ref{fig:grid-search-lr} and \ref{fig:grid-search-ols})
    \item Cross-fitting results for PPI++Synth (Figures \ref{fig:cross_fitting_lr} and \ref{fig:cross_fitting_ols})
    \item Llama-3-8b results for logistic regression (Figure \ref{fig:llama_lr}) and OLS (Figure \ref{fig:llama_ols})
    \item Qwen-3-8b results for logistic regression (Figure \ref{fig:qwen_lr}) and OLS (Figure \ref{fig:qwen_ols})
    
\end{itemize}

\begin{table}[h]
\centering
\caption{Accuracy of LLMs for prediction tasks across datasets. This represents the quality of the proxy covariates and proxy labels ($\hat{X}, \hat{Y}$).}
\begin{tabular}{l cc cc cc cc}
\toprule
& \multicolumn{2}{c}{Hedging} & \multicolumn{2}{c}{1pp} & \multicolumn{2}{c}{Stance} & \multicolumn{2}{c}{Congressional Bills} \\
\cmidrule(lr){2-3} \cmidrule(lr){4-5} \cmidrule(lr){6-7} \cmidrule(lr){8-9}
\textbf{Model} & X Pred. & Y Pred. & X Pred. & Y Pred. & X Pred. & Y Pred. & X Pred. & Y Pred. \\
\midrule
GPT-4o  & 0.764 & 0.785 & 0.994 & 0.785 & 0.864 & 0.743 & 0.184 & 0.827 \\
Llama-3-8b & 0.485 & 0.335 & 0.878 & 0.501 & 0.619 & 0.637 & 0.688 & 0.564 \\
Qwen-3-8b (Think) & 0.723 & 0.325 & 0.961 & 0.325 & 0.904 & 0.684 & 0.253 & 0.757 \\
\bottomrule
\end{tabular}
\label{tab:accuracy}
\end{table}

\begin{figure}[t]
    \centering
    \begin{subfigure}[b]{0.32\textwidth}
        \includegraphics[width=\textwidth]
        {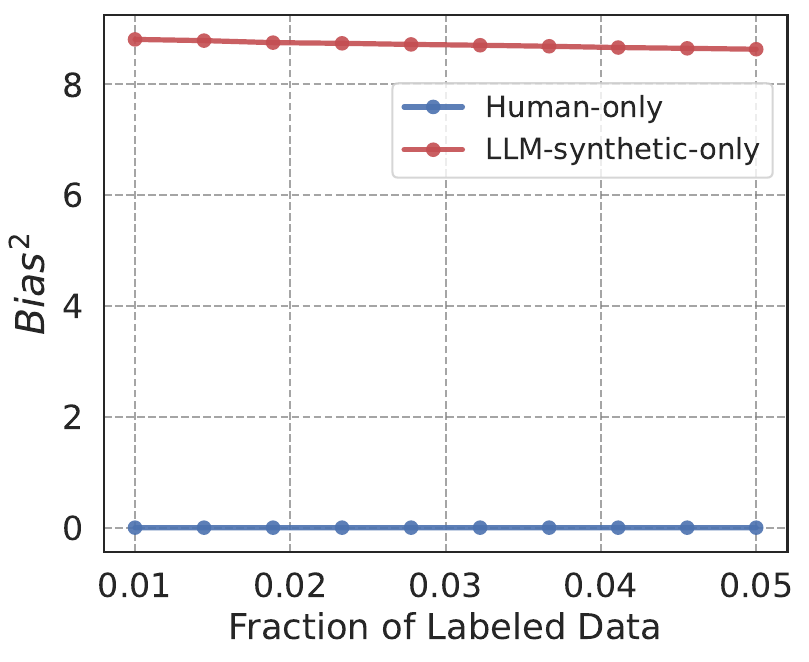}
    \end{subfigure}
    \hfill
    \begin{subfigure}[b]{0.32\textwidth}
        \includegraphics[width=\textwidth]{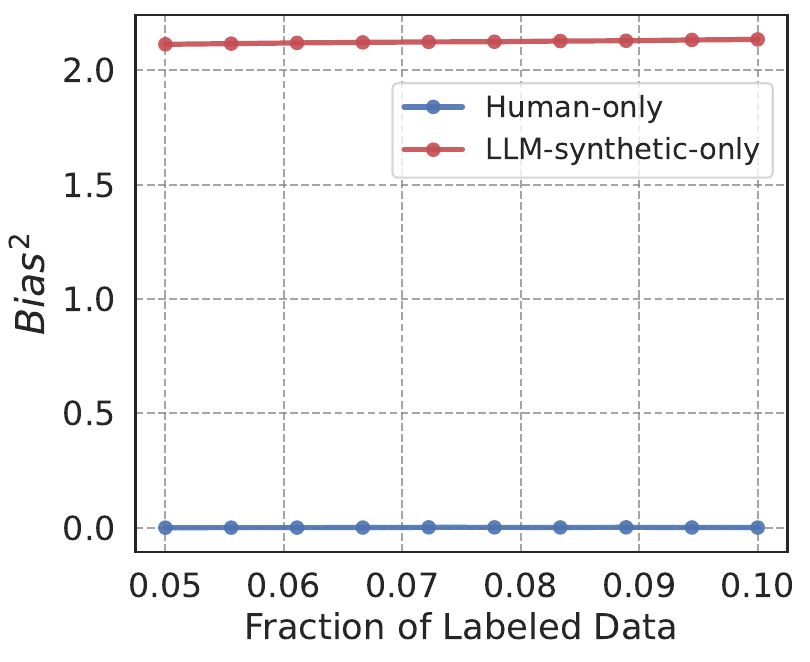}
    \end{subfigure}
    \hfill
    \begin{subfigure}[b]{0.32\textwidth}
        \includegraphics[width=\textwidth]{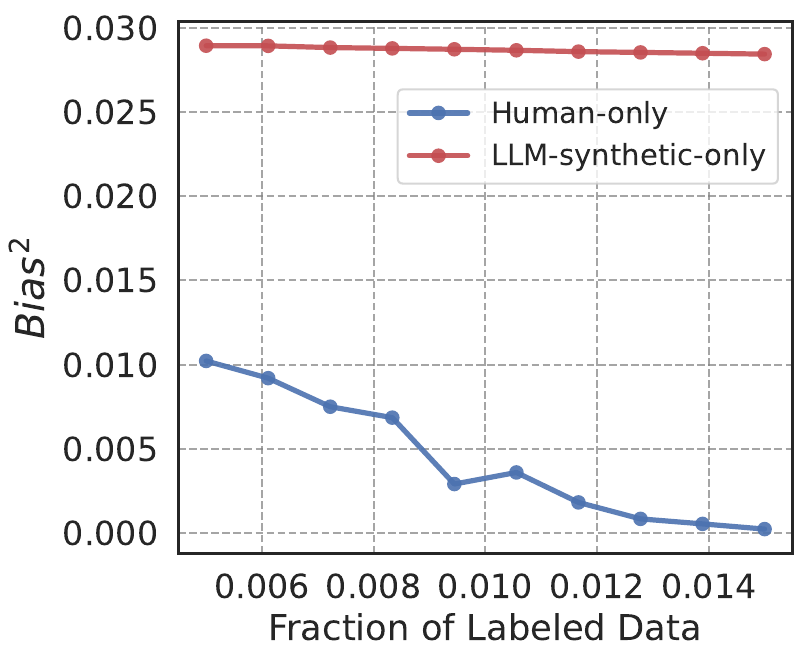}
    \end{subfigure}

    \begin{subfigure}[b]{0.32\textwidth}
        \includegraphics[width=\textwidth]{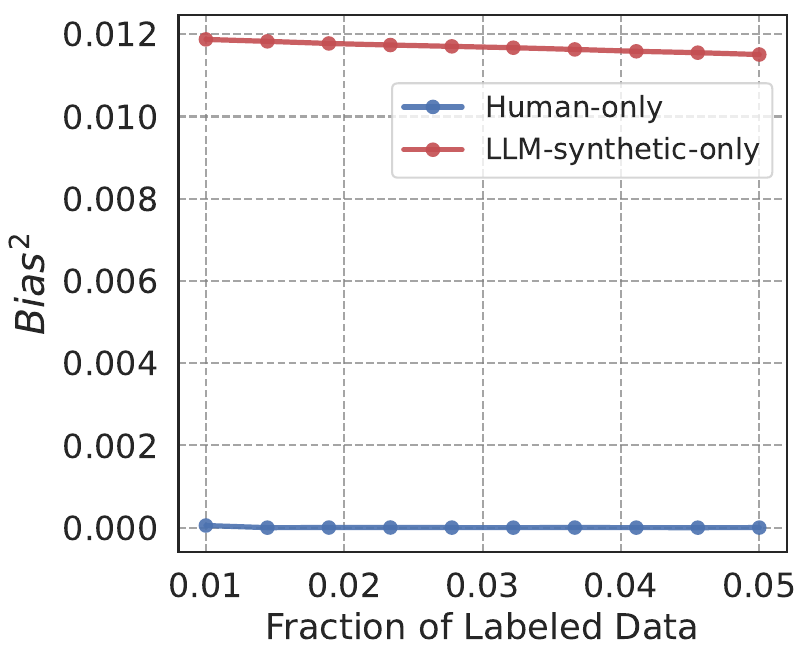}
    \end{subfigure}
    \hfill
    \begin{subfigure}[b]{0.32\textwidth}
        \includegraphics[width=\textwidth]{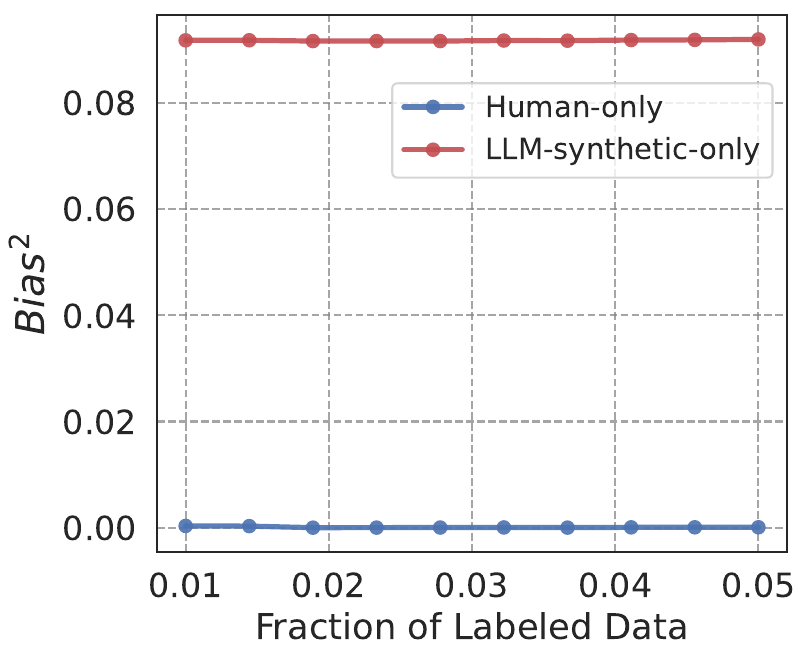}
    \end{subfigure}
    \hfill
    \begin{subfigure}[b]{0.32\textwidth}
        \includegraphics[width=\textwidth]{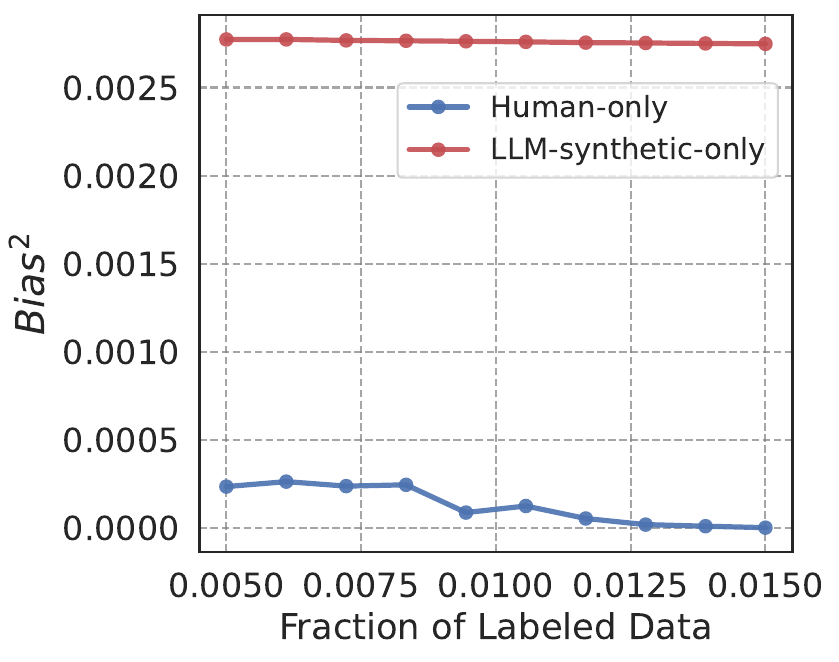}
    \end{subfigure}

    \caption{Performance of a naive estimator for logistic regression (top) and OLS (bottom) using synthetic data only (Politeness (Hedging), Stance, Congressional Bills (from left to right)). We clearly observe that naively using only synthetic data for the estimation task leads to largely biased estimates, as expected.}
    \label{fig:synth-only}
\end{figure}

\begin{figure}[t]
    \centering
    \begin{subfigure}[b]{0.4\textwidth}
        \includegraphics[width=\textwidth]{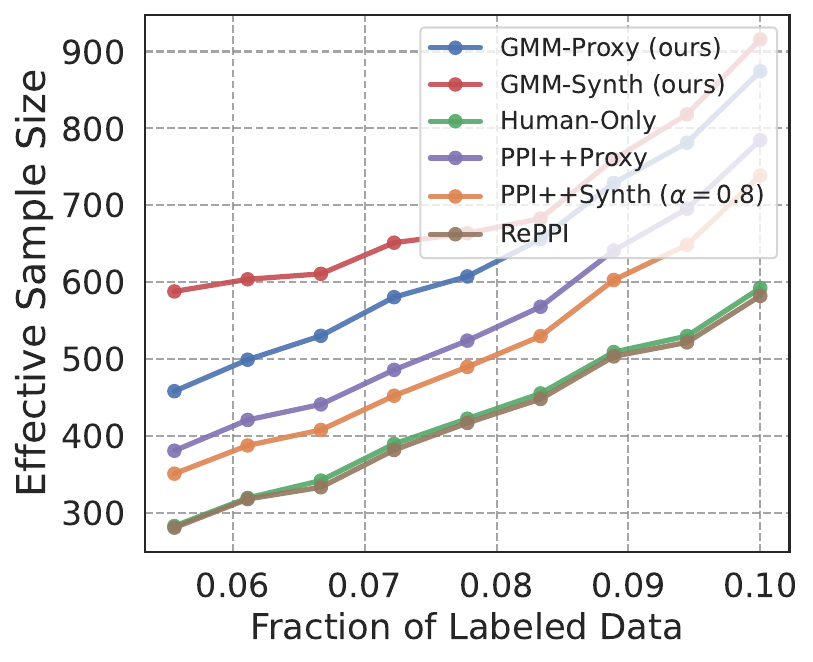}
    \end{subfigure}
    \begin{subfigure}[b]{0.4 \textwidth}
        \includegraphics[width=\textwidth]{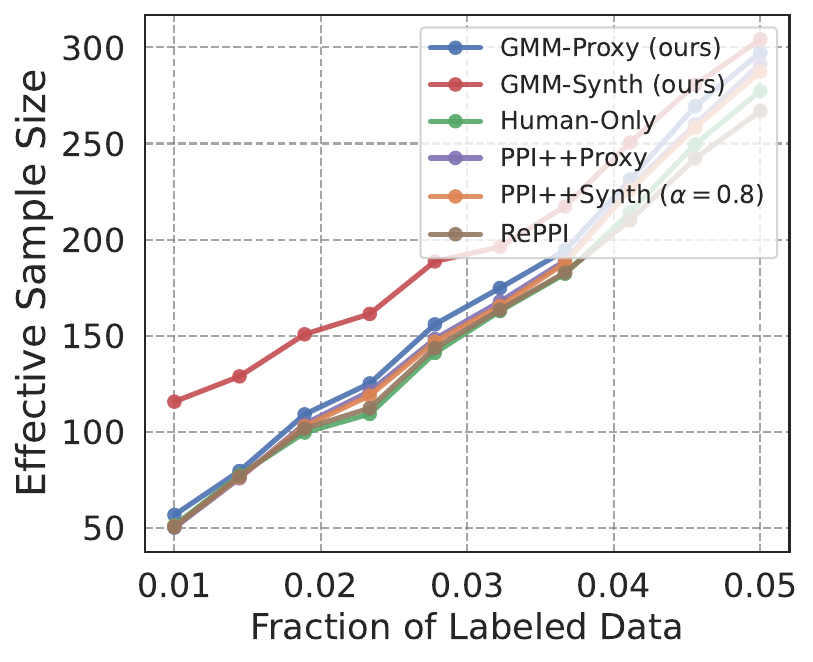}
    \end{subfigure}
    \begin{subfigure}[b]{0.4\textwidth}
        \includegraphics[width=\textwidth]{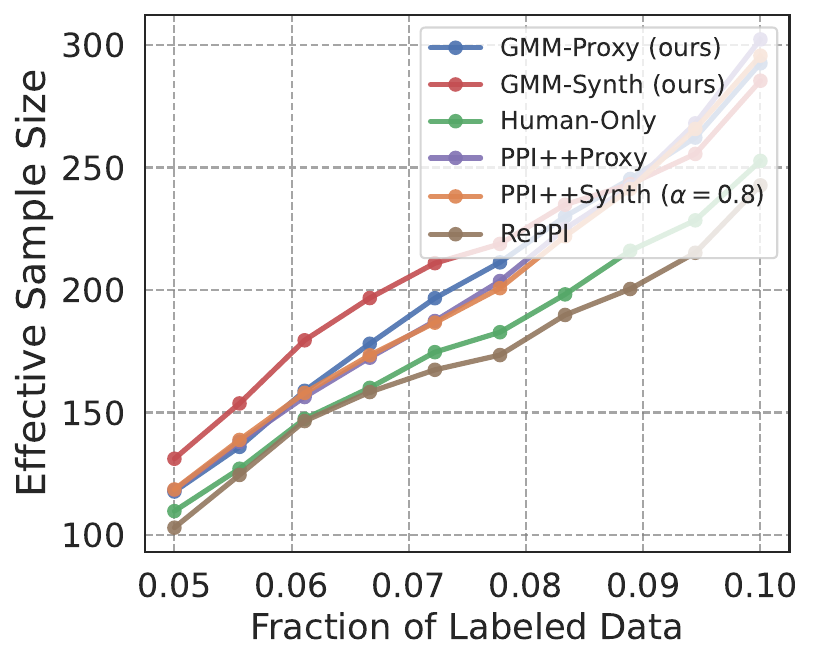}
    \end{subfigure}
    \begin{subfigure}[b]{0.4\textwidth}
        \includegraphics[width=\textwidth]{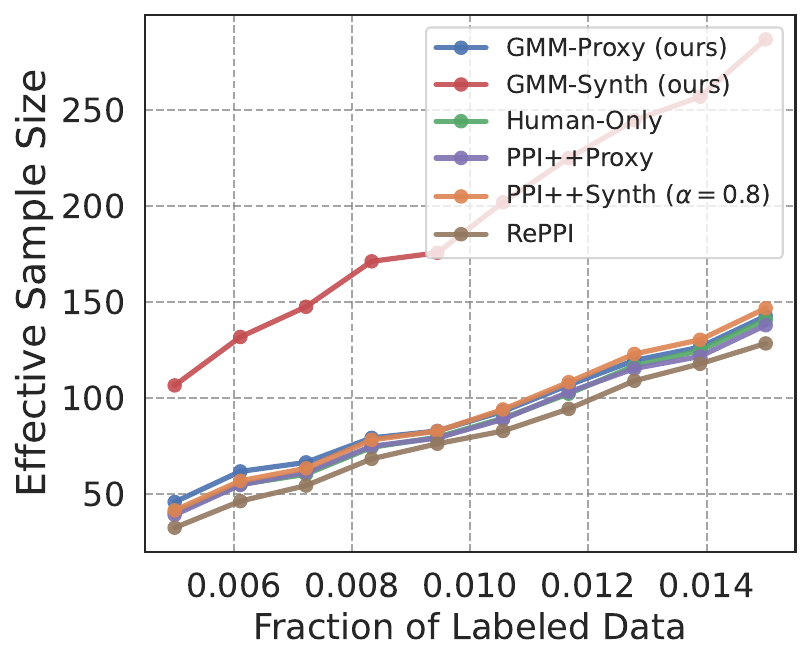}
    \end{subfigure}

    \caption{Effective sample size for logistic regression (Politeness (1pp), Politeness (Hedging), Stance, Congressional Bills (from left to right)). We observe large gains in effective sample size, up to more than 50\%. This represents how many human annotations the method effectively saves while maintaining the same performance (in terms of mean squared error).}
    \label{fig:ess-lr}
\end{figure}

\begin{figure}[t]
    \centering
    \begin{subfigure}[b]{0.4\textwidth}
        \includegraphics[width=\textwidth]{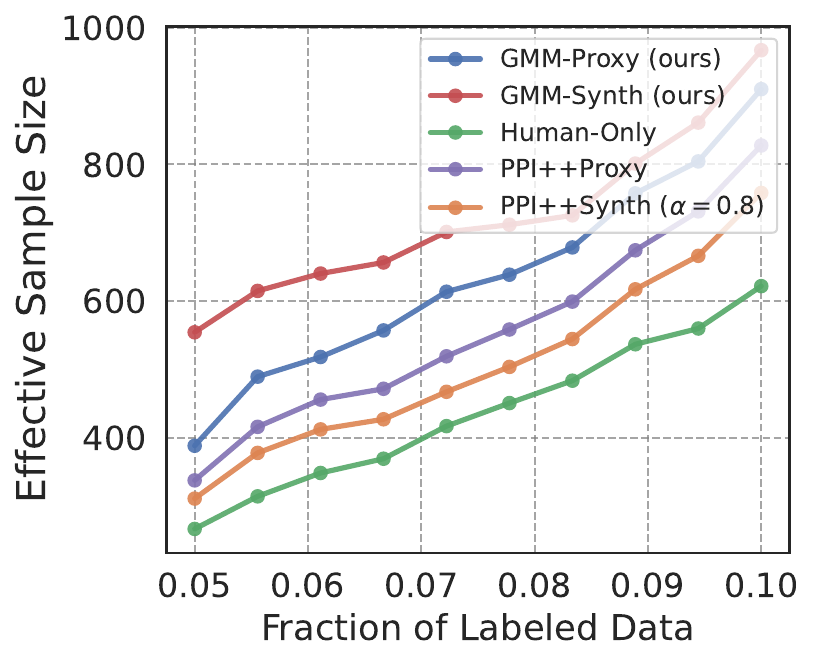}
    \end{subfigure}
    \begin{subfigure}[b]{0.4 \textwidth}
        \includegraphics[width=\textwidth]{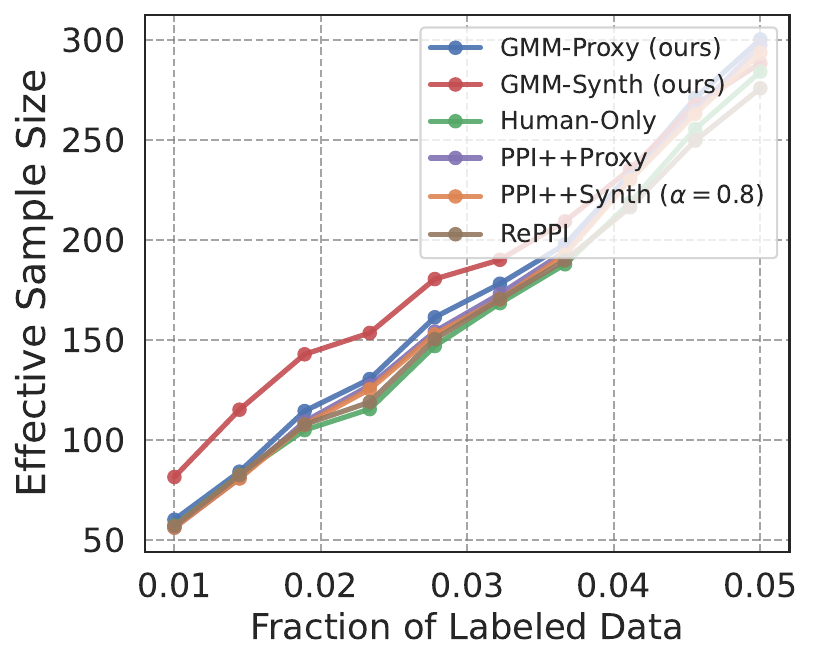}
    \end{subfigure}
    \begin{subfigure}[b]{0.4\textwidth}
        \includegraphics[width=\textwidth]{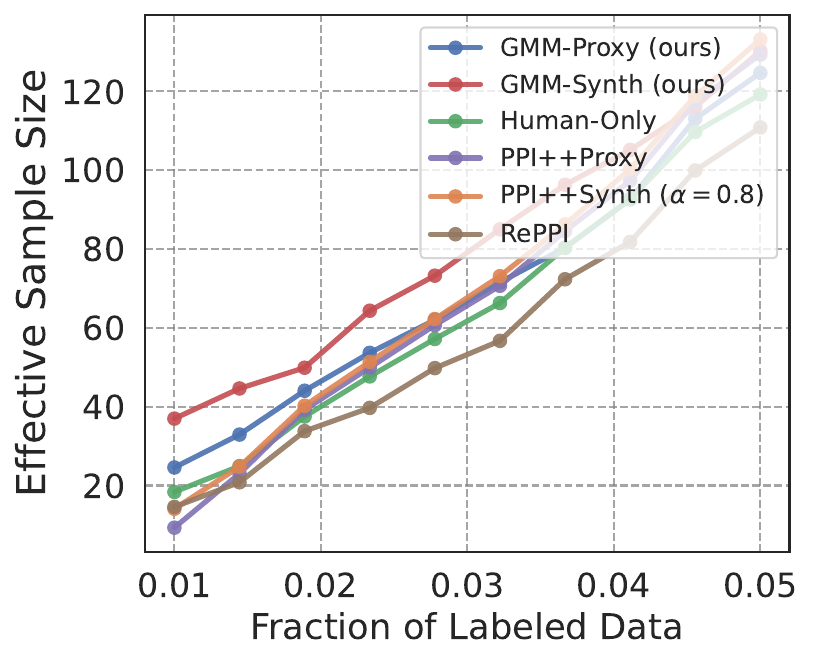}
    \end{subfigure}
    \begin{subfigure}[b]{0.4\textwidth}
        \includegraphics[width=\textwidth]{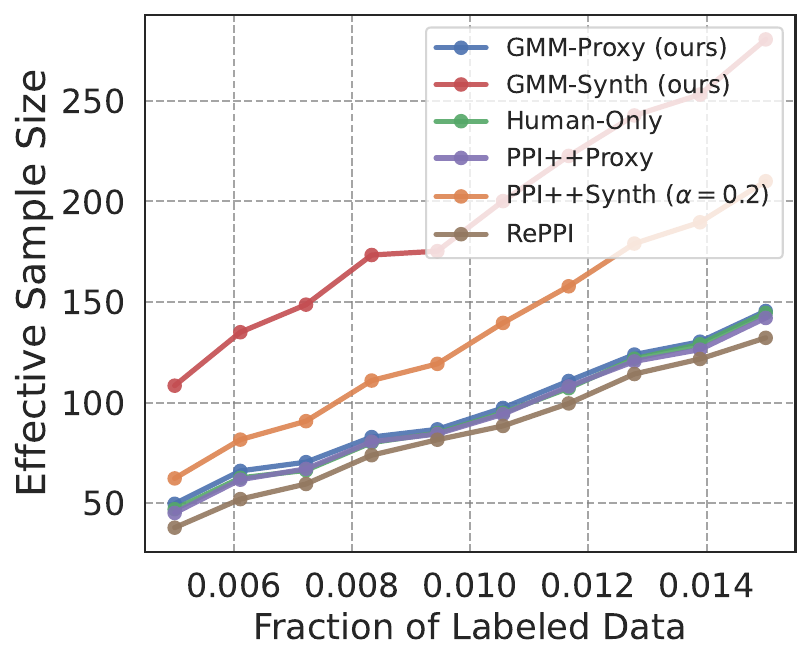}
    \end{subfigure}

    \caption{Effective sample size for OLS (Politeness (1pp), Politeness (Hedging), Stance, Congressional Bills (from left to right)). The RePPI method is omitted from the 1pp plot because its effective sample size drops too low.}
    \label{fig:ess-ols}
\end{figure}

\begin{figure}[t]
    \centering
    \begin{subfigure}[b]{0.4\textwidth}
        \includegraphics[width=\textwidth]{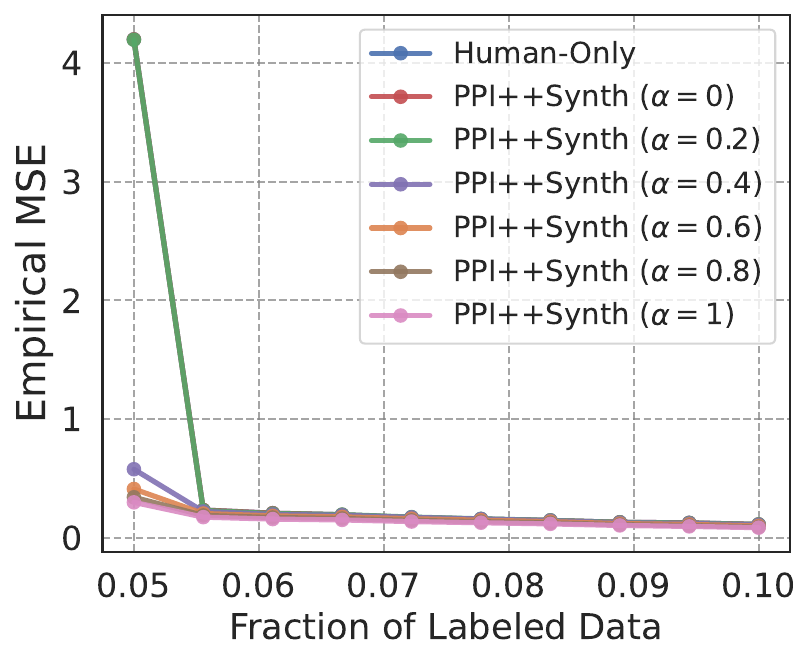}
    \end{subfigure}
    \begin{subfigure}[b]{0.4\textwidth}
        \includegraphics[width=\textwidth]{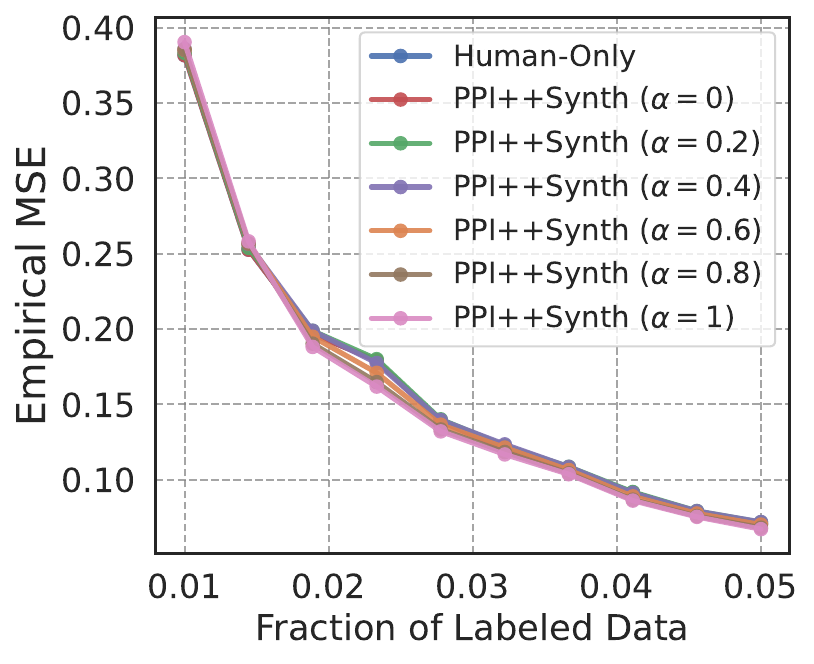}
    \end{subfigure}
    \begin{subfigure}[b]{0.4\textwidth}
        \includegraphics[width=\textwidth]{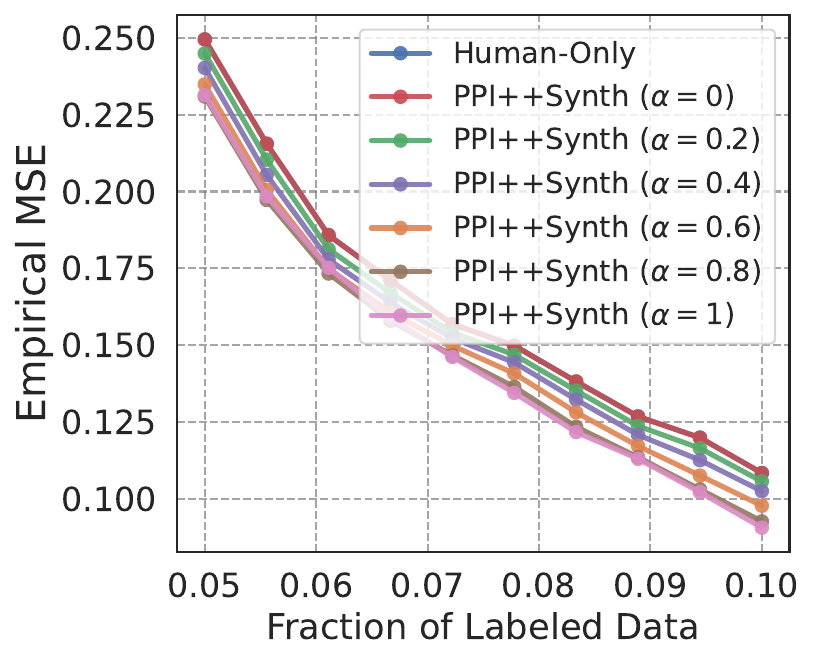}
    \end{subfigure}
    \begin{subfigure}[b]{0.4\textwidth}
        \includegraphics[width=\textwidth]{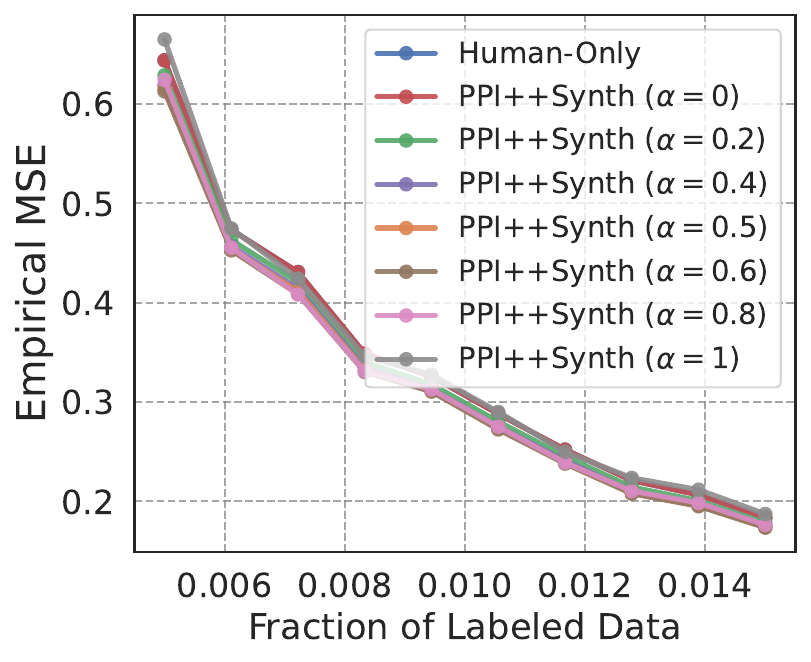}
    \end{subfigure}

    \caption{Grid search results for logistic regression (Politeness (1pp), Politeness (Hedging), Stance, Congressional Bills (from left to right)). This plot shows the grid search over different possible $\alpha$ values \textit{without} cross-fitting. Note, this is not a valid solution in our setup, as it requires peeking in hyperparameter selection, but it provides an oracle version of the baseline, which we term as PPI++Synth (Oracle). The $\alpha$ value that leads to the smallest MSE is the one reported in Figure \ref{fig:key_results} in the main text.} 
    \label{fig:grid-search-lr}
\end{figure}

\begin{figure}[t]
    \centering
    \begin{subfigure}[b]{0.4\textwidth}
        \includegraphics[width=\textwidth]{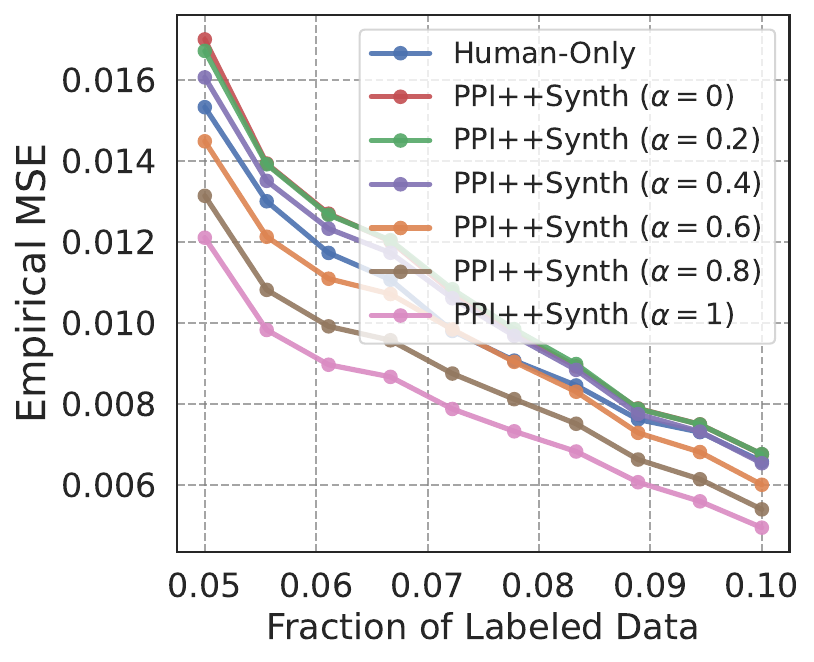}
    \end{subfigure}
    \begin{subfigure}[b]{0.4\textwidth}
        \includegraphics[width=\textwidth]{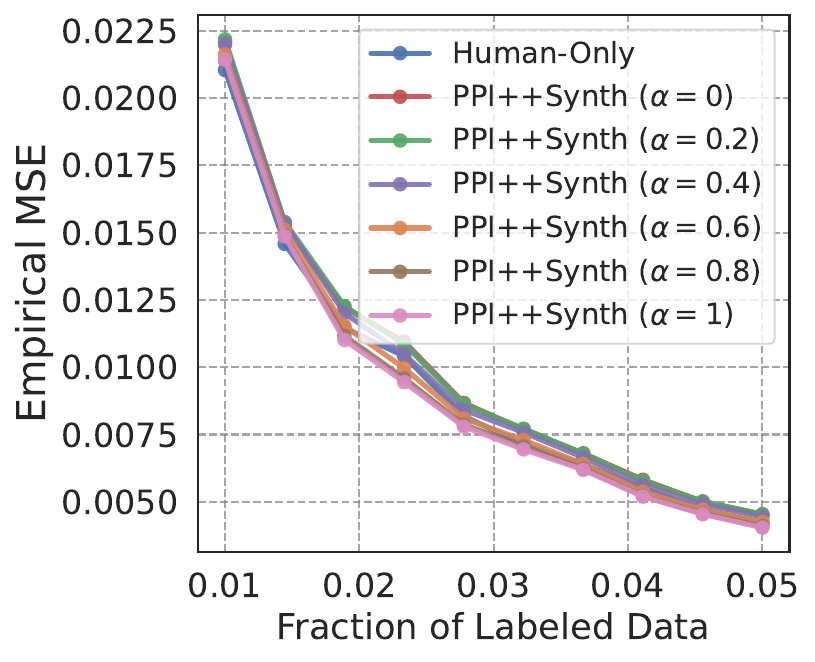}
    \end{subfigure}
    \begin{subfigure}[b]{0.4\textwidth}
        \includegraphics[width=\textwidth]{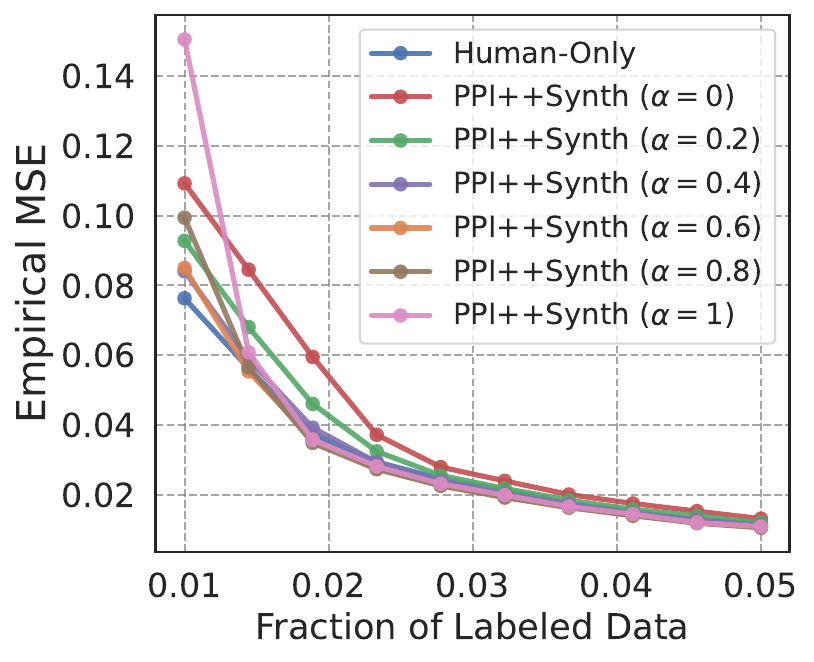}
    \end{subfigure}
    \begin{subfigure}[b]{0.4\textwidth}
        \includegraphics[width=\textwidth]{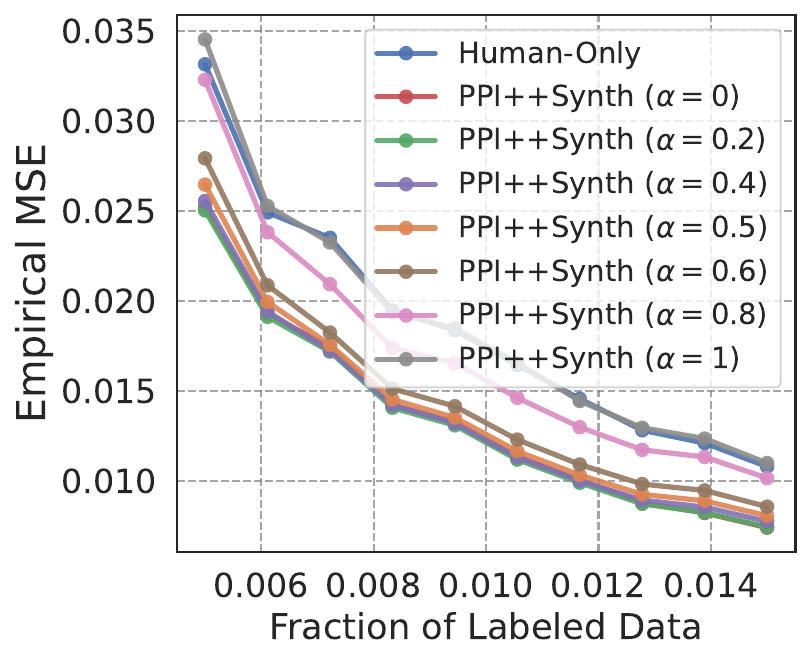}
    \end{subfigure}

    \caption{Grid search results for OLS (Politeness (1pp), Politeness (Hedging), Stance, Congressional Bills (from left to right)). The $\alpha$ value that leads to the smallest MSE is the one reported in Figure \ref{fig:key_results_ols} in the main text.}
    \label{fig:grid-search-ols}
\end{figure}

\begin{figure} [t]
    \centering
    \begin{subfigure}[b]{0.32\textwidth}
        \includegraphics[width=\textwidth]{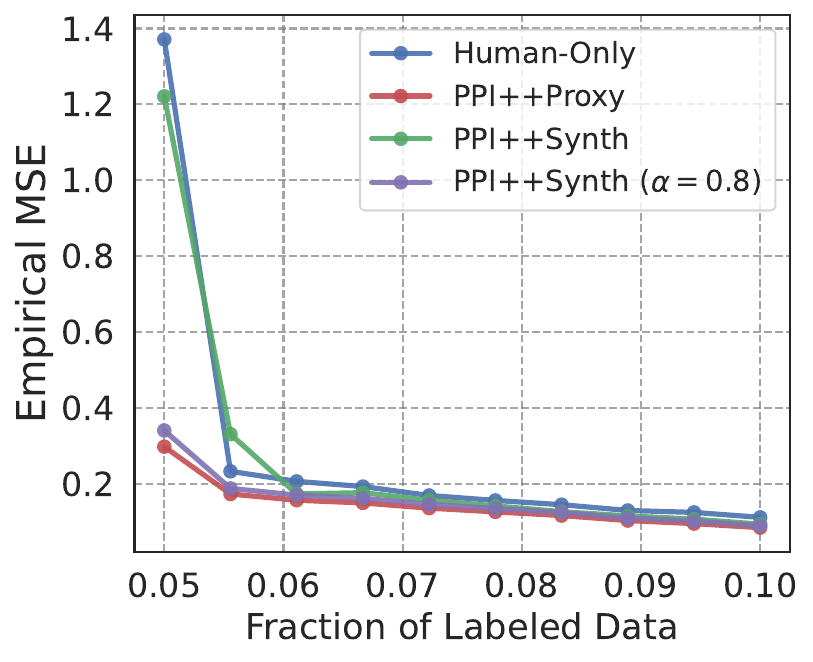}
    \end{subfigure}
    \hfill
    \begin{subfigure}[b]{0.32\textwidth}
        \includegraphics[width=\textwidth]{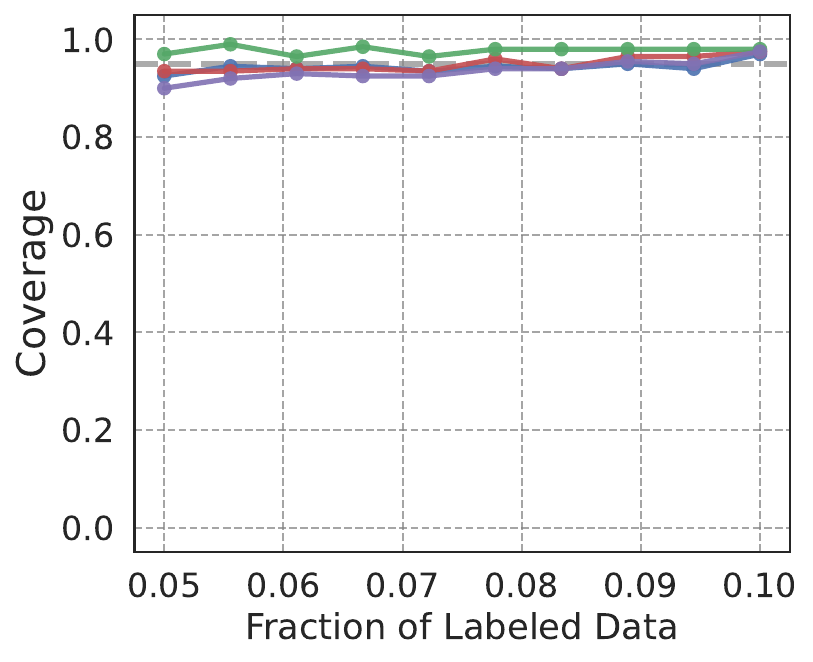}
    \end{subfigure}
    \hfill
    \begin{subfigure}[b]{0.32\textwidth}
        \includegraphics[width=\textwidth]{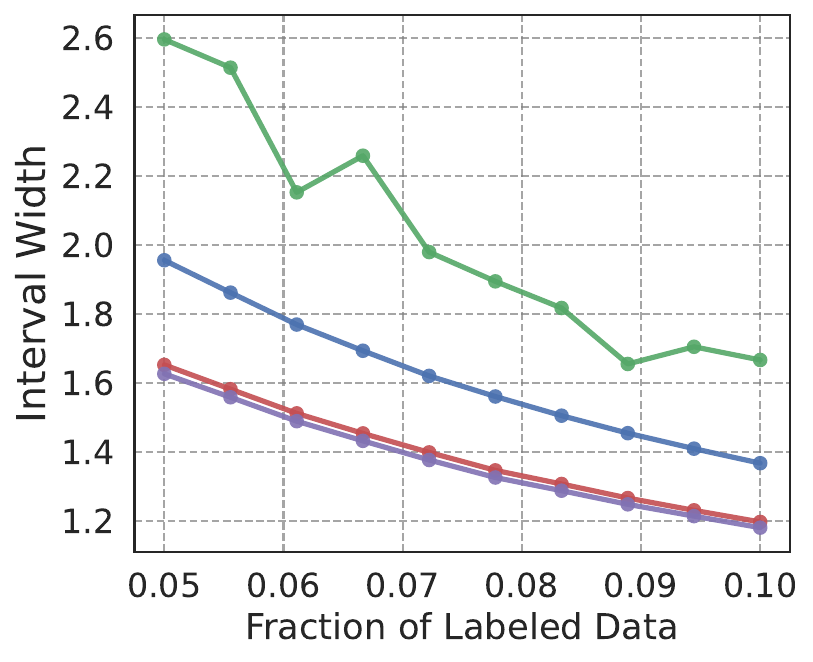}
    \end{subfigure}

    \vspace{0.25cm}
    \begin{subfigure}[b]{0.32\textwidth}
        \includegraphics[width=\textwidth]{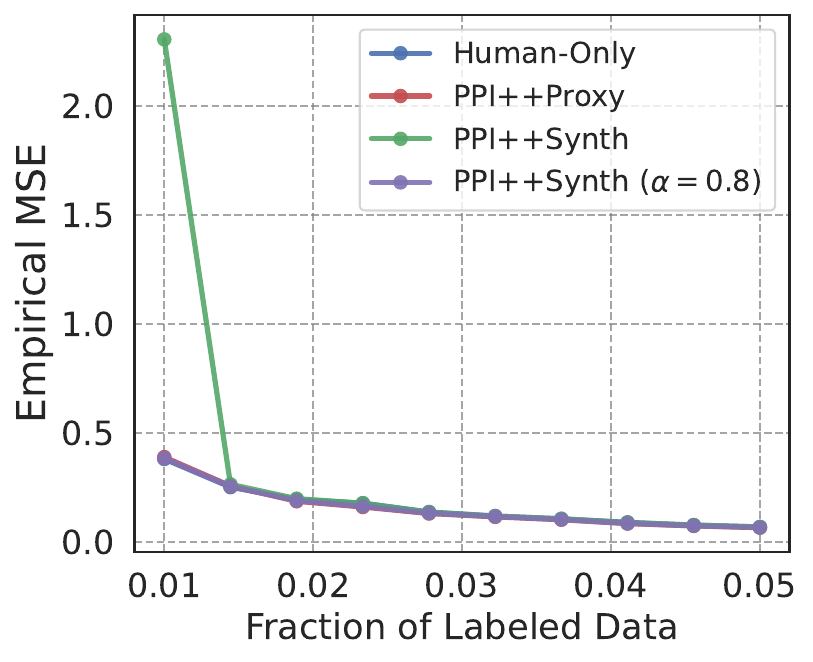}
    \end{subfigure}
    \hfill
    \begin{subfigure}[b]{0.32\textwidth}
        \includegraphics[width=\textwidth]{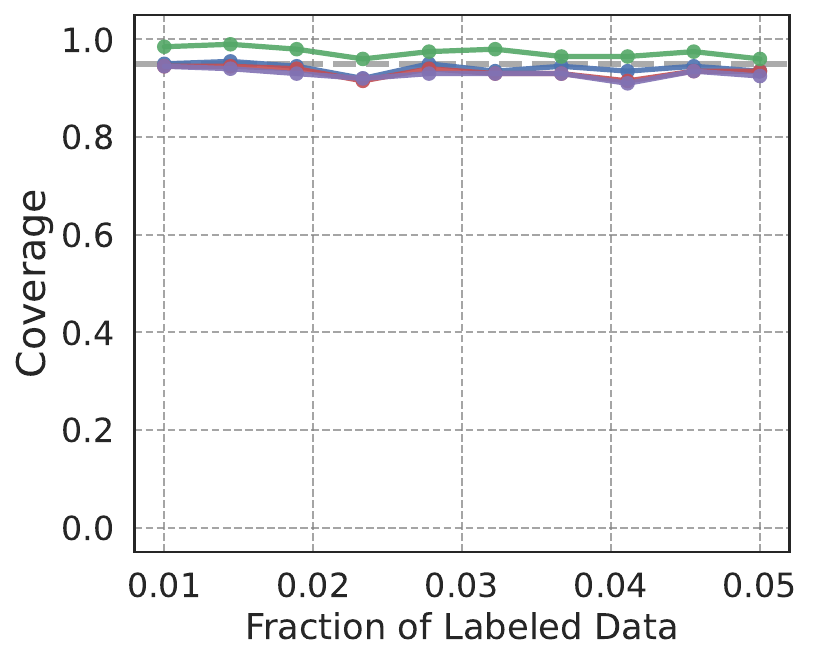}
    \end{subfigure}
    \hfill
    \begin{subfigure}[b]{0.32\textwidth}
        \includegraphics[width=\textwidth]{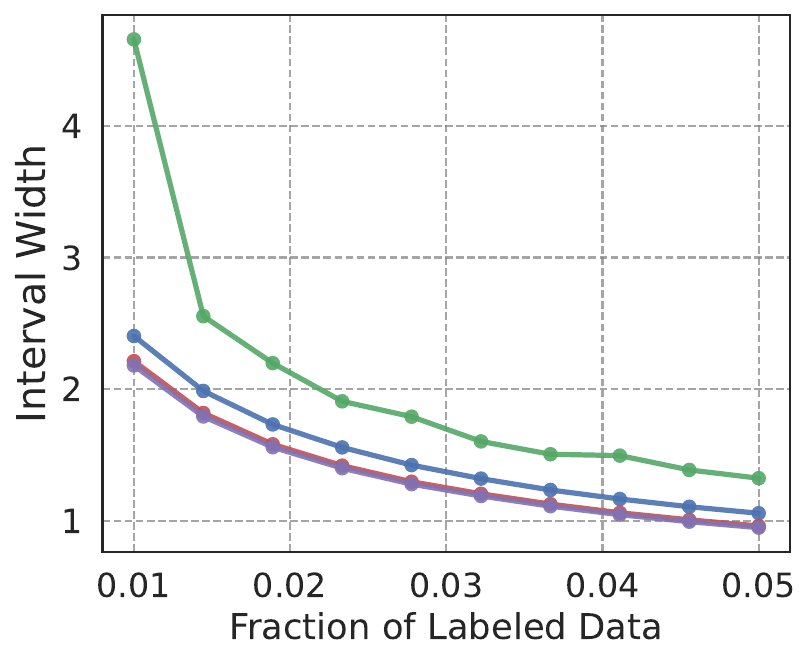}
    \end{subfigure}

    \vspace{0.25cm}

     \begin{subfigure}[b]{0.32\textwidth}
        \includegraphics[width=\textwidth]{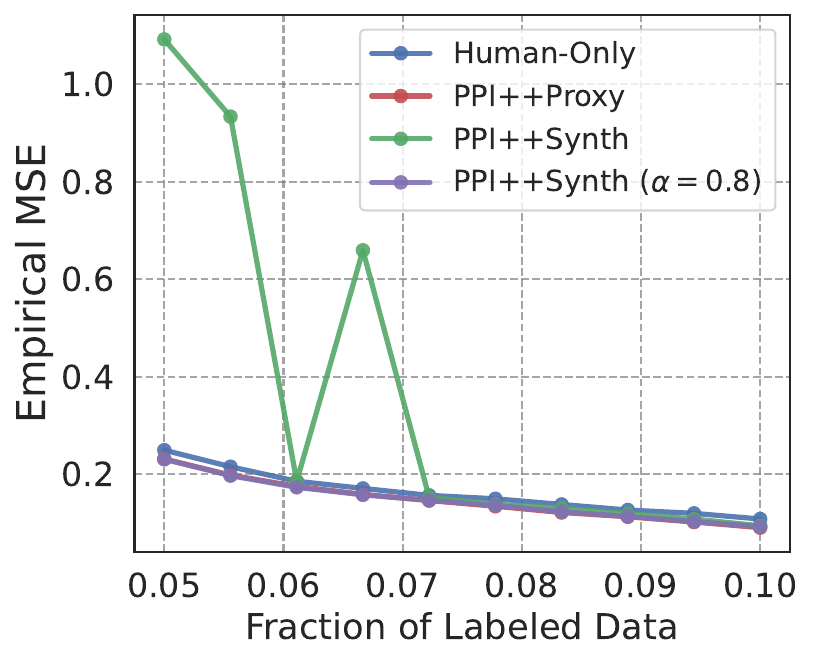}
    \end{subfigure}
    \hfill
    \begin{subfigure}[b]{0.32\textwidth}
        \includegraphics[width=\textwidth]{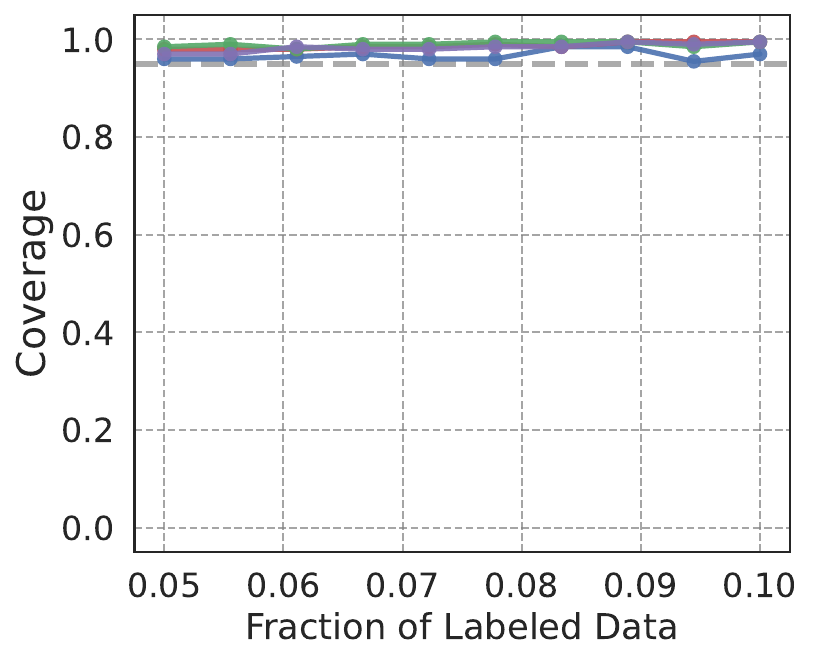}
    \end{subfigure}
    \hfill
    \begin{subfigure}[b]{0.32\textwidth}
        \includegraphics[width=\textwidth]{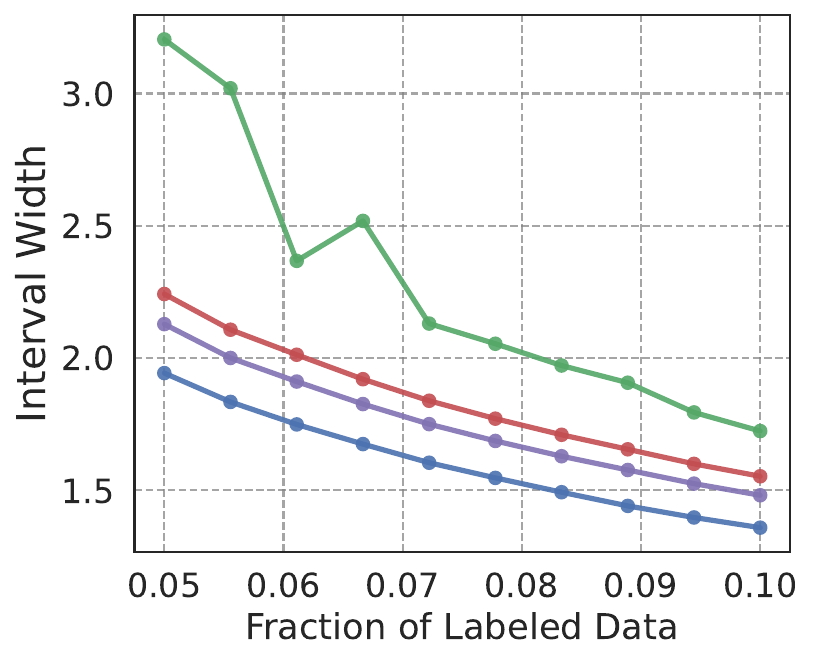}
    \end{subfigure}

    \vspace{0.25cm}

     \begin{subfigure}[b]{0.32\textwidth}
        \includegraphics[width=\textwidth]{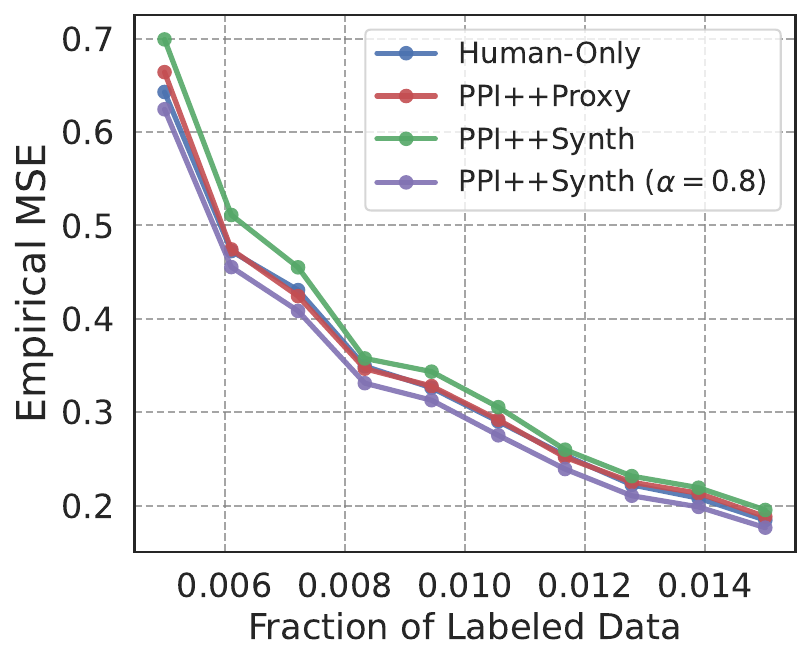}
    \end{subfigure}
    \hfill
    \begin{subfigure}[b]{0.32\textwidth}
        \includegraphics[width=\textwidth]{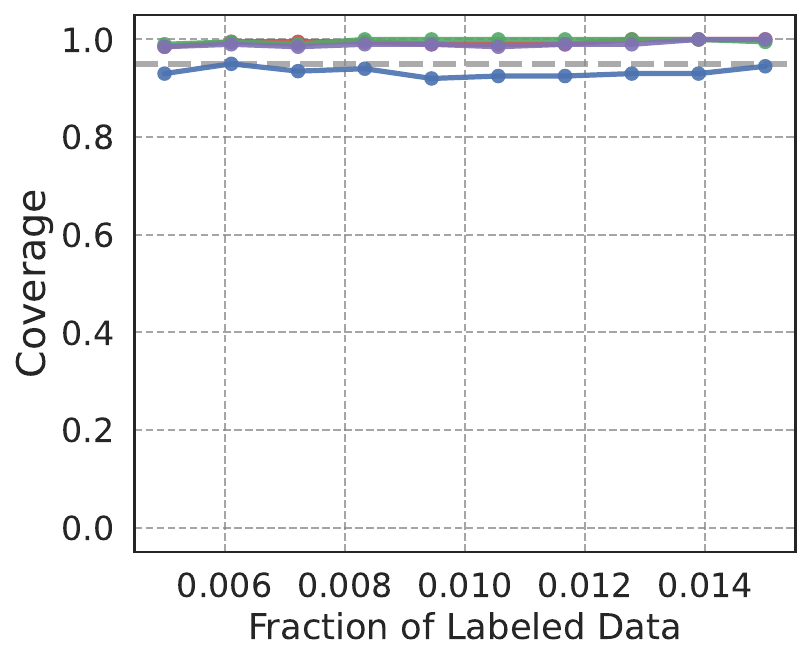}
    \end{subfigure}
    \hfill
    \begin{subfigure}[b]{0.32\textwidth}
        \includegraphics[width=\textwidth]{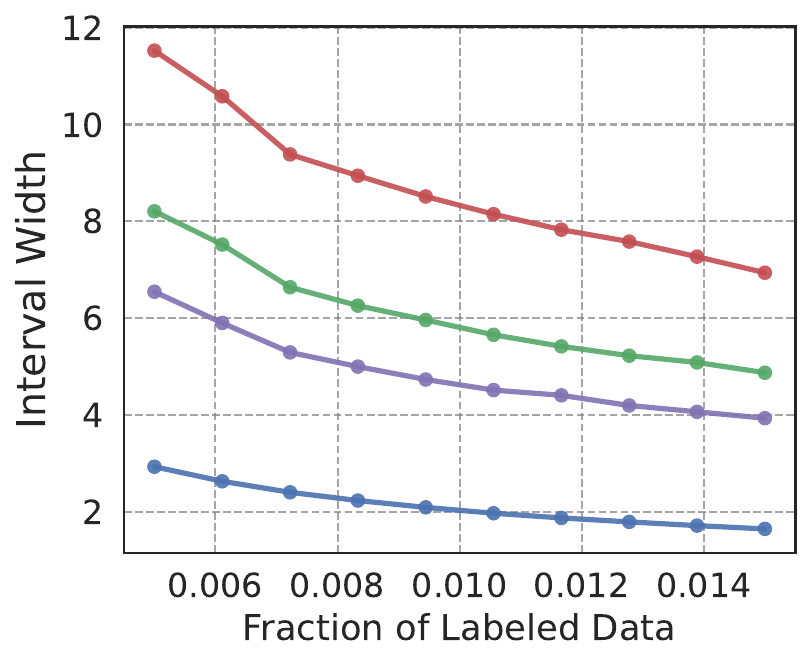}
    \end{subfigure}

    \caption{PPI++Synth results for logistic regression. This is the valid implementation with cross-fitting to select hyperparameters in a statistically valid fashion. We observe that it is upper-bounded by its oracle variant (PPI++Synth $\alpha = 0.8)$, as expected. In the main text (Fig. \ref{fig:key_results}), we report the oracle variant results  to account for potential gains from improved cross-fitting techniques.}
    \label{fig:cross_fitting_lr}
\end{figure}

\begin{figure} [t]
    \centering
    \begin{subfigure}[b]{0.32\textwidth}
        \includegraphics[width=\textwidth]{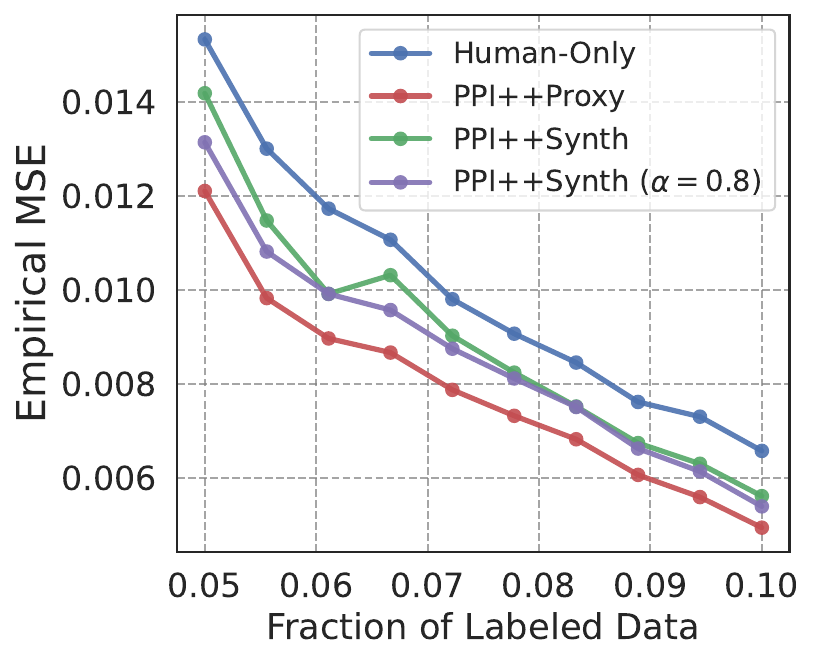}
    \end{subfigure}
    \hfill
    \begin{subfigure}[b]{0.32\textwidth}
        \includegraphics[width=\textwidth]{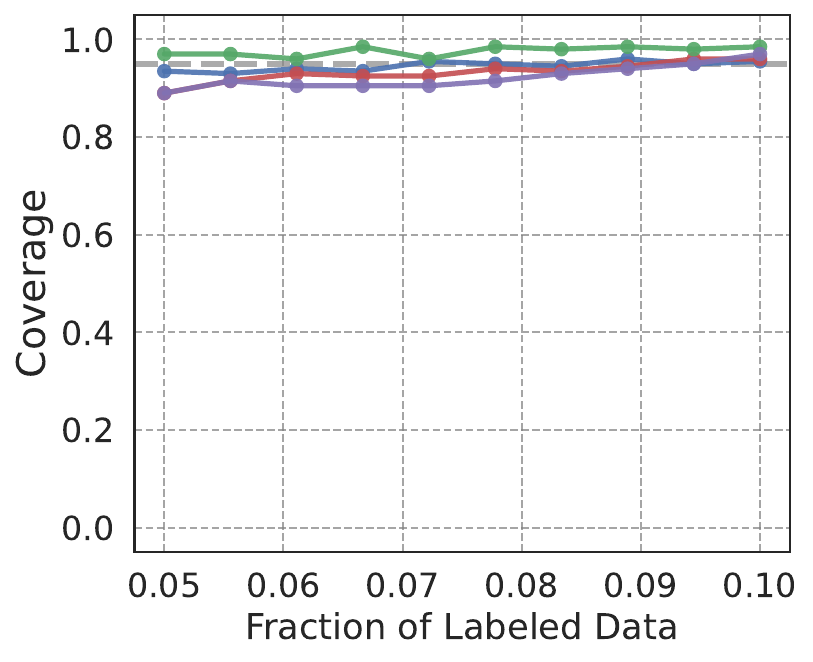}
    \end{subfigure}
    \hfill
    \begin{subfigure}[b]{0.32\textwidth}
        \includegraphics[width=\textwidth]{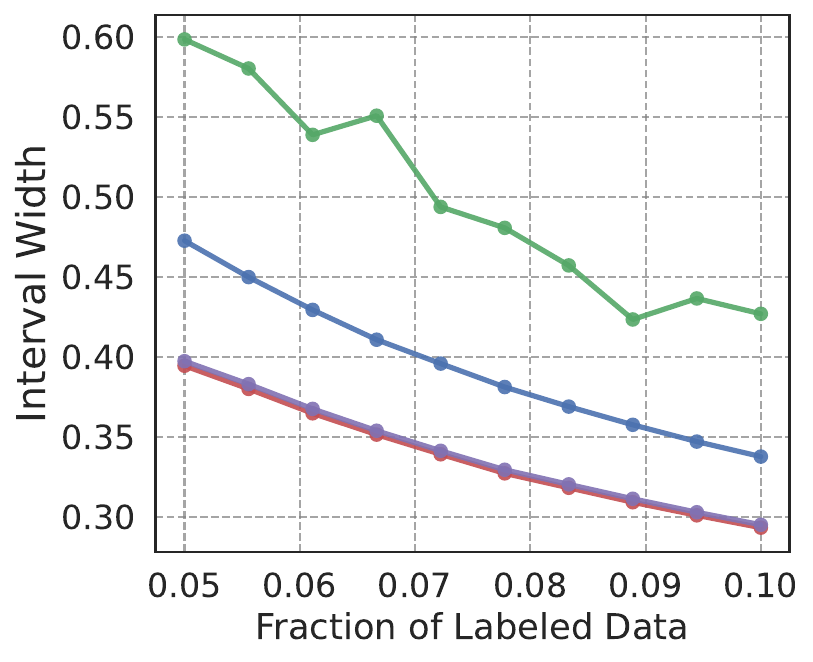}
    \end{subfigure}

    \vspace{0.25cm}
    \begin{subfigure}[b]{0.32\textwidth}
        \includegraphics[width=\textwidth]{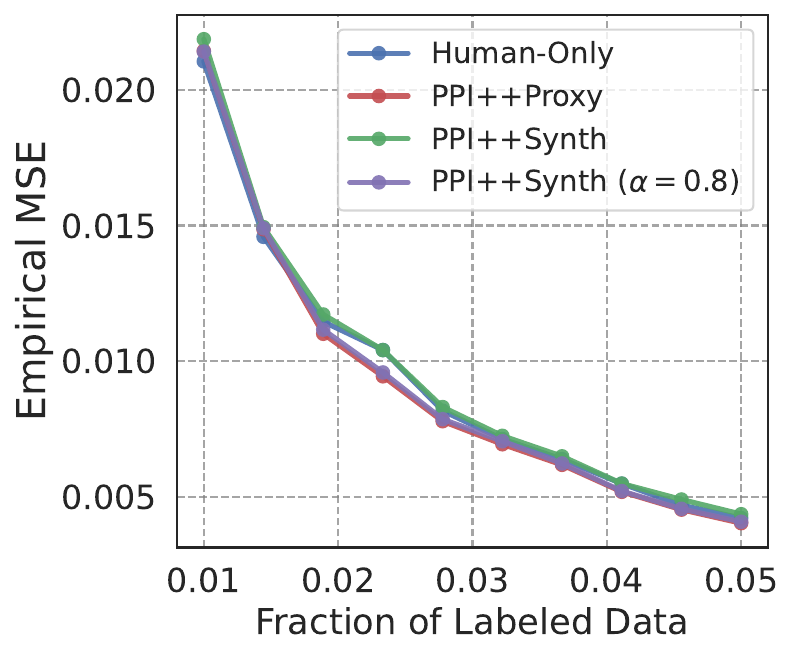}
    \end{subfigure}
    \hfill
    \begin{subfigure}[b]{0.32\textwidth}
        \includegraphics[width=\textwidth]{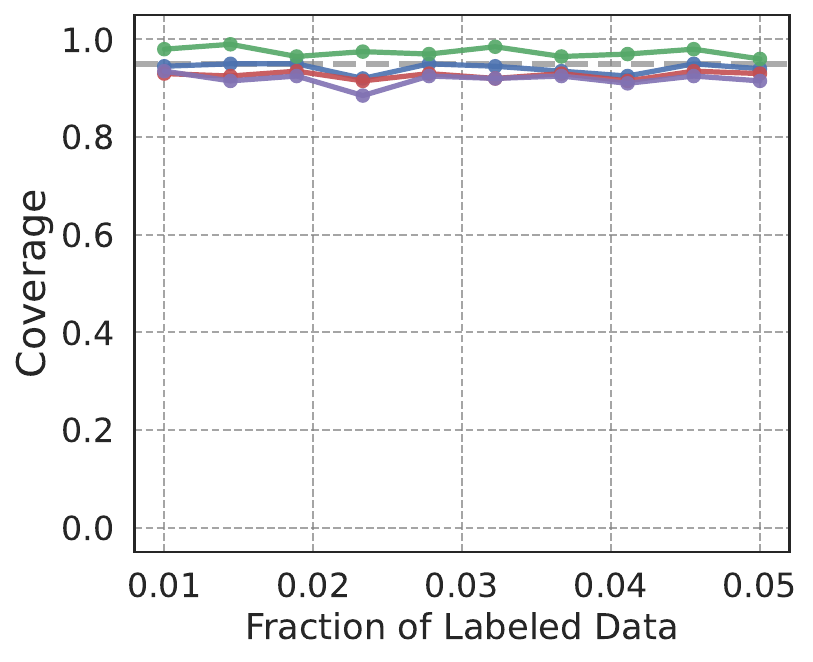}
    \end{subfigure}
    \hfill
    \begin{subfigure}[b]{0.32\textwidth}
        \includegraphics[width=\textwidth]{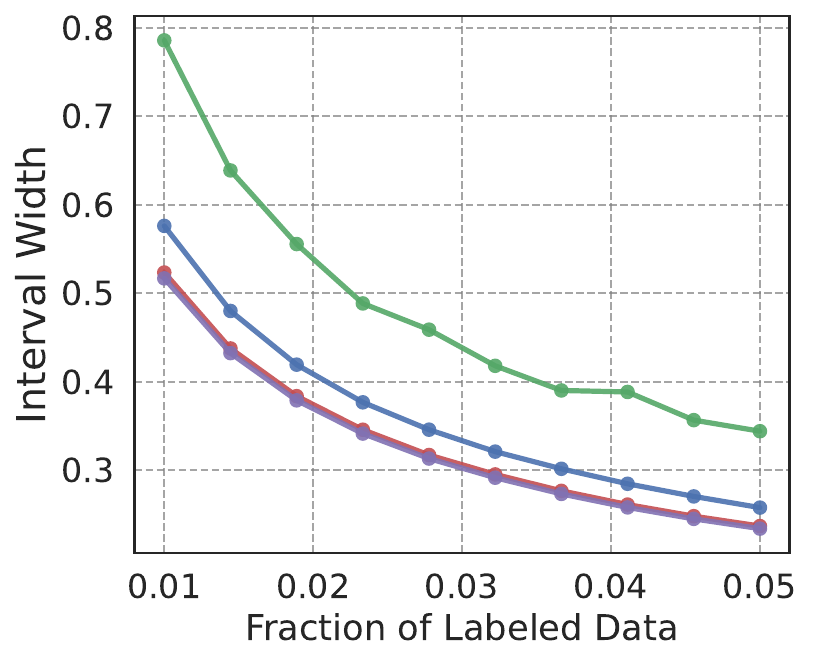}
    \end{subfigure}

    \vspace{0.25cm}

     \begin{subfigure}[b]{0.32\textwidth}
        \includegraphics[width=\textwidth]{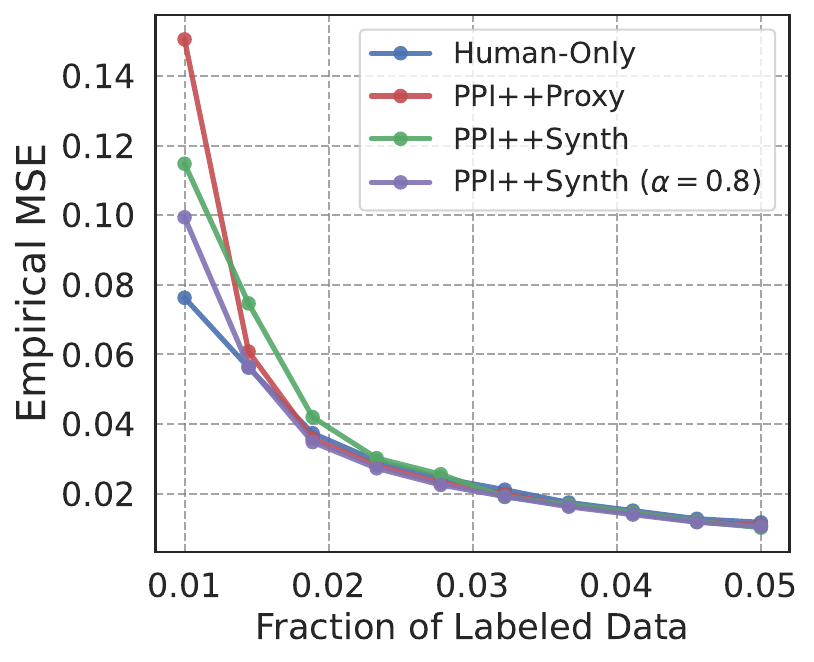}
    \end{subfigure}
    \hfill
    \begin{subfigure}[b]{0.32\textwidth}
        \includegraphics[width=\textwidth]{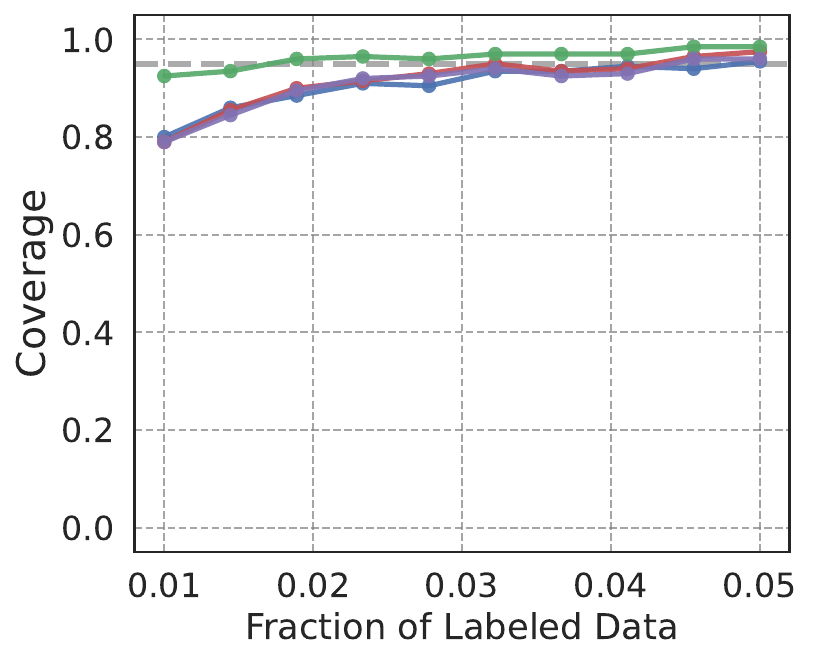}
    \end{subfigure}
    \hfill
    \begin{subfigure}[b]{0.32\textwidth}
        \includegraphics[width=\textwidth]{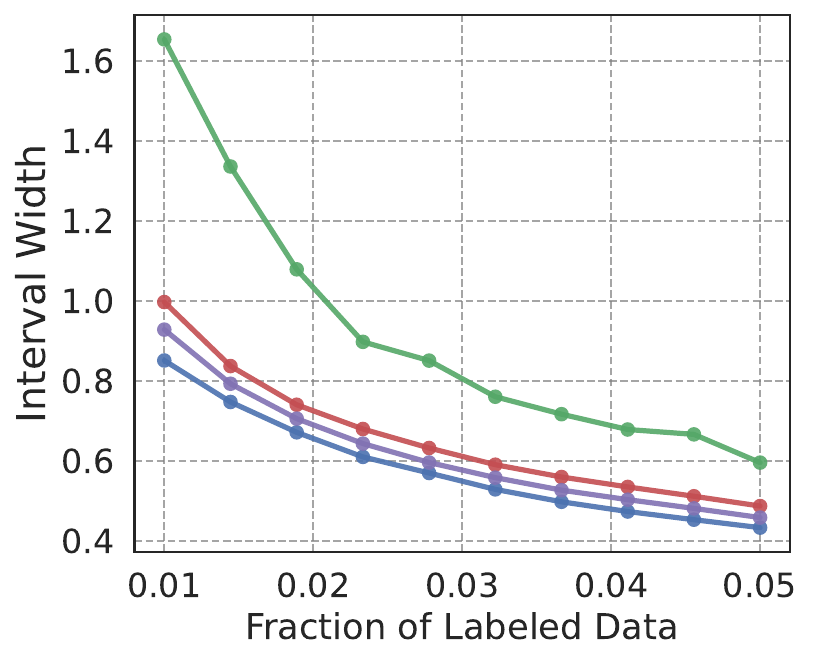}
    \end{subfigure}

    \vspace{0.25cm}

     \begin{subfigure}[b]{0.32\textwidth}
        \includegraphics[width=\textwidth]{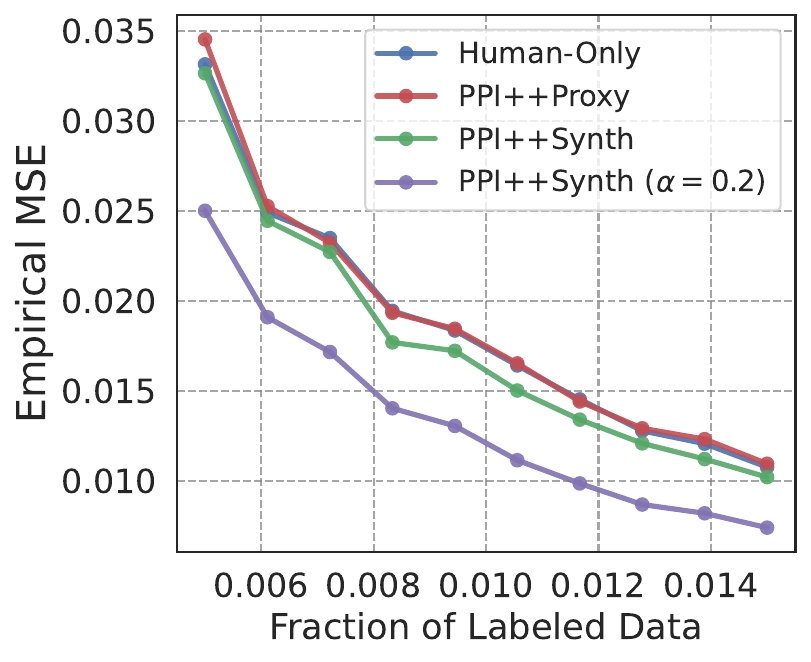}
    \end{subfigure}
    \hfill
    \begin{subfigure}[b]{0.32\textwidth}
        \includegraphics[width=\textwidth]{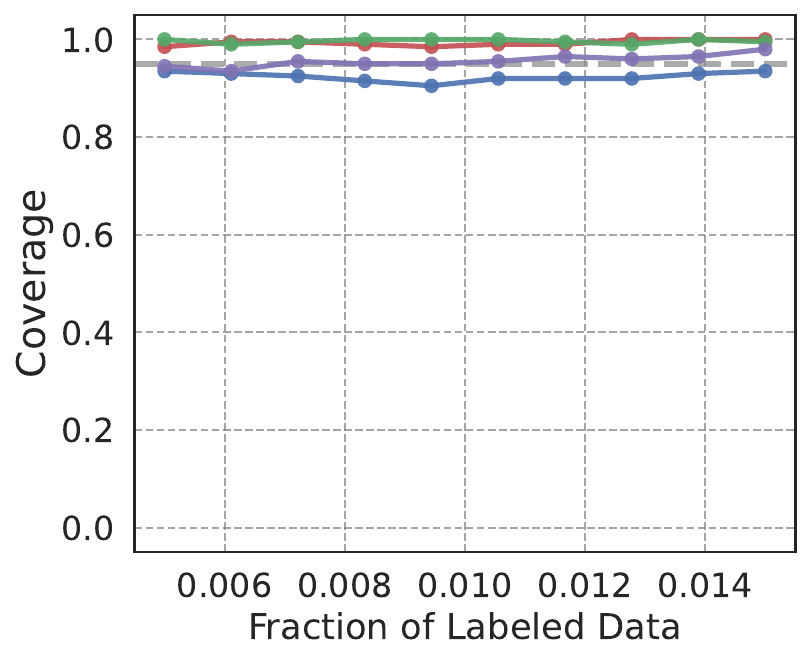}
    \end{subfigure}
    \hfill
    \begin{subfigure}[b]{0.32\textwidth}
        \includegraphics[width=\textwidth]{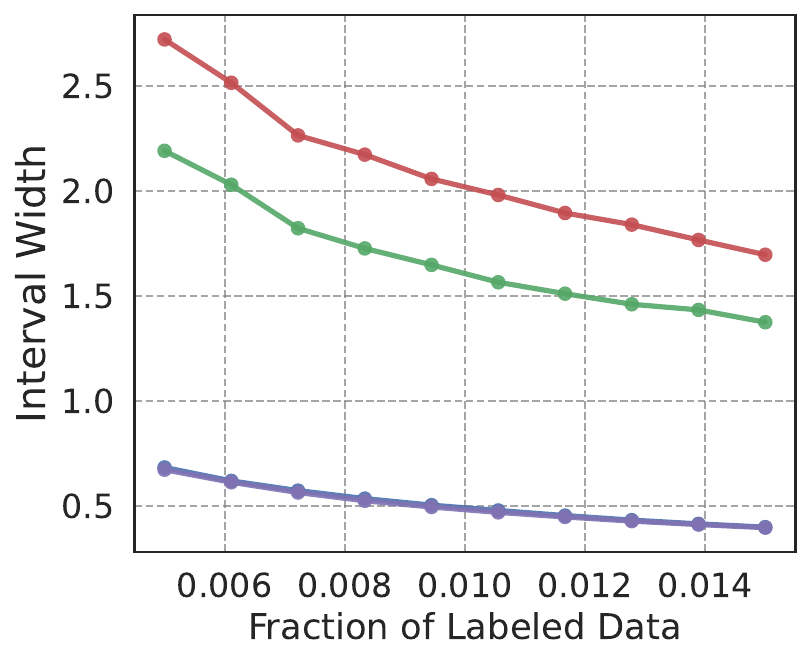}
    \end{subfigure}

    \caption{PPI++Synth results for OLS. Similarly as above, we observe that it is upper-bounded by its oracle variant (PPI++Synth $\alpha = 0.8$ and PPI++Synth $\alpha = 0.2$ for CBP), as expected.}
    \label{fig:cross_fitting_ols}
\end{figure}

\begin{figure} [t]
    \centering
    \begin{subfigure}[b]{0.32\textwidth}
        \includegraphics[width=\textwidth]{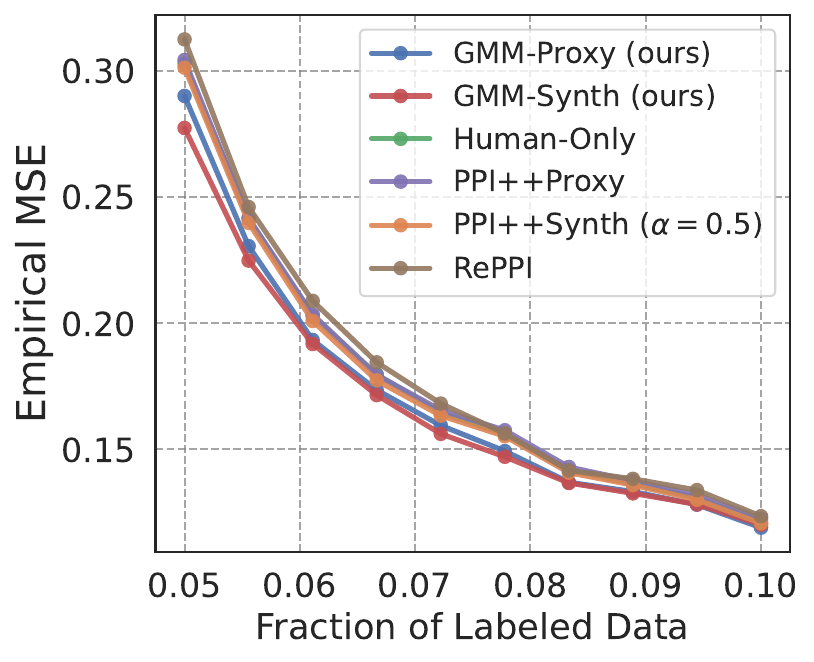}
    \end{subfigure}
    \hfill
    \begin{subfigure}[b]{0.32\textwidth}
        \includegraphics[width=\textwidth]{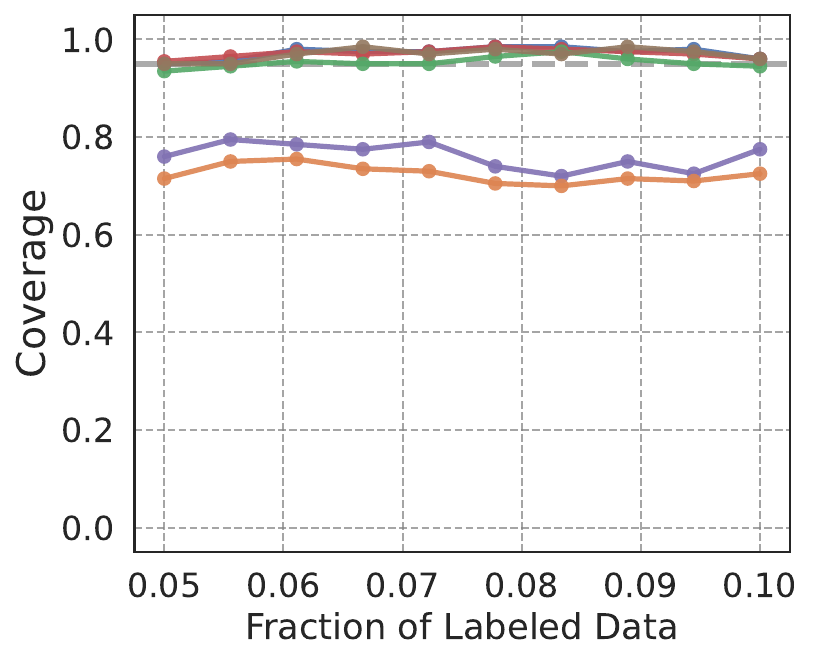}
    \end{subfigure}
    \hfill
    \begin{subfigure}[b]{0.32\textwidth}
        \includegraphics[width=\textwidth]{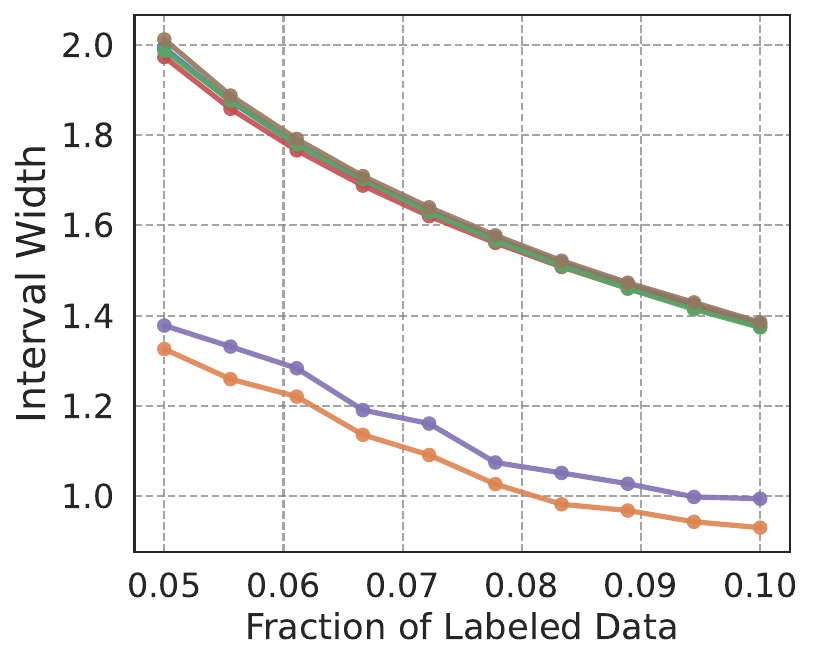}
    \end{subfigure}

    \vspace{0.25cm}
    \begin{subfigure}[b]{0.32\textwidth}
        \includegraphics[width=\textwidth]{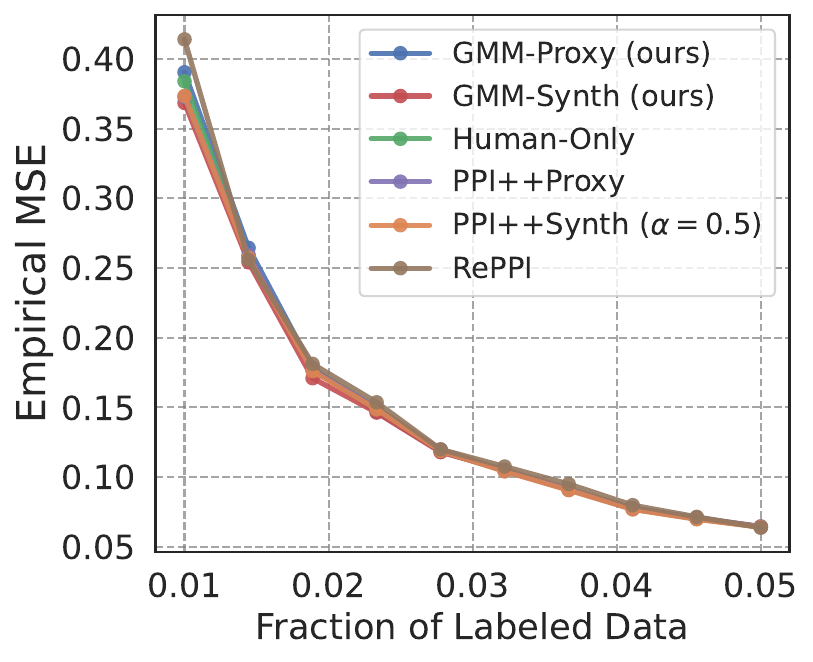}
    \end{subfigure}
    \hfill
    \begin{subfigure}[b]{0.32\textwidth}
        \includegraphics[width=\textwidth]{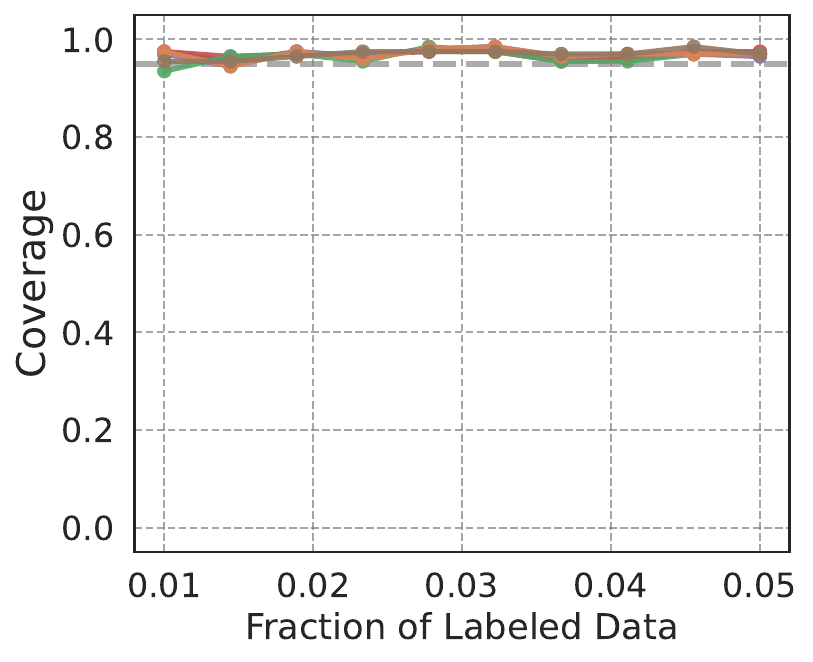}
    \end{subfigure}
    \hfill
    \begin{subfigure}[b]{0.32\textwidth}
        \includegraphics[width=\textwidth]{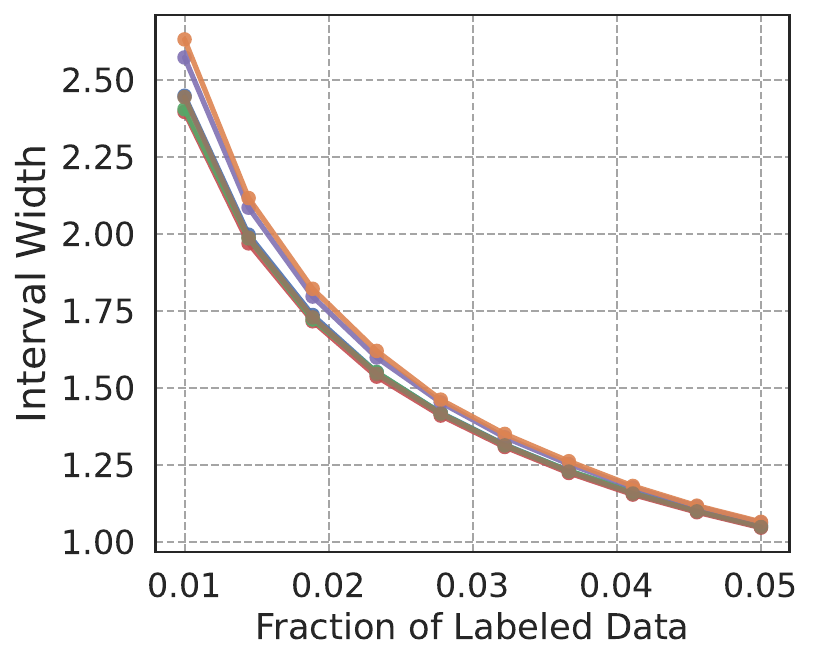}
    \end{subfigure}

    \vspace{0.25cm}

     \begin{subfigure}[b]{0.32\textwidth}
        \includegraphics[width=\textwidth]{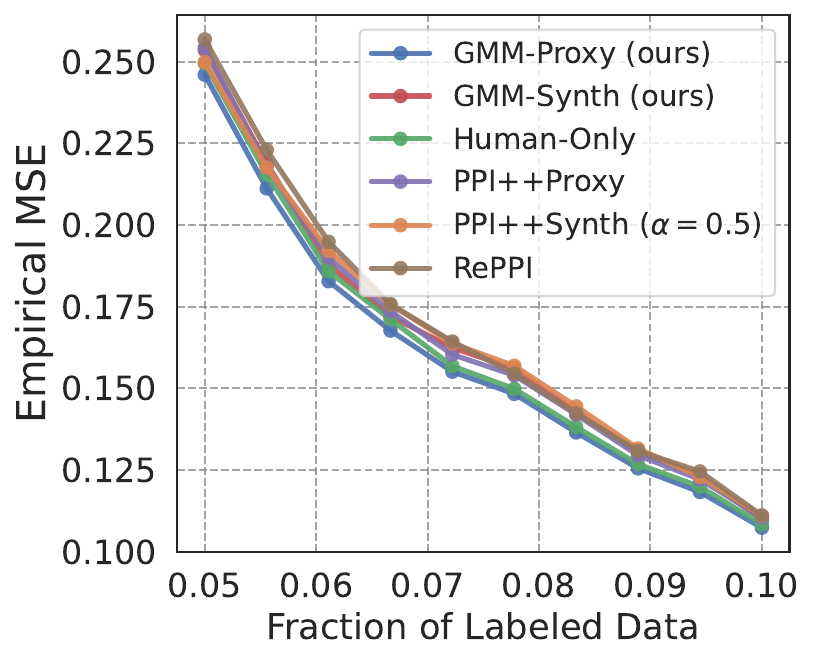}
    \end{subfigure}
    \hfill
    \begin{subfigure}[b]{0.32\textwidth}
        \includegraphics[width=\textwidth]{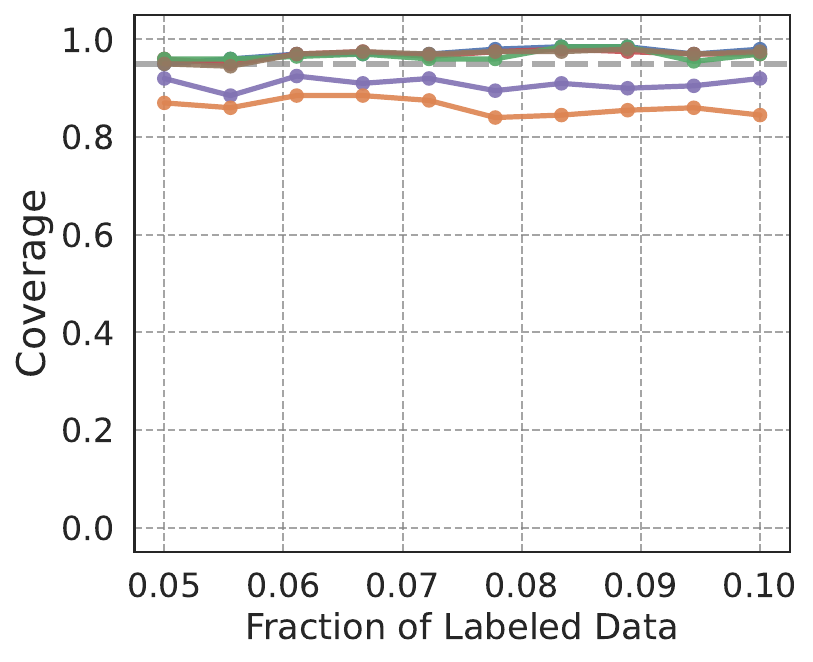}
    \end{subfigure}
    \hfill
    \begin{subfigure}[b]{0.32\textwidth}
        \includegraphics[width=\textwidth]{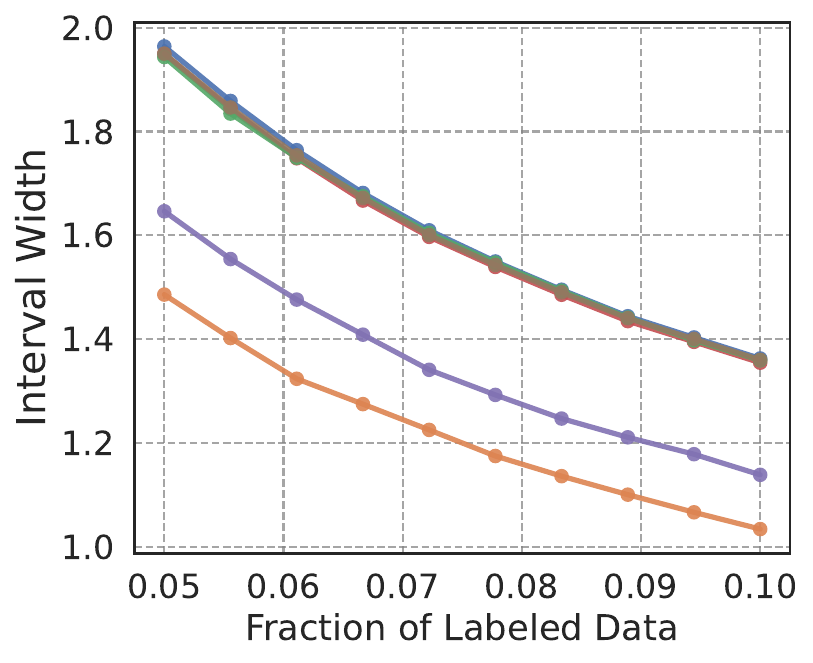}
    \end{subfigure}

    \vspace{0.25cm}

     \begin{subfigure}[b]{0.32\textwidth}
        \includegraphics[width=\textwidth]{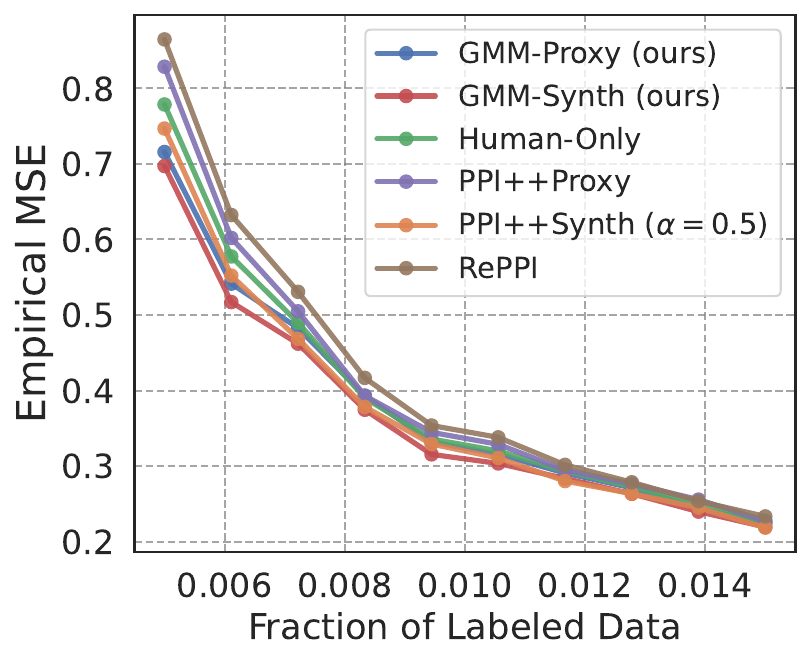}
    \end{subfigure}
    \hfill
    \begin{subfigure}[b]{0.32\textwidth}
        \includegraphics[width=\textwidth]{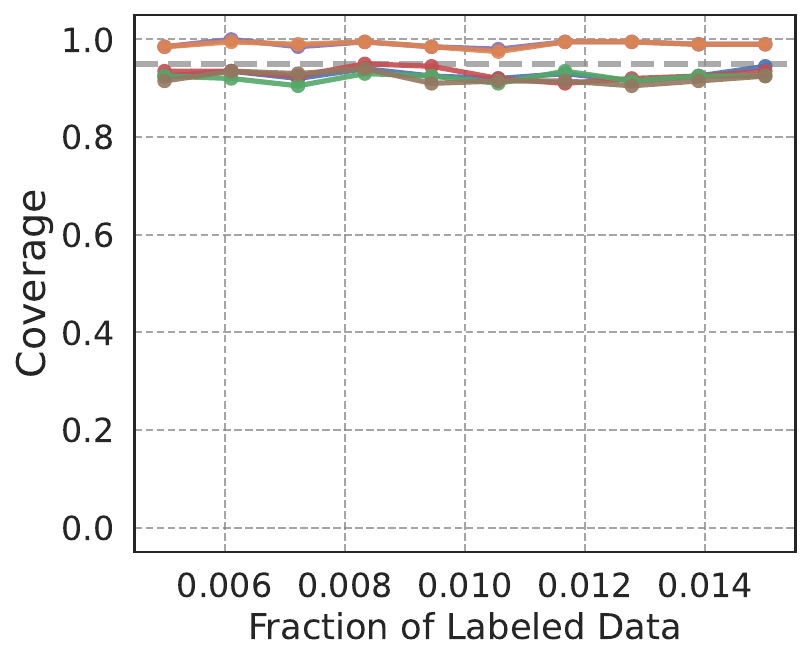}
    \end{subfigure}
    \hfill
    \begin{subfigure}[b]{0.32\textwidth}
        \includegraphics[width=\textwidth]{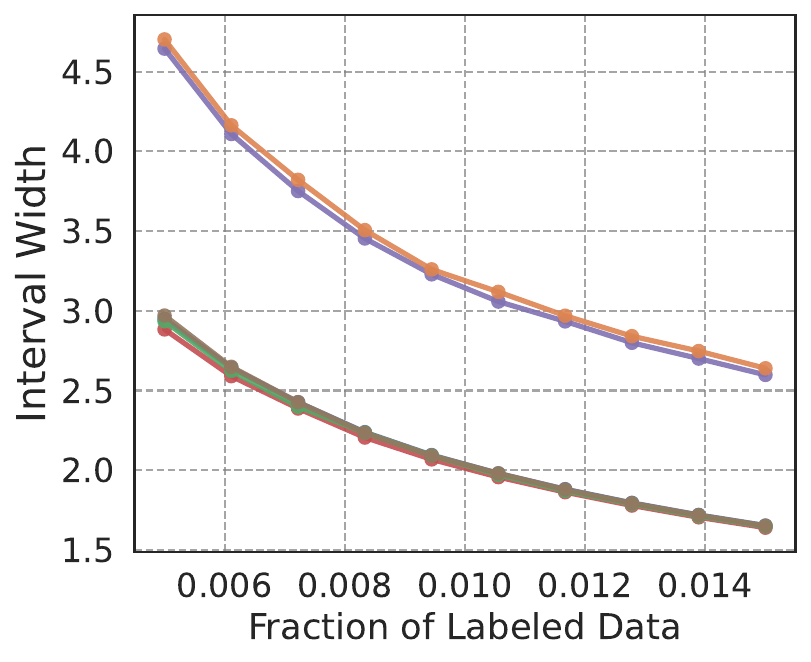}
    \end{subfigure}

    \caption{Llama-3-8b results for logistic regression. Each row corresponds to a task (i.e., 1pp, Hedging, Stance, Congressional Bills Data (from top to bottom)); each column corresponds to a metric (i.e., MSE, coverage, confidence interval width (from left to right)). Results are averaged over 200 trials.}
    \label{fig:llama_lr}
\end{figure}

\begin{figure} [t]
    \centering
    \begin{subfigure}[b]{0.32\textwidth}
        \includegraphics[width=\textwidth]{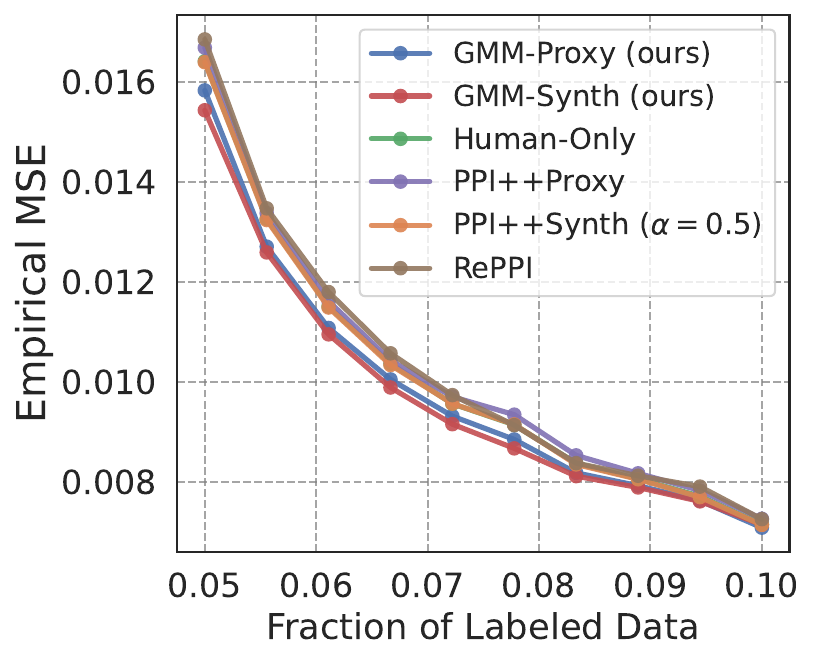}
    \end{subfigure}
    \hfill
    \begin{subfigure}[b]{0.32\textwidth}
        \includegraphics[width=\textwidth]{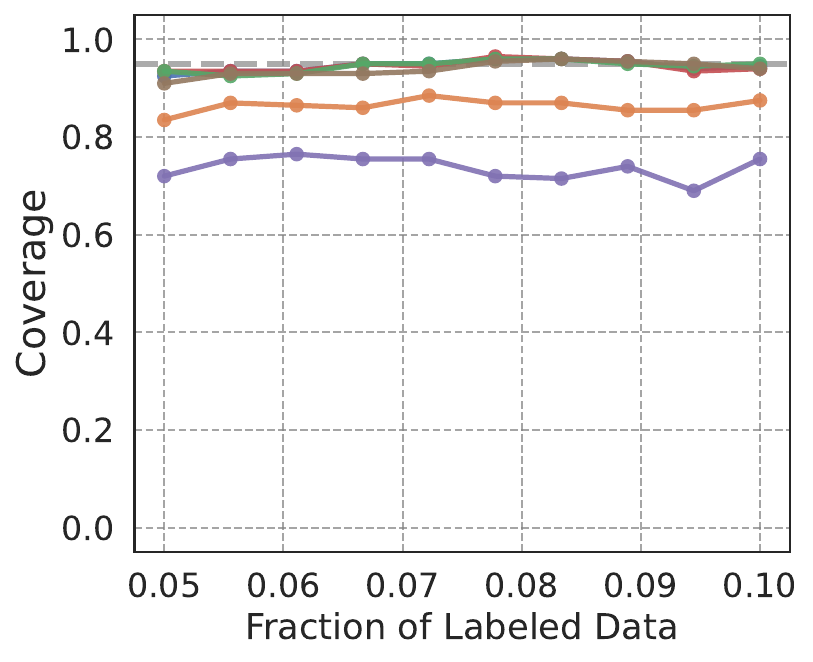}
    \end{subfigure}
    \hfill
    \begin{subfigure}[b]{0.32\textwidth}
        \includegraphics[width=\textwidth]{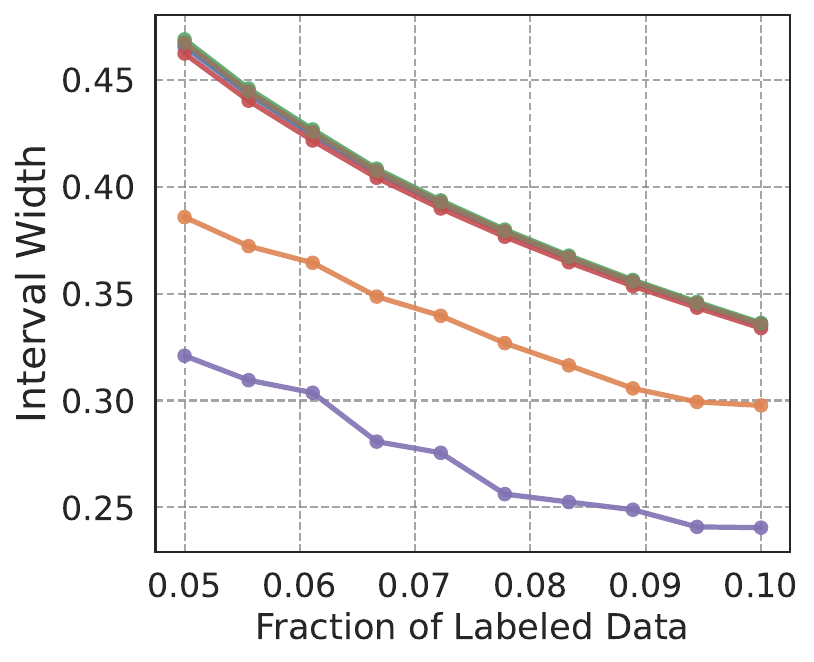}
    \end{subfigure}

    \vspace{0.25cm}
    \begin{subfigure}[b]{0.32\textwidth}
        \includegraphics[width=\textwidth]{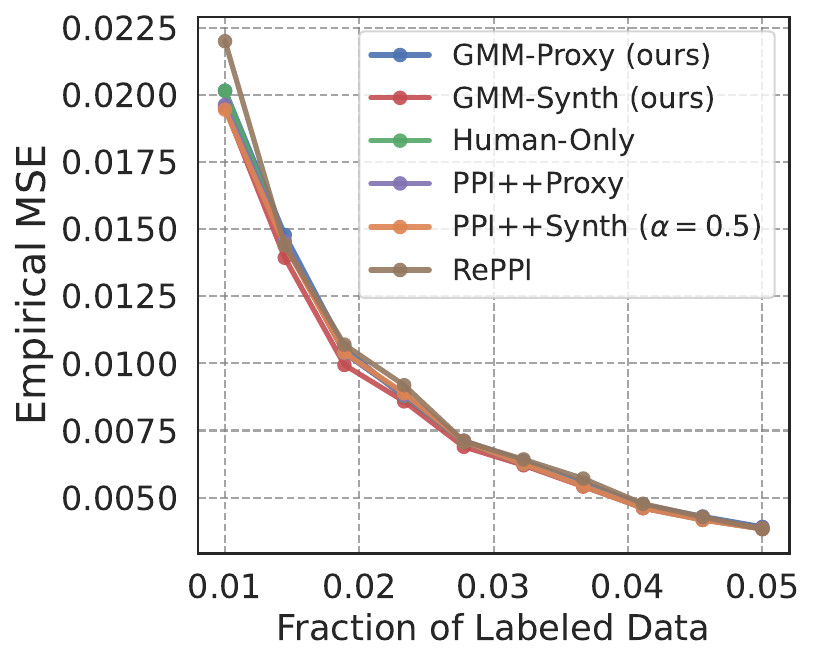}
    \end{subfigure}
    \hfill
    \begin{subfigure}[b]{0.32\textwidth}
        \includegraphics[width=\textwidth]{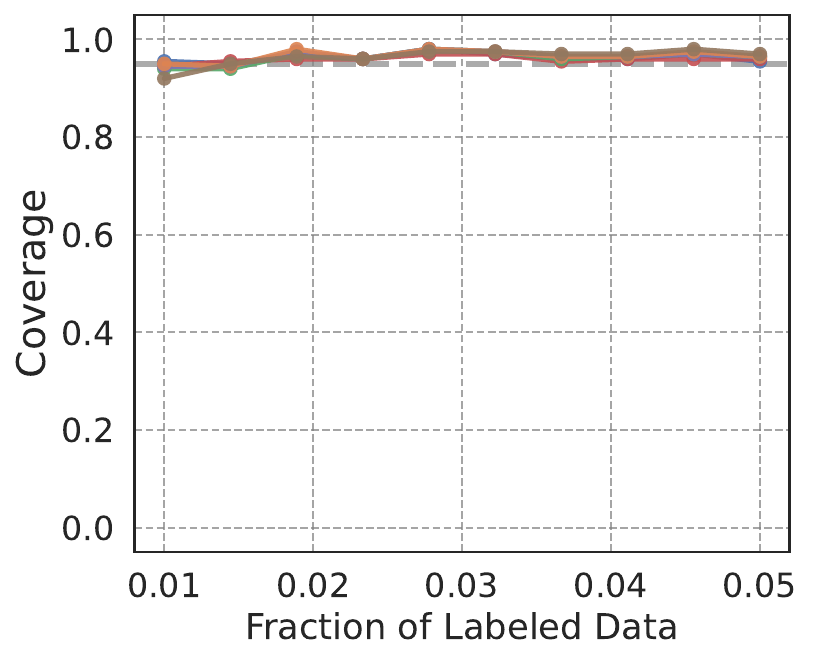}
    \end{subfigure}
    \hfill
    \begin{subfigure}[b]{0.32\textwidth}
        \includegraphics[width=\textwidth]{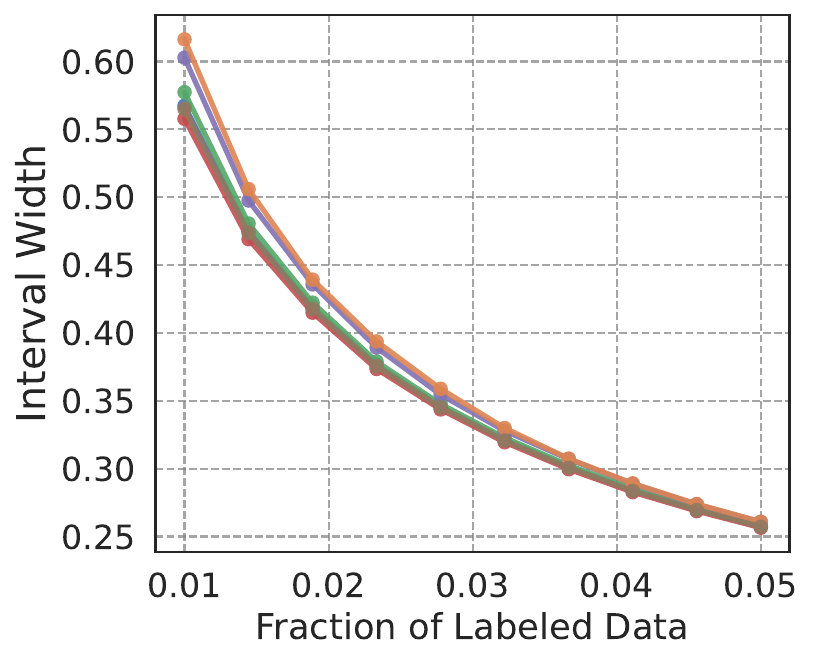}
    \end{subfigure}

    \vspace{0.25cm}

     \begin{subfigure}[b]{0.32\textwidth}
        \includegraphics[width=\textwidth]{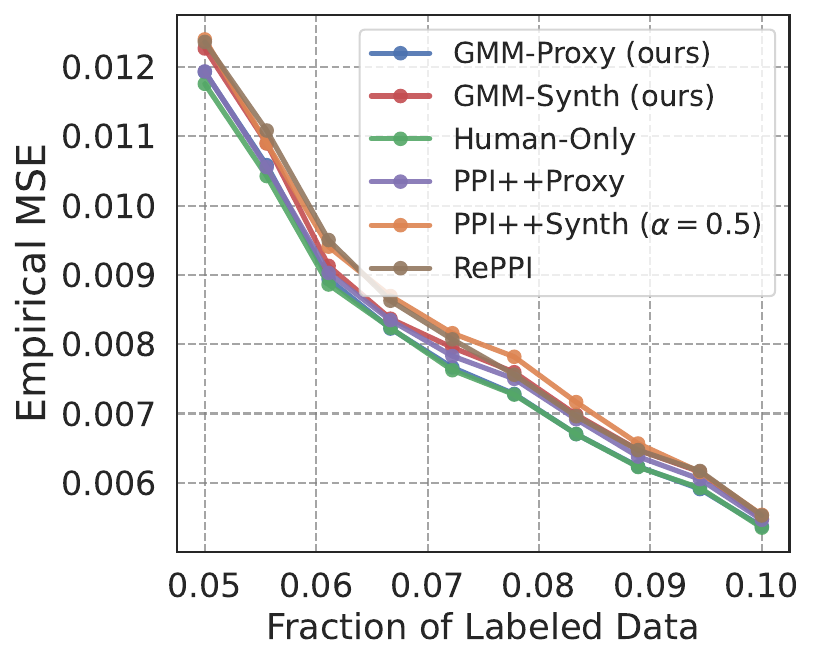}
    \end{subfigure}
    \hfill
    \begin{subfigure}[b]{0.32\textwidth}
        \includegraphics[width=\textwidth]{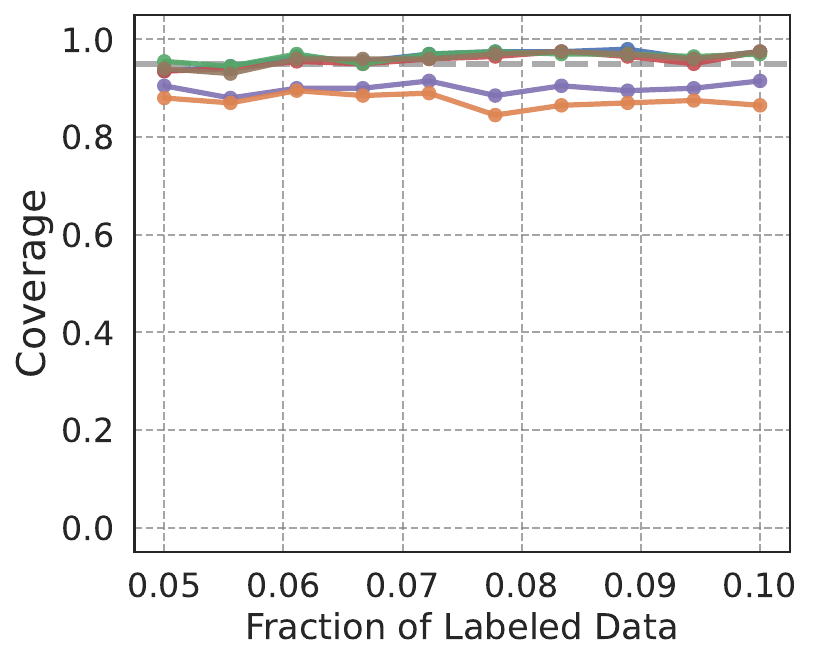}
    \end{subfigure}
    \hfill
    \begin{subfigure}[b]{0.32\textwidth}
        \includegraphics[width=\textwidth]{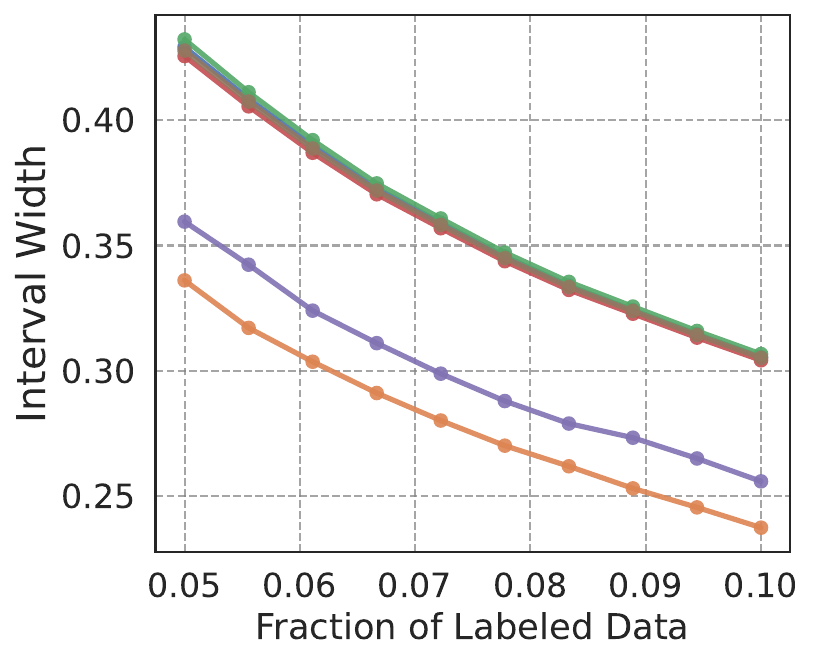}
    \end{subfigure}

    \vspace{0.25cm}

     \begin{subfigure}[b]{0.32\textwidth}
        \includegraphics[width=\textwidth]{plots/plot_cbp_llama_mse.pdf}
    \end{subfigure}
    \hfill
    \begin{subfigure}[b]{0.32\textwidth}
        \includegraphics[width=\textwidth]{plots/plot_cbp_llama_coverage.pdf}
    \end{subfigure}
    \hfill
    \begin{subfigure}[b]{0.32\textwidth}
        \includegraphics[width=\textwidth]{plots/plot_cbp_llama_intervals.pdf}
    \end{subfigure}

    \caption{Llama-3-8b results for OLS. Each row corresponds to a task (i.e., 1pp, Hedging, Stance, Congressional Bills Data (from top to bottom)); each column corresponds to a metric (i.e., MSE, coverage, confidence interval width (from left to right)). Results are averaged over 200 trials.}
    \label{fig:llama_ols}
\end{figure}

\begin{figure} [t]
    \centering
    \begin{subfigure}[b]{0.32\textwidth}
        \includegraphics[width=\textwidth]{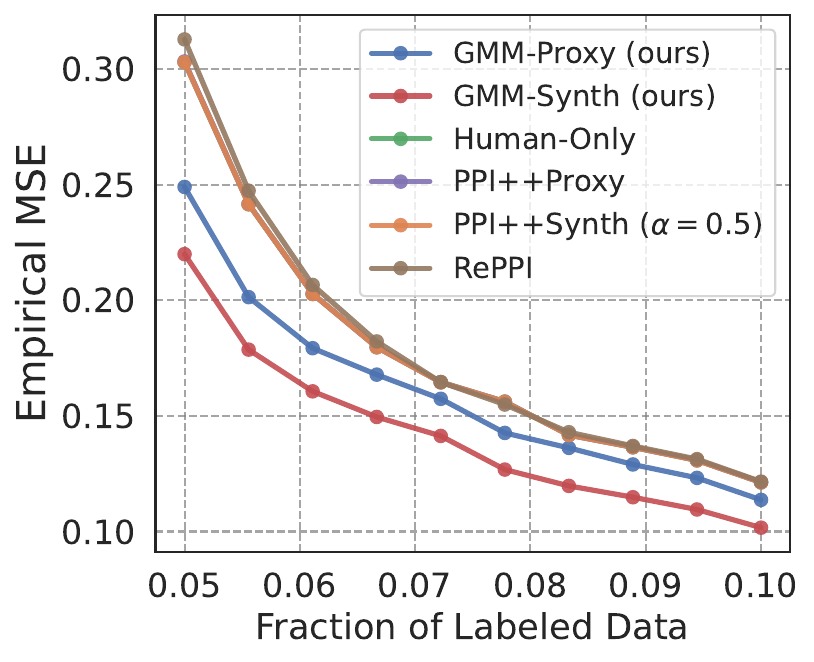}
    \end{subfigure}
    \hfill
    \begin{subfigure}[b]{0.32\textwidth}
        \includegraphics[width=\textwidth]{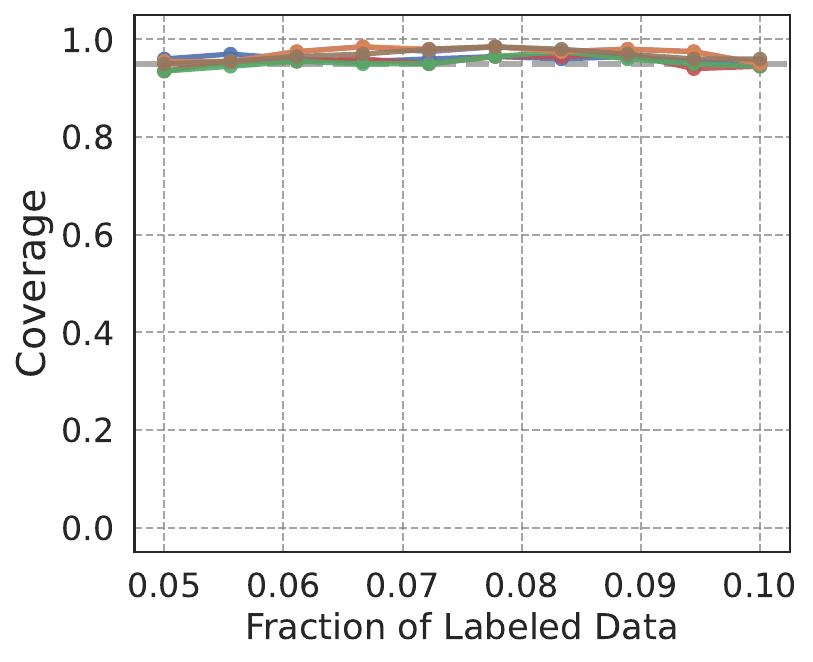}
    \end{subfigure}
    \hfill
    \begin{subfigure}[b]{0.32\textwidth}
        \includegraphics[width=\textwidth]{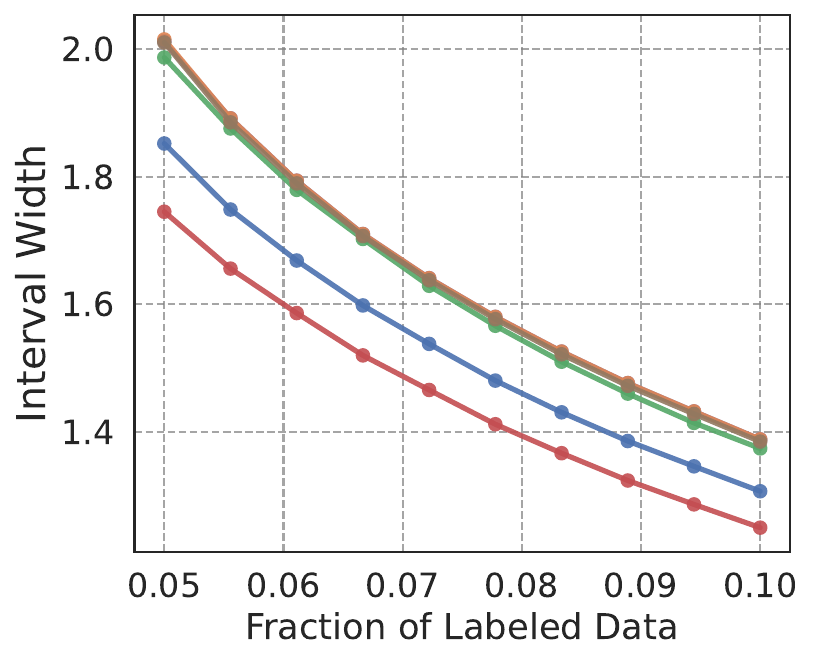}
    \end{subfigure}

    \vspace{0.25cm}
    \begin{subfigure}[b]{0.32\textwidth}
        \includegraphics[width=\textwidth]{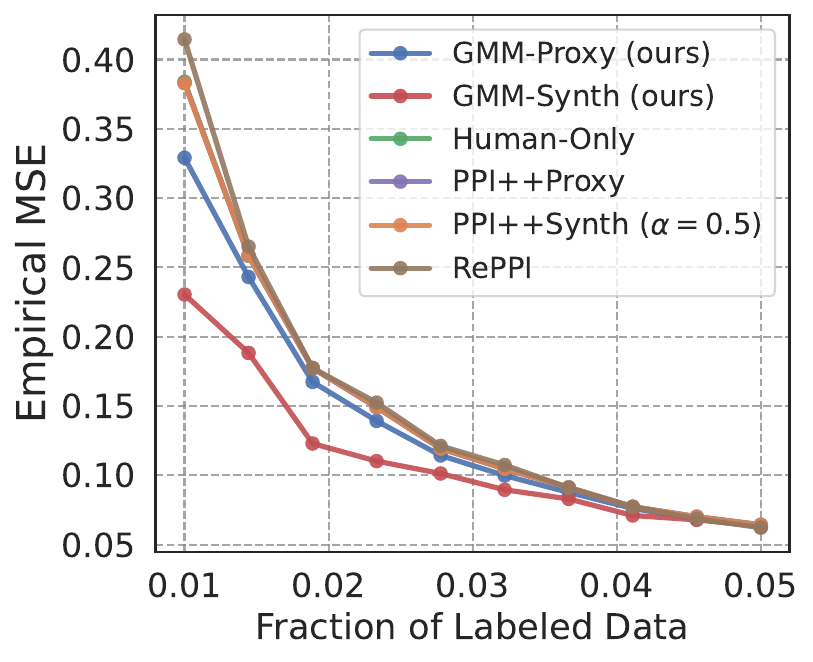}
    \end{subfigure}
    \hfill
    \begin{subfigure}[b]{0.32\textwidth}
        \includegraphics[width=\textwidth]{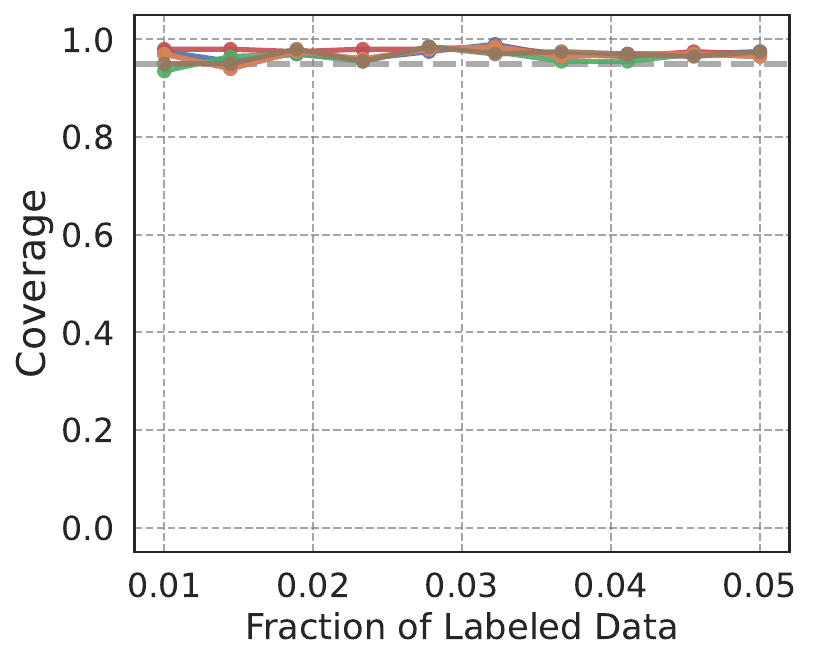}
    \end{subfigure}
    \hfill
    \begin{subfigure}[b]{0.32\textwidth}
        \includegraphics[width=\textwidth]{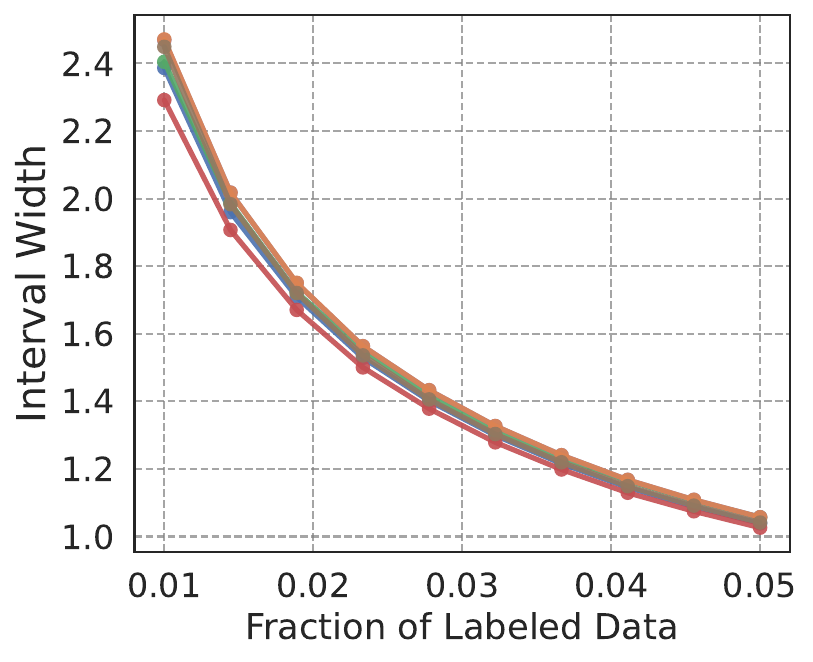}
    \end{subfigure}

    \vspace{0.25cm}

     \begin{subfigure}[b]{0.32\textwidth}
        \includegraphics[width=\textwidth]{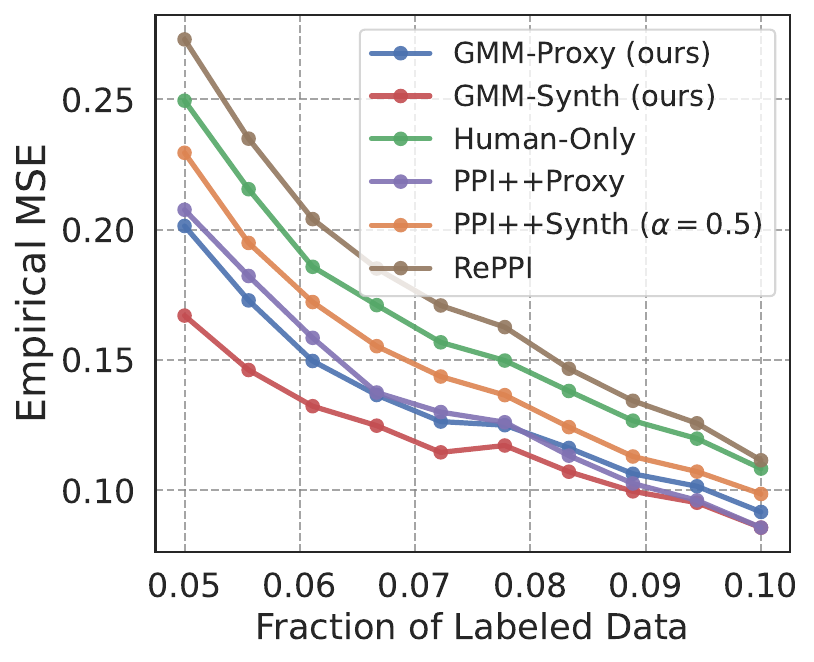}
    \end{subfigure}
    \hfill
    \begin{subfigure}[b]{0.32\textwidth}
        \includegraphics[width=\textwidth]{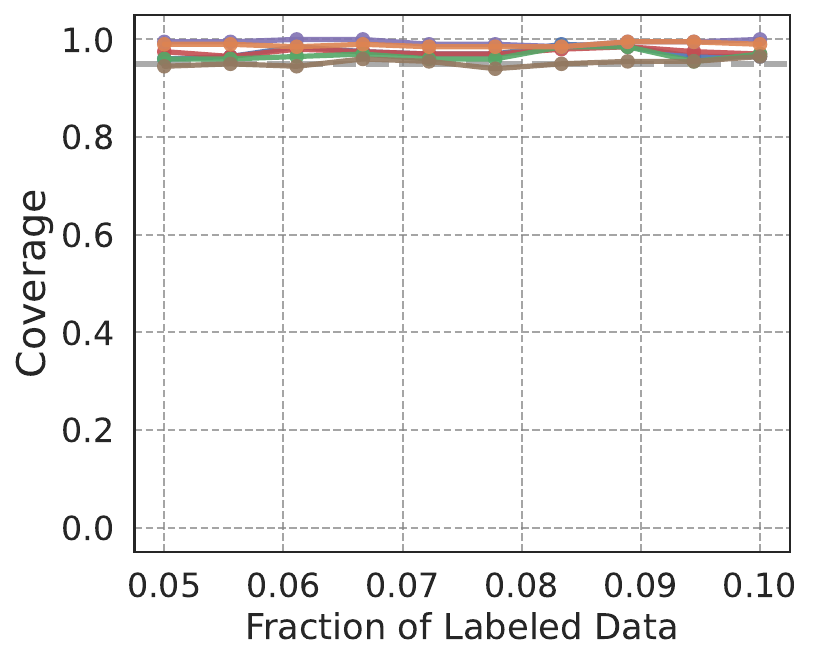}
    \end{subfigure}
    \hfill
    \begin{subfigure}[b]{0.32\textwidth}
        \includegraphics[width=\textwidth]{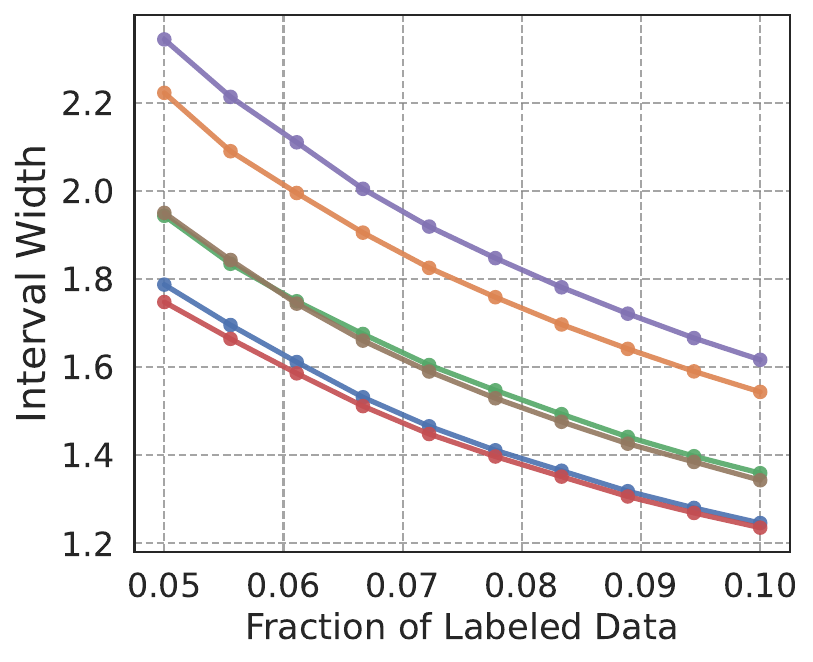}
    \end{subfigure}

    \vspace{0.25cm}

     \begin{subfigure}[b]{0.32\textwidth}
        \includegraphics[width=\textwidth]{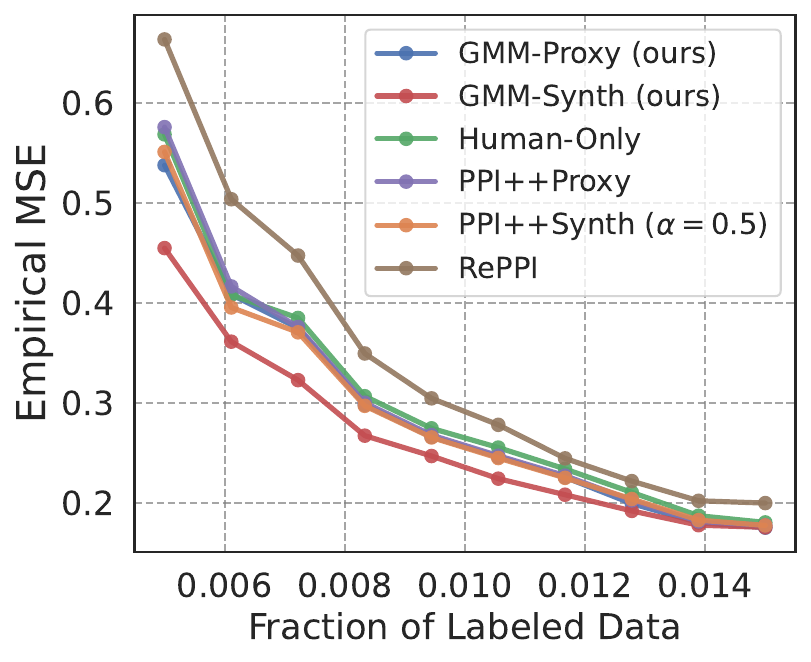}
    \end{subfigure}
    \hfill
    \begin{subfigure}[b]{0.32\textwidth}
        \includegraphics[width=\textwidth]{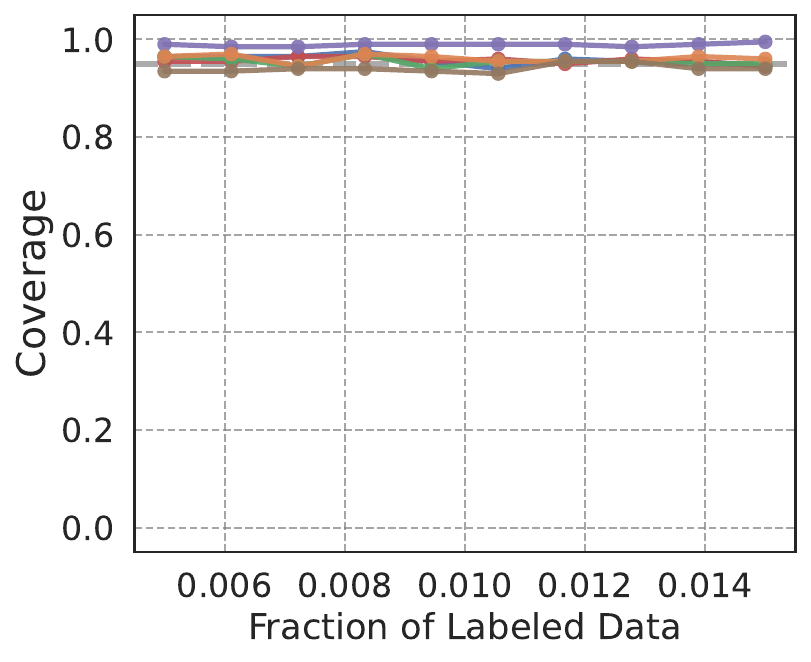}
    \end{subfigure}
    \hfill
    \begin{subfigure}[b]{0.32\textwidth}
        \includegraphics[width=\textwidth]{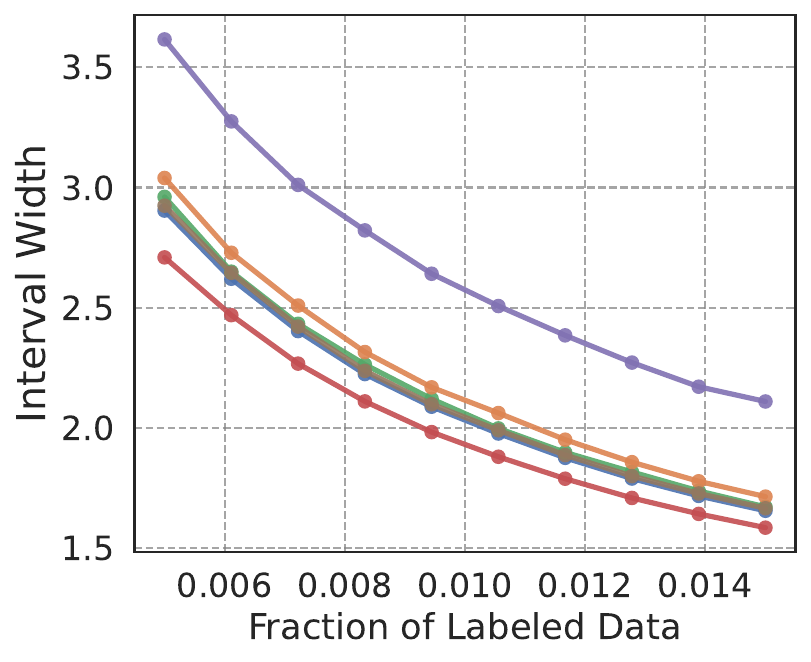}
    \end{subfigure}

    \caption{Qwen-3-8b results for logistic regression. Each row corresponds to a task (i.e., 1pp, Hedging, Stance, Congressional Bills Data (from top to bottom)); each column corresponds to a metric (i.e., MSE, coverage, confidence interval width (from left to right)). Results are averaged over 200 trials.}
    \label{fig:qwen_lr}
\end{figure}

\begin{figure} [t]
    \centering
    \begin{subfigure}[b]{0.32\textwidth}
        \includegraphics[width=\textwidth]{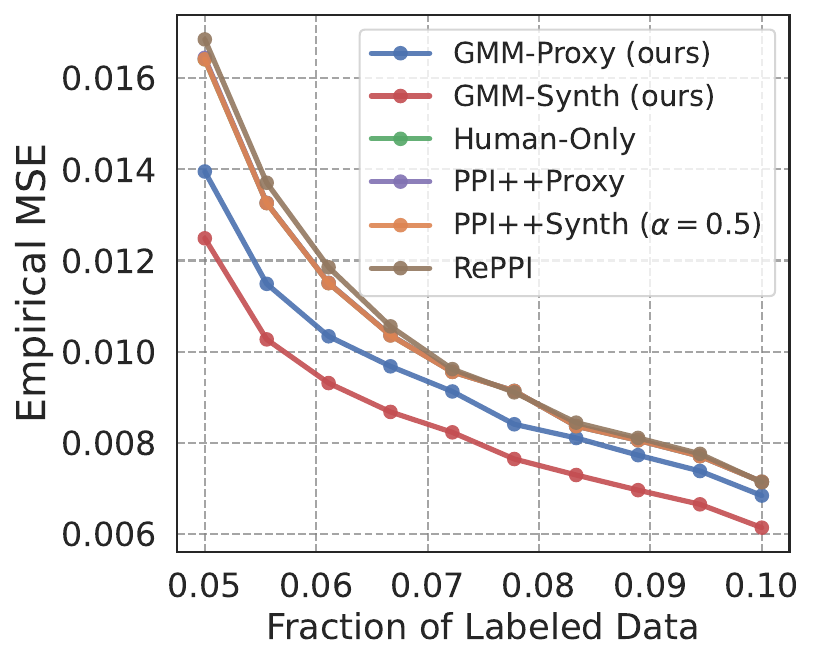}
    \end{subfigure}
    \hfill
    \begin{subfigure}[b]{0.32\textwidth}
        \includegraphics[width=\textwidth]{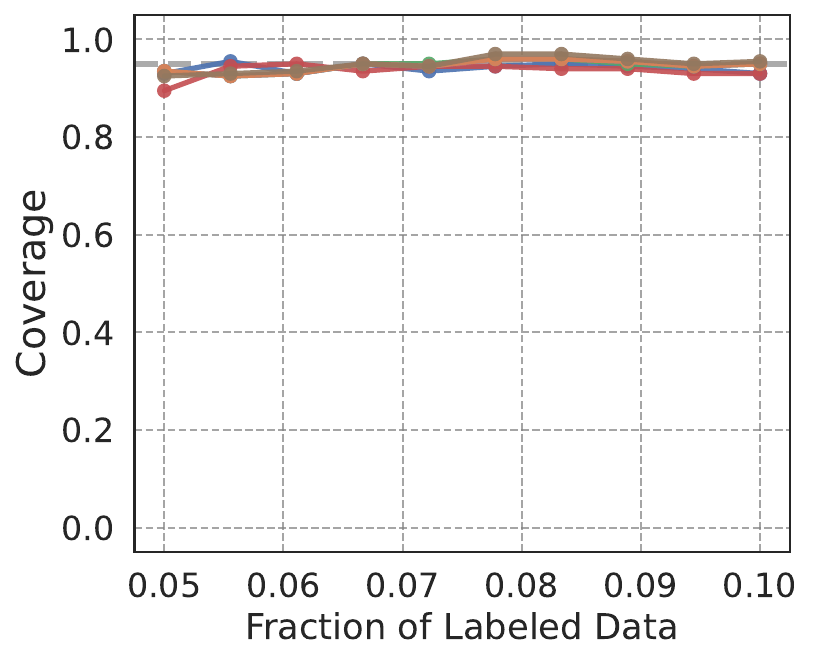}
    \end{subfigure}
    \hfill
    \begin{subfigure}[b]{0.32\textwidth}
        \includegraphics[width=\textwidth]{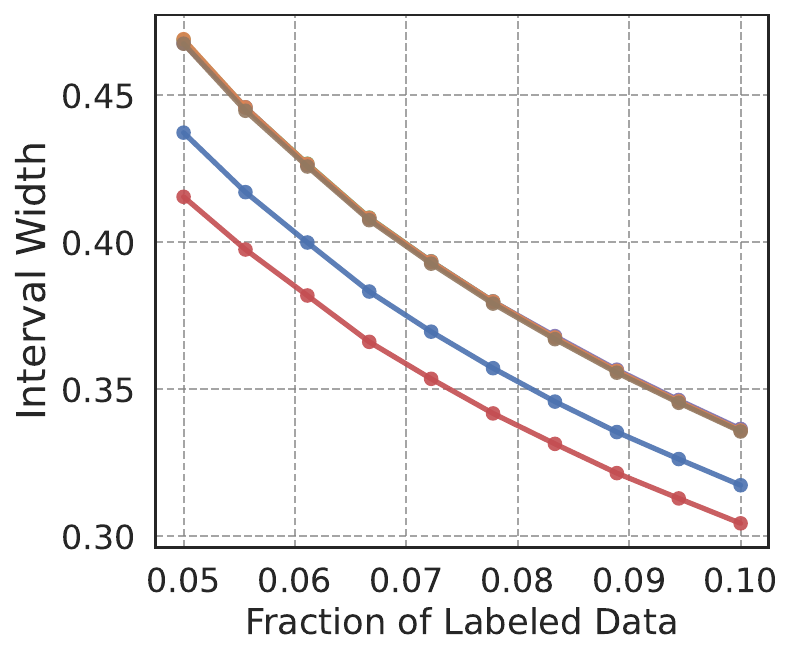}
    \end{subfigure}

    \vspace{0.25cm}
    \begin{subfigure}[b]{0.32\textwidth}
        \includegraphics[width=\textwidth]{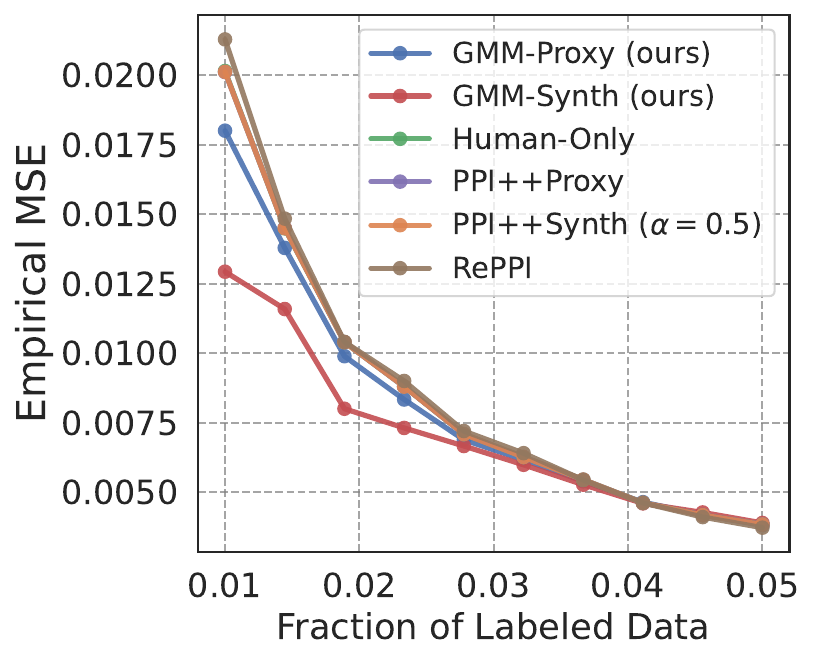}
    \end{subfigure}
    \hfill
    \begin{subfigure}[b]{0.32\textwidth}
        \includegraphics[width=\textwidth]{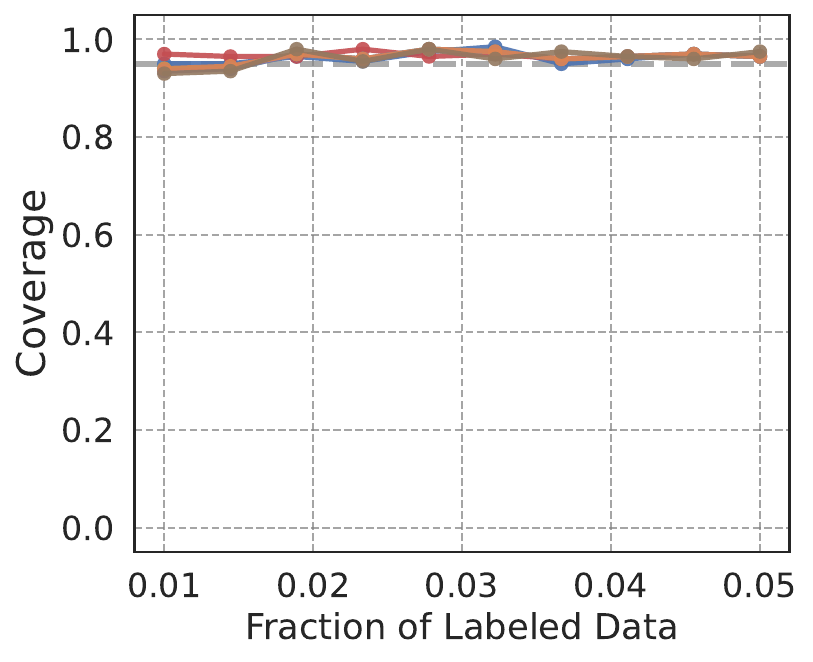}
    \end{subfigure}
    \hfill
    \begin{subfigure}[b]{0.32\textwidth}
        \includegraphics[width=\textwidth]{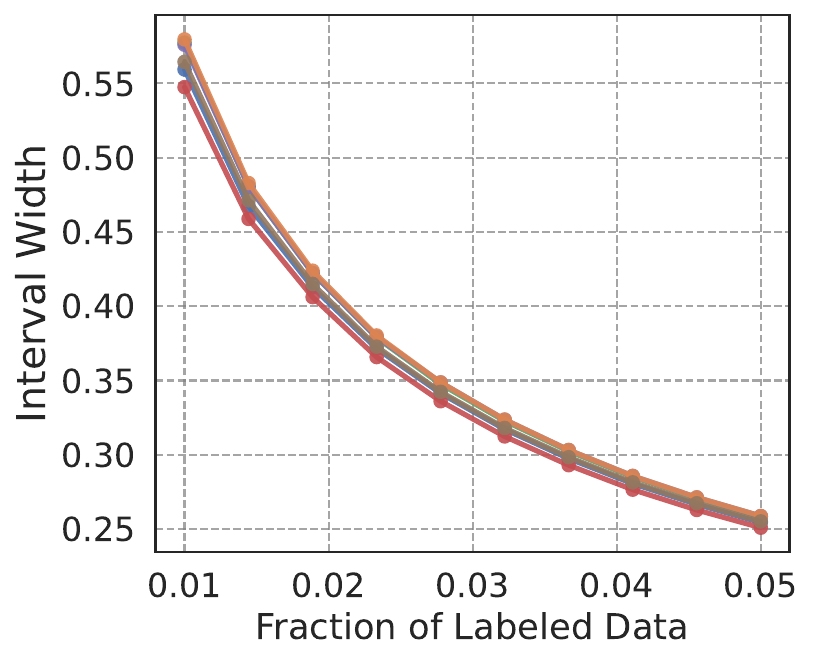}
    \end{subfigure}

    \vspace{0.25cm}

     \begin{subfigure}[b]{0.32\textwidth}
        \includegraphics[width=\textwidth]{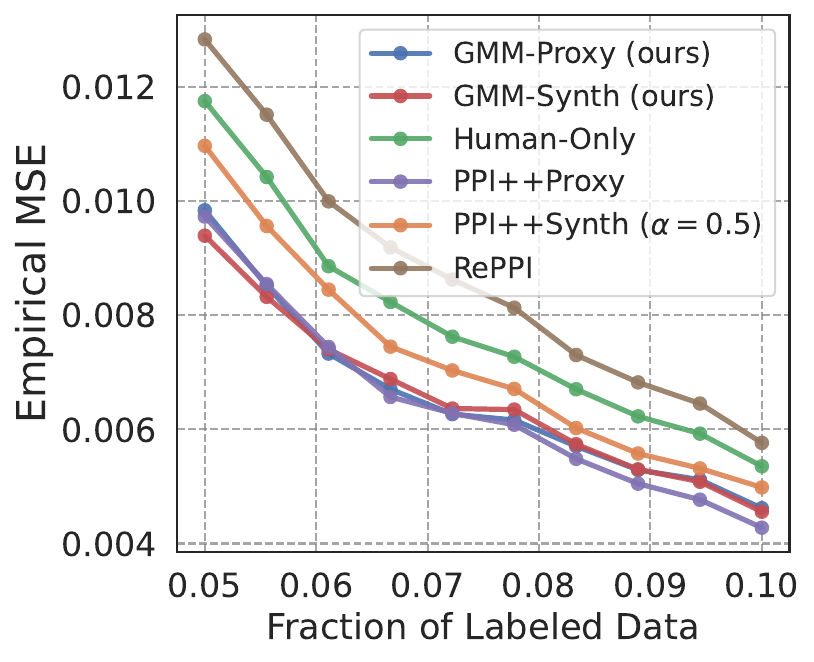}
    \end{subfigure}
    \hfill
    \begin{subfigure}[b]{0.32\textwidth}
        \includegraphics[width=\textwidth]{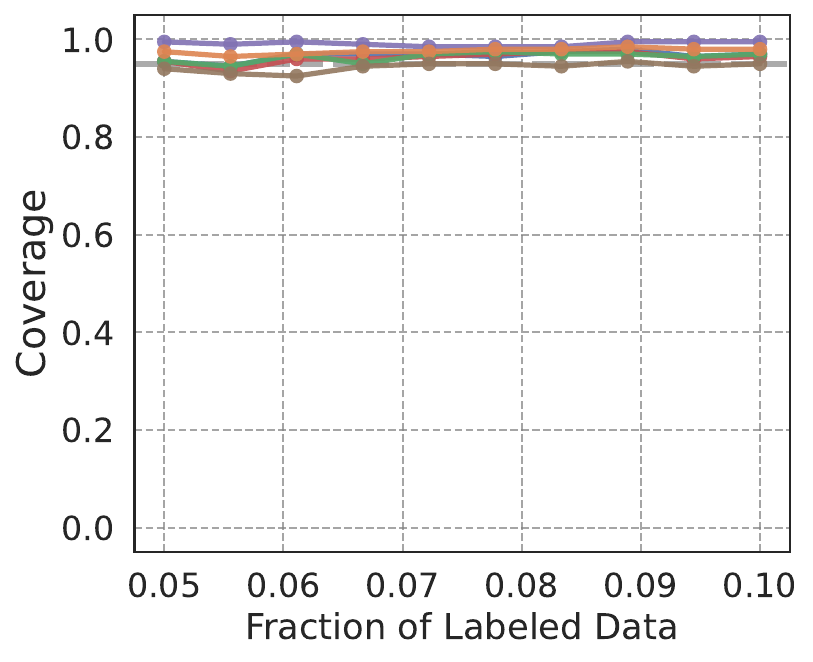}
    \end{subfigure}
    \hfill
    \begin{subfigure}[b]{0.32\textwidth}
        \includegraphics[width=\textwidth]{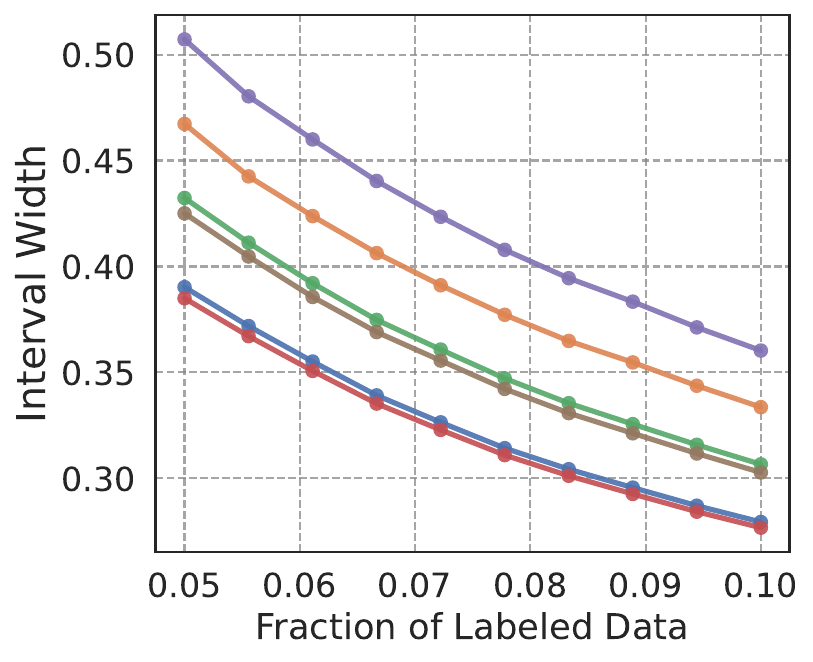}
    \end{subfigure}

    \vspace{0.25cm}

     \begin{subfigure}[b]{0.32\textwidth}
        \includegraphics[width=\textwidth]{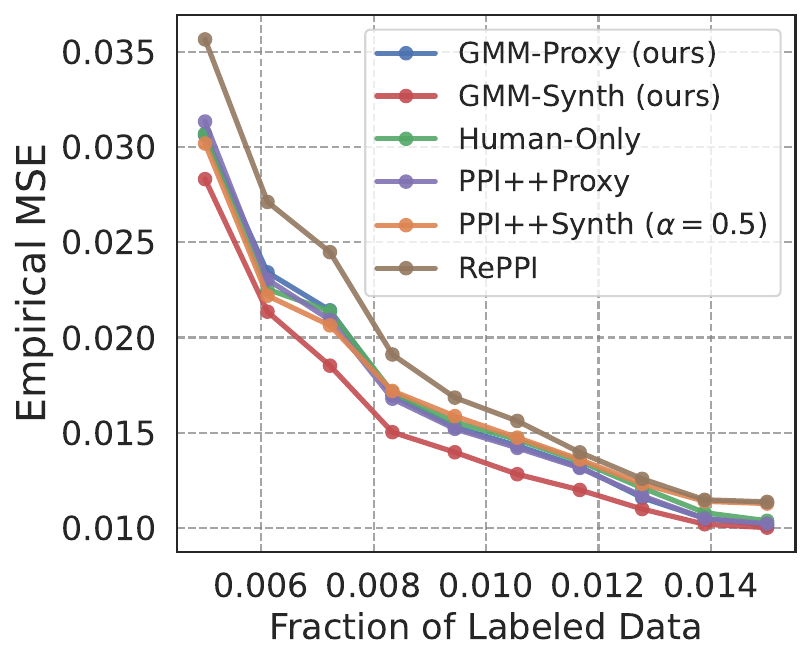}
    \end{subfigure}
    \hfill
    \begin{subfigure}[b]{0.32\textwidth}
        \includegraphics[width=\textwidth]{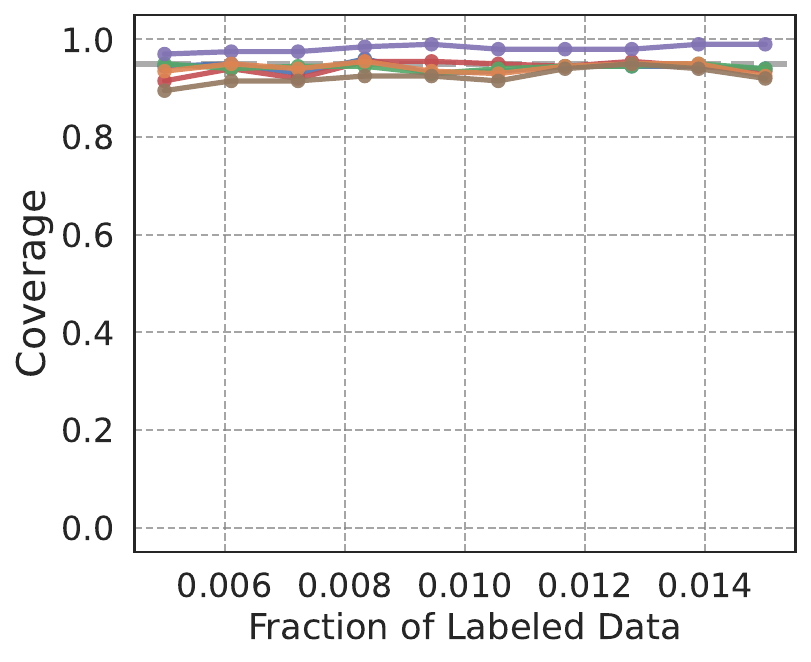}
    \end{subfigure}
    \hfill
    \begin{subfigure}[b]{0.32\textwidth}
        \includegraphics[width=\textwidth]{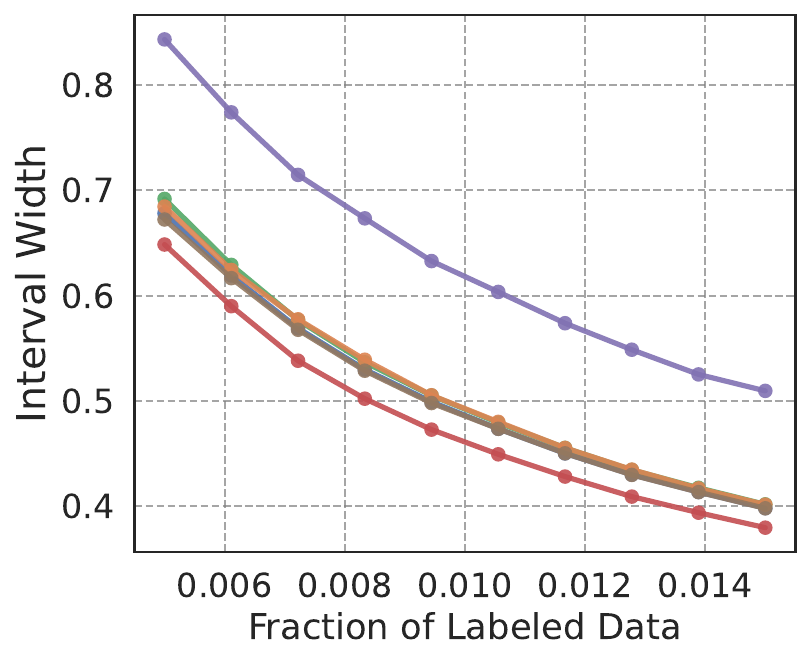}
    \end{subfigure}

    \caption{Qwen-3-8b results for OLS. Each row corresponds to a task (i.e., 1pp, Hedging, Stance, Congressional Bills Data (from top to bottom)); each column corresponds to a metric (i.e., MSE, coverage, confidence interval width (from left to right)). Results are averaged over 200 trials.}
    \label{fig:qwen_ols}
\end{figure}

\begin{figure} [t]
    \centering
    \begin{subfigure}[b]{0.32\textwidth}
        \includegraphics[width=\textwidth]{plots/plot_1pp_mse.pdf}
    \end{subfigure}
    \hfill
    \begin{subfigure}[b]{0.32\textwidth}
        \includegraphics[width=\textwidth]{plots/plot_1pp_coverage.pdf}
    \end{subfigure}
    \hfill
    \begin{subfigure}[b]{0.32\textwidth}
        \includegraphics[width=\textwidth]{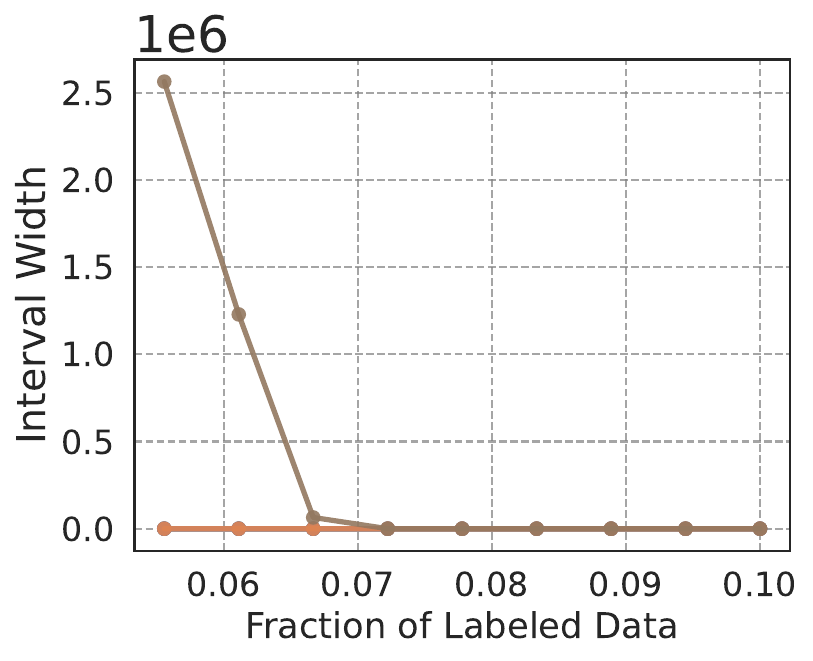}
    \end{subfigure}

    \vspace{0.25cm}
    \begin{subfigure}[b]{0.32\textwidth}
        \includegraphics[width=\textwidth]{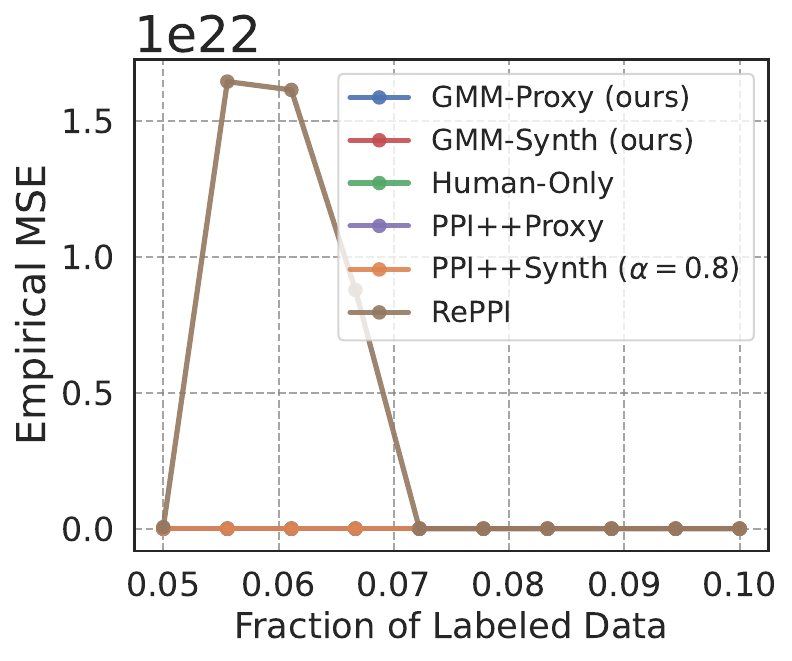}
    \end{subfigure}
    \hfill
    \begin{subfigure}[b]{0.32\textwidth}
        \includegraphics[width=\textwidth]{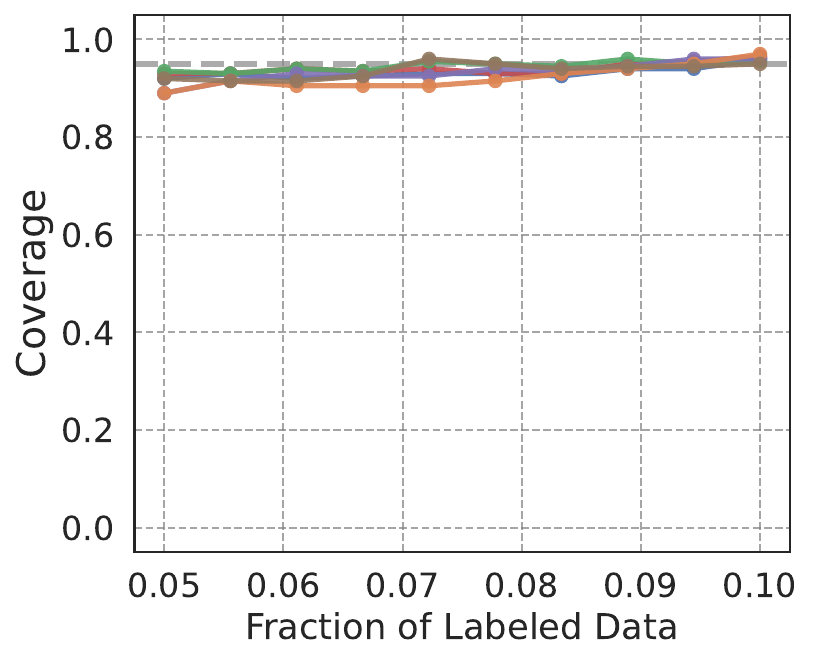}
    \end{subfigure}
    \hfill
    \begin{subfigure}[b]{0.32\textwidth}
        \includegraphics[width=\textwidth]{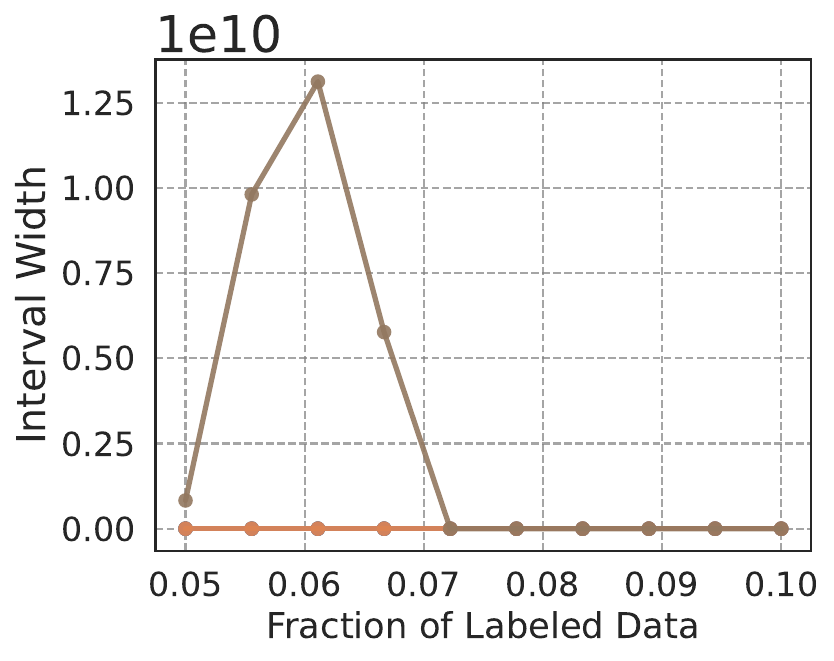}
    \end{subfigure}

    \caption{RePPI results for 1pp task for logistic regression (top) and OLS (bottom). We report the RePPI method results for the 1pp task separately here, due to some very large values. In the main text (Figures \ref{fig:key_results}, \ref{fig:key_results_ols}), we exclude the RePPI numbers from the plots for better visibility.}
    \label{fig:reppi_results}
\end{figure}

%% file: main.bbl
\begin{thebibliography}{29}
\providecommand{\natexlab}[1]{#1}
\providecommand{\url}[1]{\texttt{#1}}
\expandafter\ifx\csname urlstyle\endcsname\relax
  \providecommand{\doi}[1]{doi: #1}\else
  \providecommand{\doi}{doi: \begingroup \urlstyle{rm}\Url}\fi

\bibitem[Adler and Wilkerson(2011)]{adler2011congressional}
E~Scott Adler and John Wilkerson.
\newblock Congressional bills project.
\newblock \emph{NSF 00880066 and}, 880061, 2011.

\bibitem[Angelopoulos et~al.(2023{\natexlab{a}})Angelopoulos, Bates, Fannjiang, Jordan, and Zrnic]{angelopoulos2023prediction}
Anastasios~N Angelopoulos, Stephen Bates, Clara Fannjiang, Michael~I Jordan, and Tijana Zrnic.
\newblock Prediction-powered inference.
\newblock \emph{Science}, 382\penalty0 (6671):\penalty0 669--674, 2023{\natexlab{a}}.

\bibitem[Angelopoulos et~al.(2023{\natexlab{b}})Angelopoulos, Duchi, and Zrnic]{angelopoulos2023ppi++}
Anastasios~N Angelopoulos, John~C Duchi, and Tijana Zrnic.
\newblock Ppi++: Efficient prediction-powered inference.
\newblock \emph{arXiv preprint arXiv:2311.01453}, 2023{\natexlab{b}}.

\bibitem[Anthis et~al.(2025)Anthis, Liu, Richardson, Kozlowski, Koch, Evans, Brynjolfsson, and Bernstein]{anthis2025llm}
Jacy~Reese Anthis, Ryan Liu, Sean~M Richardson, Austin~C Kozlowski, Bernard Koch, James Evans, Erik Brynjolfsson, and Michael Bernstein.
\newblock Llm social simulations are a promising research method.
\newblock \emph{arXiv preprint arXiv:2504.02234}, 2025.

\bibitem[Argyle et~al.(2023)Argyle, Busby, Fulda, Gubler, Rytting, and Wingate]{argyle2023out}
Lisa~P Argyle, Ethan~C Busby, Nancy Fulda, Joshua~R Gubler, Christopher Rytting, and David Wingate.
\newblock Out of one, many: Using language models to simulate human samples.
\newblock \emph{Political Analysis}, 31\penalty0 (3):\penalty0 337--351, 2023.

\bibitem[Brand et~al.(2023)Brand, Israeli, and Ngwe]{brand2023using}
James Brand, Ayelet Israeli, and Donald Ngwe.
\newblock Using llms for market research.
\newblock \emph{Harvard Business School Marketing Unit Working Paper}, \penalty0 (23-062), 2023.

\bibitem[Chamberlain(1987)]{chamberlain1987asymptotic}
Gary Chamberlain.
\newblock Asymptotic efficiency in estimation with conditional moment restrictions.
\newblock \emph{Journal of econometrics}, 34\penalty0 (3):\penalty0 305--334, 1987.

\bibitem[Chen et~al.()Chen, Wang, Xu, Yuan, Zhang, Shi, Xie, Li, Yang, Zhu, et~al.]{chenpersona}
Jiangjie Chen, Xintao Wang, Rui Xu, Siyu Yuan, Yikai Zhang, Wei Shi, Jian Xie, Shuang Li, Ruihan Yang, Tinghui Zhu, et~al.
\newblock From persona to personalization: A survey on role-playing language agents.
\newblock \emph{Transactions on Machine Learning Research}.

\bibitem[Chernozhukov et~al.(2018)Chernozhukov, Chetverikov, Demirer, Duflo, Hansen, Newey, and Robins]{chernozhukov2018double}
Victor Chernozhukov, Denis Chetverikov, Mert Demirer, Esther Duflo, Christian Hansen, Whitney Newey, and James Robins.
\newblock Double/debiased machine learning for treatment and structural parameters: Double/debiased machine learning.
\newblock \emph{The Econometrics Journal}, 21\penalty0 (1), 2018.

\bibitem[Danescu-Niculescu-Mizil et~al.(2013)Danescu-Niculescu-Mizil, Sudhof, Jurafsky, Leskovec, and Potts]{danescu2013computational}
Cristian Danescu-Niculescu-Mizil, Moritz Sudhof, Dan Jurafsky, Jure Leskovec, and Christopher Potts.
\newblock A computational approach to politeness with application to social factors.
\newblock \emph{arXiv preprint arXiv:1306.6078}, 2013.

\bibitem[Dillion et~al.(2023)Dillion, Tandon, Gu, and Gray]{dillion2023can}
Danica Dillion, Niket Tandon, Yuling Gu, and Kurt Gray.
\newblock Can ai language models replace human participants?
\newblock \emph{Trends in Cognitive Sciences}, 27\penalty0 (7):\penalty0 597--600, 2023.

\bibitem[Dominguez-Olmedo et~al.(2024)Dominguez-Olmedo, Hardt, and Mendler-D{\"u}nner]{dominguez2024questioning}
Ricardo Dominguez-Olmedo, Moritz Hardt, and Celestine Mendler-D{\"u}nner.
\newblock Questioning the survey responses of large language models.
\newblock \emph{Advances in Neural Information Processing Systems}, 37:\penalty0 45850--45878, 2024.

\bibitem[Egami et~al.(2023)Egami, Hinck, Stewart, and Wei]{egami2023using}
Naoki Egami, Musashi Hinck, Brandon Stewart, and Hanying Wei.
\newblock Using imperfect surrogates for downstream inference: Design-based supervised learning for social science applications of large language models.
\newblock \emph{Advances in Neural Information Processing Systems}, 36:\penalty0 68589--68601, 2023.

\bibitem[Gligori{\'c} et~al.(2024)Gligori{\'c}, Zrnic, Lee, Cand{\`e}s, and Jurafsky]{gligoric2024can}
Kristina Gligori{\'c}, Tijana Zrnic, Cinoo Lee, Emmanuel~J Cand{\`e}s, and Dan Jurafsky.
\newblock Can unconfident llm annotations be used for confident conclusions?
\newblock \emph{arXiv preprint arXiv:2408.15204}, 2024.

\bibitem[Grattafiori et~al.(2024)Grattafiori, Dubey, Jauhri, Pandey, Kadian, Al-Dahle, Letman, Mathur, Schelten, Vaughan, et~al.]{grattafiori2024llama}
Aaron Grattafiori, Abhimanyu Dubey, Abhinav Jauhri, Abhinav Pandey, Abhishek Kadian, Ahmad Al-Dahle, Aiesha Letman, Akhil Mathur, Alan Schelten, Alex Vaughan, et~al.
\newblock The llama 3 herd of models.
\newblock \emph{arXiv preprint arXiv:2407.21783}, 2024.

\bibitem[Hansen(1982)]{hansen1982large}
Lars~Peter Hansen.
\newblock Large sample properties of generalized method of moments estimators.
\newblock \emph{Econometrica: Journal of the econometric society}, pages 1029--1054, 1982.

\bibitem[Hmielowski et~al.(2014)Hmielowski, Feldman, Myers, Leiserowitz, and Maibach]{hmielowski2014attack}
Jay~D Hmielowski, Lauren Feldman, Teresa~A Myers, Anthony Leiserowitz, and Edward Maibach.
\newblock An attack on science? media use, trust in scientists, and perceptions of global warming.
\newblock \emph{Public Understanding of Science}, 23\penalty0 (7):\penalty0 866--883, 2014.

\bibitem[Hurst et~al.(2024)Hurst, Lerer, Goucher, Perelman, Ramesh, Clark, Ostrow, Welihinda, Hayes, Radford, et~al.]{hurst2024gpt}
Aaron Hurst, Adam Lerer, Adam~P Goucher, Adam Perelman, Aditya Ramesh, Aidan Clark, AJ~Ostrow, Akila Welihinda, Alan Hayes, Alec Radford, et~al.
\newblock Gpt-4o system card.
\newblock \emph{arXiv preprint arXiv:2410.21276}, 2024.

\bibitem[Hwang et~al.(2025{\natexlab{a}})Hwang, Bernstein, Sundar, Zhang, Horta~Ribeiro, Lu, Chang, Wu, Yang, Williams, Park, Ognyanova, Xiao, Shaw, and Shamma]{10.1145/3706599.3716299}
Angel Hsing-Chi Hwang, Michael~S. Bernstein, S.~Shyam Sundar, Renwen Zhang, Manoel Horta~Ribeiro, Yingdan Lu, Serina Chang, Tongshuang Wu, Aimei Yang, Dmitri Williams, Joon~Sung Park, Katherine Ognyanova, Ziang Xiao, Aaron Shaw, and David~A. Shamma.
\newblock Human subjects research in the age of generative ai: Opportunities and challenges of applying llm-simulated data to hci studies.
\newblock In \emph{Proceedings of the Extended Abstracts of the CHI Conference on Human Factors in Computing Systems}, CHI EA '25, New York, NY, USA, 2025{\natexlab{a}}. Association for Computing Machinery.
\newblock ISBN 9798400713958.
\newblock \doi{10.1145/3706599.3716299}.
\newblock URL \url{https://doi.org/10.1145/3706599.3716299}.

\bibitem[Hwang et~al.(2025{\natexlab{b}})Hwang, Bernstein, Sundar, Zhang, Horta~Ribeiro, Lu, Chang, Wu, Yang, Williams, et~al.]{hwang2025human}
Angel Hsing-Chi Hwang, Michael~S Bernstein, S~Shyam Sundar, Renwen Zhang, Manoel Horta~Ribeiro, Yingdan Lu, Serina Chang, Tongshuang Wu, Aimei Yang, Dmitri Williams, et~al.
\newblock Human subjects research in the age of generative ai: Opportunities and challenges of applying llm-simulated data to hci studies.
\newblock In \emph{Proceedings of the Extended Abstracts of the CHI Conference on Human Factors in Computing Systems}, pages 1--7, 2025{\natexlab{b}}.

\bibitem[Ji et~al.(2025)Ji, Lei, and Zrnic]{ji2025predictions}
Wenlong Ji, Lihua Lei, and Tijana Zrnic.
\newblock Predictions as surrogates: Revisiting surrogate outcomes in the age of ai.
\newblock \emph{arXiv preprint arXiv:2501.09731}, 2025.

\bibitem[Lewis et~al.(2024)Lewis, Poole, Rosenthal, Boche, Rudkin, and Sonnet]{lewis2024congressional}
Jeffrey~B Lewis, Keith Poole, Howard Rosenthal, Adam Boche, Aaron Rudkin, and Luke Sonnet.
\newblock Congressional roll-call votes database.
\newblock \emph{Published Online}, 2024.

\bibitem[Newey and McFadden(1994)]{newey1994large}
Whitney~K Newey and Daniel McFadden.
\newblock Large sample estimation and hypothesis testing.
\newblock \emph{Handbook of econometrics}, 4:\penalty0 2111--2245, 1994.

\bibitem[Park et~al.(2023)Park, O'Brien, Cai, Morris, Liang, and Bernstein]{park2023generative}
Joon~Sung Park, Joseph O'Brien, Carrie~Jun Cai, Meredith~Ringel Morris, Percy Liang, and Michael~S Bernstein.
\newblock Generative agents: Interactive simulacra of human behavior.
\newblock In \emph{Proceedings of the 36th annual acm symposium on user interface software and technology}, pages 1--22, 2023.

\bibitem[Robins et~al.(1994)Robins, Rotnitzky, and Zhao]{robins1994estimation}
James~M Robins, Andrea Rotnitzky, and Lue~Ping Zhao.
\newblock Estimation of regression coefficients when some regressors are not always observed.
\newblock \emph{Journal of the American statistical Association}, 89\penalty0 (427):\penalty0 846--866, 1994.

\bibitem[Rothschild et~al.(2024)Rothschild, Brand, Schroeder, and Wang]{rothschild2024opportunities}
David~M Rothschild, James Brand, Hope Schroeder, and Jenny Wang.
\newblock Opportunities and risks of llms in survey research.
\newblock \emph{Available at SSRN}, 2024.

\bibitem[Van Der~Vaart and Wellner()]{van1996weak}
Aad~W Van Der~Vaart and Jon~A Wellner.
\newblock \emph{Weak convergence and empirical processes: with applications to statistics}.
\newblock Springer.

\bibitem[Yang et~al.(2025)Yang, Li, Yang, Zhang, Hui, Zheng, Yu, Gao, Huang, Lv, et~al.]{yang2025qwen3}
An~Yang, Anfeng Li, Baosong Yang, Beichen Zhang, Binyuan Hui, Bo~Zheng, Bowen Yu, Chang Gao, Chengen Huang, Chenxu Lv, et~al.
\newblock Qwen3 technical report.
\newblock \emph{arXiv preprint arXiv:2505.09388}, 2025.

\bibitem[Ziems et~al.(2024)Ziems, Held, Shaikh, Chen, Zhang, and Yang]{ziems2024can}
Caleb Ziems, William Held, Omar Shaikh, Jiaao Chen, Zhehao Zhang, and Diyi Yang.
\newblock Can large language models transform computational social science?
\newblock \emph{Computational Linguistics}, 50\penalty0 (1):\penalty0 237--291, 2024.

\end{thebibliography}
